\documentclass[review,authoryear]{elsarticle}

\usepackage{soul}
\usepackage[normalem]{ulem}

\newcounter{subfig}[figure]

\usepackage[T1]{fontenc}
\usepackage{babel}
\usepackage{bm}
\usepackage{graphicx}      %
\usepackage{array}         %
\usepackage{adjustbox}

\usepackage{caption}

\usepackage{mwe}           %

\newcommand{\plot}[1]{%
  \raisebox{-0.3\height}{\includegraphics[width=\linewidth]{#1}}}

\newcommand{\rowlabel}[1]{%
  \rotatebox[origin=c]{90}{\textbf{#1}}}

\usepackage{amsmath}
\usepackage{amssymb}
\usepackage{amsthm}

\usepackage{hyperref}
\usepackage{url}
\usepackage{eucal}
\usepackage{lipsum}
\usepackage{amsfonts}
\usepackage{graphicx}
\usepackage{epstopdf}
\usepackage{algorithmic}
\usepackage{enumitem}
\usepackage{amsopn}
\usepackage{natbib}
\usepackage[margin=1.25in]{geometry}

\allowdisplaybreaks

\journal{Applied and Computational Harmonic Analysis}

\makeatletter
\def\ps@pprintTitle{%
 \let\@oddhead\@empty
 \let\@evenhead\@empty
 \def\@oddfoot{}%
 \let\@evenfoot\@oddfoot}
\makeatother

\usepackage{amsmath}
\usepackage{amssymb}
\usepackage{bm}
\usepackage{mathrsfs}

\usepackage{algorithmic}
\usepackage{algorithm}

\usepackage{color}

\usepackage{url}

\usepackage{supertabular}

\DeclareMathAlphabet{\mathsfsl}{OT1}{cmss}{m}{sl}

\renewcommand{\phi}{\varphi}

\newcommand{\R}{\mathbb{R}}

\newcommand{\argmin}{\operatorname*{arg\; min}}
\newcommand{\argmax}{\operatorname*{arg\; max}}

\newcommand{\Expect}{\operatorname{\mathbb{E}}}

\newcommand{\Sin}{S_{\mathrm{in}}}
\newcommand{\Sout}{S_{\mathrm{out}}}

\newcommand{\rank}{\operatorname{rank}}

\newcommand{\diag}{\operatorname{diag}}

\newcommand{\dist}{\operatorname{dist}}

\newcommand{\bx}{\boldsymbol{x}}

\newcommand{\bZ}{\boldsymbol{Z}}

\newcommand{\bz}{\boldsymbol{z}}
\newcommand{\bX}{\boldsymbol{X}}

\newcommand{\bE}{\boldsymbol{E}}
\newcommand{\bF}{\boldsymbol{F}}

\newcommand{\bW}{\boldsymbol{W}}

\newcommand{\bP}{\boldsymbol{P}}
\newcommand{\bQ}{\boldsymbol{Q}}

\newcommand{\sX}{\mathcal{X}}

\newcommand{\bSigma}{\boldsymbol\Sigma}

\newcommand{\bU}{\boldsymbol{U}}
\newcommand{\bV}{\boldsymbol{V}}
\renewcommand{\algorithmicrequire}{\textbf{Input:}}
\renewcommand{\algorithmicensure}{\textbf{Output:}}

\def\psdleq{\preceq}
\def\psdgeq{\succeq}

\def\calX{\mathcal{X}}

\def\cO{\mathcal{O}}
\def\reals{\mathbb{R}}
\def\bx{\boldsymbol{x}}

\def\bu{\boldsymbol{u}}
\def\bb{\boldsymbol{b}}
\def\bm{\boldsymbol{m}}
\def\bS{\boldsymbol{S}}
\def\b0{\mathbf{0}}
\def\bP{\boldsymbol{P}}
\def\bQ{\boldsymbol{Q}}
\def\bSigma{\boldsymbol\Sigma}

\def\bU{\boldsymbol{U}}
\def\bR{\boldsymbol{R}}

\def\bv{\boldsymbol{v}}
\def\bA{\boldsymbol{A}}

\def\bB{\boldsymbol{B}}
\def\bg{\boldsymbol{g}}
\def\ba{\boldsymbol{a}}
\def\bw{\boldsymbol{w}}

\def\bdelta{\boldsymbol{\delta}}
\def\bzero{\boldsymbol{0}}

\def\bTheta{\boldsymbol{\Theta}}

\def\bI{\boldsymbol{I}}

\def\im{\mathrm{im}}

\def\tr{\mathrm{tr}}

\def\nin{n_{\mathrm{in}}}
\def\nout{n_{\mathrm{out}}}

\def\Span{\mathrm{Span}}

\def\Sp{\mathrm{Sp}}

\usepackage{algorithmic}
\usepackage{algorithm}
\newcommand{\di}{{\,\mathrm{d}}}

\def\scrG{\mathscr{G}}
\def\scrA{\mathscr{A}}
\providecommand{\bg}{\mathbf{g}}

\newtheorem{lemma}{Lemma}

\newtheorem{theorem}{Theorem}
\newtheorem{corollary}{Corollary}
\newtheorem{proposition}{Proposition}
\newtheorem{remark}{Remark}
\newtheorem{assumption}{Assumption}

\begin{document}

\begin{frontmatter}

\title{Global Convergence of Iteratively Reweighted Least Squares for Robust Subspace Recovery}

\author[umn]{Gilad Lerman}
\ead{lerman@umn.edu}
\author[ucf]{Kang Li}
\ead{kang.li@ucf.edu}
\author[brand]{Tyler Maunu}
\ead{maunu@brandeis.edu}
\author[ucf]{Teng Zhang}
\ead{teng.zhang@ucf.edu}
\affiliation[umn]{University of Minnesota}
\affiliation[ucf]{University of Central Florida}
\affiliation[brand]{Brandeis University}

\begin{abstract}
    Robust subspace estimation is fundamental to many machine learning and data analysis tasks. Iteratively Reweighted Least Squares (IRLS) is an elegant and empirically effective approach to this problem, yet its theoretical properties remain poorly understood. This paper establishes that, under deterministic conditions, a variant of IRLS with dynamic smoothing regularization converges linearly to the underlying subspace from any initialization. We extend these guarantees to affine subspace estimation, a setting that lacks prior recovery theory. Additionally, we illustrate the practical benefits of IRLS through an application to low-dimensional neural network training. Our results provide the first global convergence guarantees for IRLS in robust subspace recovery and, more broadly, for nonconvex IRLS on a Riemannian manifold.
\end{abstract}

\end{frontmatter}

\section{Introduction}
 
Many real-world datasets in computer vision, machine learning, and bioinformatics exhibit underlying low-dimensional structure that can be effectively modeled using subspaces~\citep{basri2003lambertian,hartley2003multiple,lerman2015robust,ding2020robust,li2022low,yu2024subspace,wang2024diffusion,arous2024high}. A variety of methods have been proposed to estimate such a subspace. The most prevalent of these is Principal Component Analysis (PCA) \citep{pearson1901liii,hotelling1933analysis,jolliffe1986principal}, which finds the directions of maximum variance within a dataset. However, PCA is well-documented to be sensitive to outliers.

These limitations spurred research into robust variants of PCA that can handle outliers; see \cite{lerman2018overview} for a comprehensive review. We focus specifically on robust subspace recovery (RSR), where the goal is to identify a low-dimensional subspace that explains a subset of the data points (inliers) while ignoring the effect of corrupted data (outliers). Unlike classical PCA, which assumes Gaussian-distributed noise, RSR explicitly models the data as a mixture of inliers that lie near or on a low-dimensional subspace and outliers that may be arbitrarily positioned.

Formally, we consider a dataset of $n$ observations in $\R^D$, denoted by the multiset $\calX$. This multiset consists of an inlier multiset and an outlier multiset, $\calX=\calX_{\mathrm{in}}\cup\calX_{\mathrm{out}}$, where the inliers lie on a $d$-dimensional subspace $L_\star$. For clarity, we focus on the noiseless inlier setting in this paper. The RSR problem seeks to recover $L_\star$ given only $\calX$.

In the following, we focus on global convergence. That is, we say that an iterative algorithm \emph{globally recovers} a point if it converges to that point from any initialization. To accomplish this, we revisit the IRLS method for RSR from \cite{lerman2018fast}, known as the Fast Median Subspace (FMS) algorithm. 
FMS seeks to minimize a robust, nonconvex, least absolute deviations energy function over the Grassmannian manifold $\scrG(D,d)$, which is the set of $d$-dimensional linear subspaces in $\R^D$. 

The focus of this paper is efficient solutions to the RSR problem when the subspace $L_\star$ and the dataset $\calX$ satisfy certain conditions. We say that a subspace $L$ and dataset $\calX$ satisfy \emph{deterministic conditions} if there exist nonrandom functions $g_i$ such that $g_i(\calX_{\mathrm{in}}, \calX_{\mathrm{out}}, L) \geq 0$, $i=1, \dots, k$. These functions encode conditions that compare the statistics of the inliers and outliers, and we formulate our specific conditions later in Section \ref{sec:mainresults}. At a high level, such conditions end up being sufficient for solving the RSR problem efficiently, while in general it is NP-hard \citep{clarkson2015input}. 

The rest of this section is organized as follows: Section \ref{subsec:contributions} presents the primary contributions of our work. Then, Section \ref{subsec:paperorg} discusses the organization of the paper, and Section \ref{subsec:notation} introduces the necessary notation.

\subsection{Contributions}
\label{subsec:contributions}

The goals of our work are threefold: (i) to develop strong theoretical guarantees for FMS, (ii) to generalize these guarantees to affine subspaces,  and (iii) to provide further practical motivation for FMS.  We have made significant progress in all three directions. In particular, our key contributions are:
\begin{enumerate}
    \item Under deterministic conditions on $\calX$ alone, FMS with a dynamically decreasing smoothing linearly converges to $L_\star$ from arbitrary initialization. 
    To the best of our knowledge, this is the first global convergence result for IRLS on a Riemannian manifold in a nonconvex setting. As a corollary, under the same deterministic conditions on $\calX$ alone, the FMS algorithm with fixed regularization reaches a point close to $L_\star$ from arbitrary subspace initialization when its fixed smoothing parameter is initialized by the data quantile. 
    \item We extend FMS to the estimation of affine subspaces. For this new algorithm, we establish local linear convergence under a modified set of deterministic conditions on $\calX$ and the initialization.
    \item We run numerical experiments that demonstrate the advantages of dynamic smoothing for robust subspace recovery in adversarial settings.
    \item We demonstrate the practical utility of FMS through an application to low-dimensional neural network training.
\end{enumerate} 
These results address a long-standing question of providing broad theoretical guarantees that align with the strong empirical performance of FMS, while also enabling new applications in modern machine learning. They are likely to be of broad interest to the machine learning community. In particular, IRLS is a specific instance of a more general paradigm in optimization known as \emph{Majorize-Minimization (MM) algorithms}. This class of algorithms is based on upper bounding a complex objective function with a surrogate function whose minimum is simpler to compute. One can generate a sequence whose objective values monotonically decrease by iteratively minimizing the surrogate. MM algorithms have been widely used in signal processing and machine learning \citep{SIG-053,sun2016majorization}, and can provide insight into machine learning building blocks like transformers \citep{yang2022transformers}.

{ A key algorithmic and theoretical innovation is the use of \emph{dynamic smoothing} within the FMS algorithm. The application of these ideas within IRLS began with \citep{daubechies2010iteratively} in the context of sparse signal recovery.
Recent works have applied these ideas to IRLS for convex penalty subproblems \citep{burke2015iterative}, projection onto 
$\ell_p$ balls \citep{yang2022towards}, and robust regression
\citep{mukhoty2019globally,peng2023convergence,peng2022global}. Choosing the correct smoothing schedule leads to provable linear and super-linear convergence. Crucially, the use of dynamic smoothing allows one to show linear convergence of the IRLS algorithm by ensuring that the weights do not blow up too quickly while still being allowed to go to infinity.
}

Furthermore, our algorithms are specific instances of \emph{Riemannian optimization algorithms}, which operate when the constraint set is a Riemannian manifold. There has been broad interest in studying properties of these methods \citep{absil2007optimization_book,boumal2023introduction} and their practical application, for example, see \cite{Vandereycken13_completion_riemannian,jaquier2020bayesian,Hu20_manifold_optimization}.

Finally, there has also been widespread interest in analyzing the convergence of nonconvex optimization algorithms, which is notoriously difficult \citep{danilova2022recent}. Fortunately, in many settings, nonconvex problems exhibit benign properties \citep{li2019symmetry} that enable favorable algorithmic behavior \citep{maunu2019well}. { Most existing convergence analyses for nonconvex optimization on manifolds focus on methods that take \emph{explicit} steps along the Riemannian gradient (and sometimes the Hessian) of a smooth objective (e.g., gradient descent, trust-region, and related variants). In contrast, we analyze an IRLS scheme that updates $L$ by minimizing a quadratic majorizer $G_\epsilon(L,L_k)$. While the weights in $G_\epsilon(\cdot,L_k)$ implicitly encode first-order information of the smoothed energy $F_\epsilon$ at $L_k$, the surrogate is not the second-order Taylor model of $F_\epsilon$: it matches $F_\epsilon$ and its first-order behavior at the base point but induces a different curvature away from it. We provide a global convergence result for this nonconvex IRLS algorithm on a Riemannian manifold under a deterministic condition.
}

\subsection{Paper Organization}
\label{subsec:paperorg}

In Section \ref{sec:background}, we rigorously define the robust subspace recovery problem, provide the necessary mathematical background, and review related work. In Section~\ref{sec:method}, we propose our dynamic smoothing scheme in the FMS algorithm and its variant for estimating affine subspaces. In Section~\ref{sec:theoreticalanalysis}, we establish the theoretical guarantees for the proposed method. In Section~\ref{sec:numerics}, we run simulations on synthetic data that demonstrate the competitiveness of FMS with dynamic smoothing, and we also present the effective application of FMS in dimension-reduced training of neural networks.

\subsection{Notation}
\label{subsec:notation}

For a finite multiset $\calX$, $|\calX|$ denotes the number of points. We use bold uppercase letters for matrices and bold lowercase letters for column vectors. Let $\bI_k$ denote the identity matrix in $\mathbb{R}^{k \times k}$, where if $k$ is obvious, we simply write $\bI$. For $d \leq D$, $\cO(D,d)$ is the set of semi-orthogonal matrices in $\mathbb{R}^{D \times d}$, i.e., $\bU \in \cO(D,d)$ if and only if $\bU^\top \bU = \bI_{d}$.  
For a $d$-dimensional linear subspace $L$, we denote by $\bP_L$ the $D\times D$ matrix representing the orthogonal projector onto $L$, and $\bQ_L = \bP_{L^\perp}$ the projection onto the orthogonal complement of $L$. Projection matrices are symmetric and satisfy $\bP^2_L=\bP_L$ and $\im(\bP_L)=L$, where $\im(\cdot)$ denotes the column space of a matrix. For a matrix $\bA$, $\|\bA\|_2$ denotes its operator or spectral norm, and $\|\bA\|_{F}$ denotes its Frobenius norm. Further, $\sigma_j(\bA)$ will denote the $j$th singular value of $\bA$.  For our inlier-outlier dataset $\calX$, we let $n = |\calX|$, $\nin = |\calX_{\mathrm{in}}|$, and $\nout = |\calX_{\mathrm{out}}|$. Writing $f(x) \gtrsim g(x)$ means that $f(x) \geq c g(x)$ for some absolute constant $c$.

\section{Problem Setup and Related Work}\label{sec:background}

This section sets the stage for our algorithm and our theoretical results. In Section \ref{subsec:background}, we rigorously define the problem we seek to solve and provide the necessary mathematical background. After this, we review related work in Section \ref{subsec:relatedwork}.

\subsection{Problem Formulation and Background}
\label{subsec:background}

We briefly summarize the necessary background to understand the primary contributions of this work, building upon the setup provided in the introduction. For simplicity, we formulate the problem for linear subspaces. The treatment of affine subspaces is discussed in Section \ref{sec:affineIRLS}.

Let $\dist(\bx, L) = \|\bQ_L \bx\|$ be the orthogonal distance between $\bx \in \calX$ and $L \in \scrG(D,d)$, where $\bQ_L$ is the orthogonal projection onto the orthogonal complement of $L$. This paper investigates an IRLS approach to a nonconvex least absolute deviations formulation for RSR~\citep{ding2006r, Zhang_MKF09, maunu2019well}. This formulation minimizes the sum of distances from the data points to any $d$-dimensional subspace in $\scrG(D,d)$:
\begin{equation}\label{eq:problem2}
\hat{L}=\argmin_{L \in \scrG(D,d)}F(L):= \sum_{\bx\in\calX}\dist(\bx, L),
\end{equation}
This is an optimization problem over the Grassmannian manifold \citep{edelman1998geometry}. In fact, if we minimize the sum of squared distances $\sum_{\bx\in\calX}\dist^2(\bx,L)$, then we obtain the principal subspace. By using the non-squared distance instead, our estimator is less prone to the impact of large outliers. 

An IRLS method for \eqref{eq:problem2} iteratively applies weighted PCA. More specifically, at iteration $k$ its unregularized version minimizes $\sum_{\bx\in\calX} w_{\bx} \dist^2(\bx,L)$ with weights $w_{\bx}=1/\dist(\bx,L^{(k)})$ to find the updated subspace $L^{(k+1)}$. Here, $L^{(k)}$ is the subspace estimated in the previous iteration. While this method is elegant and widely used, it lacks global convergence and exact recovery guarantees, and this paper aims to fill this gap. Nevertheless, our theoretical analysis requires some modification of this simple procedure (see discussion of dynamic smoothing in Section \ref{sec:method}). This analysis advances the state of nonconvex IRLS for Riemannian manifolds, specifically the Grassmannian $\scrG(D,d)$. Furthermore, it is natural to extend \eqref{eq:problem2} by replacing $\scrG(D,d)$ with the set of affine subspaces. This is the subject of our algorithmic development in Section \ref{sec:affineIRLS}.

The analysis in this paper depends on the principal angles and vectors between subspaces as follows: For any $L_0, L_1\in \scrG(D,d)$, we denote their $d$  principal angles by $\{\theta_j(L_0,L_1)\}_{j=1}^d$ and let the respective principal vectors for $L_0$ and $L_1$ be $\{\bu_j(L_0,L_1)\}_{j=1}^d$ and $\{\bv_j(L_0,L_1)\}_{j=1}^d$. 
Given bases $\bW_0 \in \cO(D,d)$ for $L_0$ and $\bW_1 \in \cO(D,d)$ for $L_1$, one can compute the principal vectors and angles by the singular value decomposition $\bW_0^\top \bW_1 = \bR_0 \cos(\bTheta) \bR_1^\top$, in which case $\bTheta$ contains the principal angles, $\bW_0 \bR_0$ contains the vectors $\bu_j$ and $\bW_1 \bR_1$ contains the vectors $\bv_j$. See \citet[Section 3.2.1]{lerman2014lp} for a more detailed treatment.

\subsection{Review of Related Work}
\label{subsec:relatedwork}
 In the following, we separately discuss related work in robust subspace recovery, iterative reweighted least squares, and the least
absolute deviations-based RSR problem  \eqref{eq:problem2}.

\subsubsection{Robust Subspace Recovery (RSR)} 

RSR has been extensively studied in various contexts, with methods falling into several key categories. For a comprehensive review, see \citet{lerman2018overview}. The distinction between our problem and robust PCA (RPCA) \citep{candes2011robust} should be noted. While both deal with corrupted low-rank structures, RPCA assumes elementwise corruptions throughout the data matrix, whereas RSR assumes some samples are wholly corrupted. Consequently, algorithms designed for one problem typically do not perform well on the other.

A significant direction involves robust covariance estimation. Methods like Tyler's M-estimator (TME) \citep{tyler1987distribution,zhang2016robust} aim to robustly estimate an underlying covariance matrix from which the principal subspace can be extracted. These approaches have strong theoretical foundations regarding breakdown points, though their guarantees for subspace recovery are more limited in natural special settings. Stronger local guarantees were obtained for a subspace-constrained TME variant \citep{yu2024subspace,lerman_zhang2024}.

Many approaches frame RSR as an outlier or inlier identification problem, where the goal is first to remove outliers or find a set of inliers before applying standard PCA. Examples include RANSAC-based methods \citep{fischler1981random, ariascastro2017ransac, maunu2019robustsubspacerecoveryadversarial} and more recent approaches like Coherence Pursuit \citep{rahmani2017coherence} and Thresholding-based Outlier Pursuit \citep{cherapanamjeri2017thresholding}. While conceptually simple, these methods often require careful parameter tuning and may struggle with high outlier percentages and higher intrinsic subspace dimensions.

\subsubsection{Least
Absolute Deviations-based RSR} 

In this work, we focus on a specific approach to solving the RSR problem that relies on \emph{least absolute deviations}, i.e., the objective function in \eqref{eq:problem2}. This has been a focus of a major line of work, where one replaces the squared Euclidean distance in PCA with the absolute distance to achieve robustness. While this problem is computationally hard \citep{hardt2013algorithms,clarkson2015input}, several methods have been proposed to approximately solve it. These include convex relaxations like Outlier Pursuit \citep{xu2012robust} and REAPER \citep{lerman2015robust}, as well as nonconvex approaches like geodesic gradient descent \citep{maunu2019well}, which can be easily replaced with projected subgradient descent (PSGM) \citep{NEURIPS2018_af21d0c9}. The problem may also be addressed using Riemannian ADMM \citep{li2024riemannian}, a general framework for manifold optimization with nonsmooth objectives. In addition, PSGM and manifold proximal point algorithm \citep{9048840} have also been proposed for Dual Principal Component Pursuit (DPCP), which can be considered as a special case of \eqref{eq:problem2}, where
$d$ is replaced by $D-d$  \citep{NEURIPS2018_af21d0c9}.  Indeed, the
analysis of these methods is analogous to that of
\citet{maunu2019well}. We remark that the analysis pursued in \citet{NEURIPS2018_af21d0c9} only
holds for the special case $d=D-1$. It was later generalized for any $d$ in \citet{ding2021dual}, again with similar ideas to \citet{maunu2019well}, whose result holds for any subspace dimension.

This work investigates the FMS method proposed by \cite{lerman2018fast}, which solves \eqref{eq:problem2} using IRLS. This method is popular due to its simplicity and strong empirical performance. More recently, FMS has also been applied to hierarchical representations \citep{ding2024adacontour}. Another work that can be viewed as IRLS with 0-1 weighting is \cite{cherapanamjeri2017thresholding}.

However, the theoretical analysis of the FMS algorithm is somewhat limited, as it generally guarantees convergence only to stationary points.  In contrast, \cite{maunu2019well} investigates the landscape of the objective function and obtains a theoretical understanding of its associated gradient descent algorithm. However, they only use this analysis to prove local convergence of gradient descent, which requires significant parameter tuning. Also, until now, it has been unclear how this analysis can be extended to other optimization frameworks like IRLS, which we will discuss next.

\subsubsection{Iteratively Reweighted Least Squares (IRLS)} 

The method of IRLS is used to solve specific optimization problems with objective functions that are sums of loss functions applied to residuals \citep{holland1977robust}. In each step the method solves a weighted least squares problem.
This idea has been applied in numerous problems, such as geometric median \citep{weiszfeld1937point,Kuhn:1973vr}, robust regression \citep{holland1977robust,peng2022global}, M-estimator in robust statistics \citep{maronna2006robust}, $\ell_q$ norm-based sparse recovery \citep{daubechies2010iteratively,kummerle2020iteratively}, low-rank matrix recovery \citep{fornasier2011lowrank}, synchronization \citep{chatterjee2017robust}, and more. In particular, IRLS algorithms have also been used to solve the convex relaxation of \eqref{eq:problem2} in \citet{JMLR:v15:zhang14a,lerman2015robust}. { More recently, \cite{peng2023block} discuss the connection between IRLS and block-coordinate descent. Their general framework gives convergence to stationary points, but it does not give quantitative rates like those obtained in our work.}

{ 

A central application that shaped the modern development of iteratively reweighted least squares (IRLS) is sparse recovery, where one seeks a sparse vector by minimizing an $\ell_p$ penalty with $0<p\le 1$. Early influential instances include FOCUSS and related reweighting schemes \citep{gorodnitsky2002sparse} as well as affine-scaling methods \citep{rao2002affine}. Subsequent work began to clarify why IRLS succeeds for nonconvex sparse objectives and how to control the tendency of weights to blow up near zeros via smoothing and continuation strategies \citep{chartrand2008iteratively,daubechies2010iteratively}. These mechanisms enabled increasingly sharp convergence analyses and recovery guarantees alongside algorithmic refinements \citep{ba2013convergence,fornasier2016conjugate}. More recently, IRLS variants for basis pursuit and related convex sparse programs have been shown to enjoy global linear convergence under appropriate dynamic smoothing rules \citep{kummerle2020iteratively}.

}
 
Despite their strong empirical performance, straightforward IRLS algorithms generally lack strong theoretical guarantees. To illustrate this, for the geometric median problem \citep{weiszfeld1937point,Beck:2015vn}, which minimizes  
$F(\bx) = \sum_{i=1}^n \|\bx - \bx_i\|$, 
an IRLS method iteratively reweights the data points with weights  
$w_i^{(k)} = {1}/{\|\bx^{(k)} - \bx_i\|}$. 
If $\bx^{(k)} = \bx_1$, then $\bx^{(k+1)} = \bx_1$, as the weight for $\bx_1$ becomes infinite. 
\citet{daubechies2010iteratively} indicate the same problem of infinite weight in IRLS for sparse recovery. To address this issue, a common strategy is to introduce a smoothing parameter that bounds the weights from above, like choosing $w_i^{(k)} =1/(\epsilon^2+\|\bx_i-\bx\|^2)^{1/2}$ or $w_i=1/\max(\epsilon,\|\bx_i-\bx\|)$. This strategy is used heuristically in \citet{JMLR:v15:zhang14a,lerman2015robust,lerman2018fast}. Still, as pointed out in \citet{peng2022global}, this method only converges to an $\epsilon$-approximate solution, and no theory exists on its global convergence rate.

Researchers have used different insights to reach a consensus on dynamically updating the weights \citep{daubechies2010iteratively,fornasier2011lowrank}. Recent advances in IRLS \citep{kummerle2020iteratively,peng2022global} suggest applying a dynamic smoothing-parameter strategy that adaptively selects the regularization parameter based on the current weights.  
Under this strategy, they prove that the IRLS algorithm converges linearly to the solution of the original problem. { However, we note that such an analysis of IRLS has not been applied to manifold optimization problems such as RSR. Furthermore, the only theoretical guarantees in nonconvex optimization that arise in $\ell_p$ minimization for $p<1$ yield a local result in which the initialization must be very close to the solution. Initial results required that the initializer be very close to the solution \citep{daubechies2010iteratively}, while some more recent results have improved and extended these results \citep{kummerle2018harmonic,kummerle2021scalable,peng2022global,peng2023convergence}.
}

Another work that examines dynamic smoothing on manifolds is \cite{beck2023dynamic}, where the authors prove a sublinear convergence rate for a gradient method to a stationary point.

\section{FMS: Review and Extension to Affine Subspaces}\label{sec:method}

In this section, we outline our algorithmic contributions to the RSR problem. In Section \ref{subsec:IRLS_derivation}, we review the FMS algorithm and discuss the addition of dynamic smoothing. Then, in Section \ref{sec:affineIRLS}, we discuss a novel affine variant of the FMS algorithm.

\subsection{The FMS Algorithm}
\label{subsec:IRLS_derivation}

Here, we study the iteratively reweighted least-squares method (IRLS) to solve \eqref{eq:problem2}, which alternates between finding the top eigenvectors of a weighted covariance and updating the weights. We recall that the unregularized IRLS method at the $k$-th iteration is the solution to the weighted squared problem of
\begin{equation}
L^{(k+1)}= \argmin_{L}\sum_{\bx\in\calX}w^{(k)}_{\bx}\dist(\bx,L)^2,\,\,\,\text{where $w^{(k)}_{\bx}=\frac{1}{\dist(\bx,L^{(k)})}$}.
\end{equation}
This solution coincides with the standard PCA procedure:
\begin{equation}\label{eq:IRLS}
L^{(k+1)}=T(L^{(k)})=\text{span of top $d$ eigenvectors of $\sum_{\bx\in\calX} w_{\bx}^{(k)} \bx\bx^\top$}.
\end{equation}
To put forward an IRLS algorithm that gives a re-weighting without infinite components in the weight, it is natural to regularize the weights to be $w^{(k)}_{\bx}=1 / \max(\dist(\bx,L^{(k)}),\epsilon_k)$.  The FMS algorithm by  \citet{lerman2018fast} adopts a fixed regularization 
$\epsilon_k$ for all $k\geq 0$. 

In this work, we propose to choose $\epsilon_k$ in each step adaptively by incorporating a ``Dynamic Smoothing'' procedure. Let $\gamma$ denote a hyperparameter, which can be related to the inlier percentage (see Section \ref{sec:theoreticalanalysis}). %
For a nonempty finite multiset $\mathcal Z=\{z_i\}_{i=1}^m$ of real numbers, let $z_{(1)}\leq\cdots\leq z_{(m)}$ be its order statistics. We use the upper empirical quantile convention
\[
q_\eta(\mathcal Z)=
\begin{cases}
z_{(\lfloor\eta m\rfloor+1)}, & 0\leq\eta<1,\\
z_{(m)}, & \eta=1.
\end{cases}
\]
In particular, for $0\leq\eta\leq1$,
\[
|\{i:z_i<q_\eta(\mathcal Z)\}|\leq\eta m,
\qquad
|\{i:z_i\leq q_\eta(\mathcal Z)\}|\geq\eta m.
\]
For the dynamic smoothing update, we apply $q_\gamma$ to the multiset $\{\dist(\bx,L^{(k)})\}_{\bx\in\calX}$.
We then set the dynamic smoothing parameter
\begin{equation}\label{eq:epsilon_choose}
 \epsilon_k=\min(\epsilon_{k-1}, q_\gamma(\{\dist(\bx,L^{(k)})\}_{\bx\in\calX})),
{ \text{ where } \epsilon_{-1}:=\infty}.
 \end{equation}
 This choice, inspired by \citet{kummerle2020iteratively} and generalized to the subspace recovery setting, ensures that
$\epsilon_k$ forms a nonincreasing sequence while remaining sufficiently large relative to the distances ${\dist(\bx,L^{(k)})}_{\bx\in\calX}$. We summarize the regularized IRLS in Algorithm~\ref{alg:IRLS}, which we refer to as Fast Median Subspace with Dynamic Smoothing (FMS-DS).

Dynamic smoothing allows $\epsilon$ to decrease over the iterations. This leads to a procedure that can recover a solution to the unregularized problem by solving a sequence of regularized problems. This is made rigorous in Section \ref{sec:theoreticalanalysis}.

The FMS-DS algorithm requires two inputs: an initial guess for the subspace $L^{(0)}$, and a hyperparameter $\gamma$. This hyperparameter can be estimated using prior knowledge or cross-validation. Its choice for specific inlier-outlier models is discussed in Section \ref{sec:theoreticalanalysis}.

\begin{algorithm}%
\caption{{{FMS-DS}
}}  
\label{alg:IRLS}
\begin{flushleft} 
  {\bf Input:}  Set of observations $\calX\subset\reals^D$, initial estimated subspace $L^{(0)}\subset \scrG(D,d)$, $\gamma$: hyperparameter.\\ 
  {\bf Output:} An estimate of the underlying subspace.\\
  {\bf Steps:}\\
    {\bf 1:} Let $k=0$. \\
  {\bf 2:} Choose $\epsilon_k$ suitably from  $L^{(k)}$ and $\calX$ by \eqref{eq:epsilon_choose}.\\ 
  {\bf 3:} $L^{(k+1)}=T_{\epsilon_k}(L^{(k)})$, where   \begin{equation}\label{eq:IRLS2}
T_{\epsilon}(L)=\text{span of top $d$ eigenvectors of \ $\sum_{i=1}^n\frac{\bx_i\bx_i^\top}{\max(\dist(\bx_i,L),\epsilon)}$}.
\end{equation}
  \\
  {\bf 4:} $k=k+1$.\\ 
  {\bf 5:} 
  Repeat Steps 2-4 until convergence.
  \end{flushleft} 
  \end{algorithm}

\subsection{The Affine FMS Algorithm}\label{sec:affineIRLS}

Most existing works on RSR (see Section~\ref{subsec:relatedwork}) focus on recovering a robust linear $d$-dimensional subspace in $\reals^D$. While our main theoretical results concern this linear case, our objective function~\eqref{eq:problem2} and its IRLS algorithm naturally extend to the affine setting, allowing for robust local recovery of an affine $d$-dimensional subspace.

Here we use $\scrA(D,d)$ to denote the set of all $d$-dimensional affine subspaces in $\reals^D$. We distinguish between a representative pair $A = (L, \bm) \in \scrG(D,d) \times \reals^D$, and the affine subspace $[A] = \{ \bm + \bx : \bx \in L \}$, which is the affine subspace formed by translating the $d$-dimensional linear subspace $L$ by $\bm$. The representative pair is not unique, since $(L, \bm)$ and $(L, \bm')$ refer to the same affine subspace whenever $\bm' = \bm + \bP_L \bdelta$ for some $\bdelta \in \reals^D$. In fact, $[A]$ is an equivalence class of representative pairs. For any $\bx \in \reals^D$, the distance to $[A]$ is given by $\dist(\bx, [A]) = \| \bQ_L (\bx - \bm) \|$, which is invariant under the choice of representative. Thus, we sometimes also write $\dist(\bx, A) = \| \bQ_L (\bx - \bm) \|$.

We consider the following natural generalization of \eqref{eq:problem2}:
\begin{equation}\label{eq:problem_affine}
    [\hat{A}]=\argmin_{[A] \in \scrA(D,d)}F^{(a)}([A]):=\sum_{\bx\in\calX}\dist(\bx,[A])= \sum_{\bx\in\calX}\|\bQ_L (\bx-\bm)\|.
\end{equation}
This problem is invariant to the choice of representation, and so we rewrite it in parametrized form as
\begin{equation*}
    \hat{A}=\argmin_{A \in \scrG(D,d) \times \R^D }F^{(a)}(A):=\sum_{\bx\in\calX}\dist(\bx,A)= \sum_{\bx\in\calX}\|\bQ_L (\bx-\bm)\|.
\end{equation*}

To solve this problem, we generalize Algorithm~\ref{alg:IRLS} based on a straightforward generalization of the argument in Section~\ref{subsec:IRLS_derivation}. The affine IRLS method iterates by solving the weighted least squares problem at the $k$th iteration, where we replace distances to linear subspaces with distances to affine subspaces. The affine setting of the FMS algorithm can thus be written as
\begin{equation}\label{eq:IRLS_affine}
A^{(k+1)}=T^{(a)}_{\epsilon_k}(A^{(k)}):= \argmin_{A \in \scrA(D,d)}\sum_{\bx\in\calX}w^{(k)}_{\bx}\dist(\bx,A)^2,\,\,\,\text{where $w^{(k)}_{\bx}=\frac{1}{\max(\dist(\bx,A^{(k)}), \epsilon_k)}$}.
\end{equation}
The solution can be obtained from the standard PCA procedure applied to $\bx\in\calX$ with weights $1/\max(\dist(\bx,A^{(k)}),\epsilon_k)$. Therefore, $T^{(a)}_\epsilon(A)=(L,\bm)$ is such that $\bm= \sum_{\bx\in\calX}w_{\bx}\bx / \sum_{\bx\in\calX}w_{\bx}$, where $w_{\bx}= 1/\max(\dist(\bx,A),\epsilon)$, and $L$ is the span of top $d$ eigenvectors of $\sum_{\bx\in\calX}w_{\bx}(\bx-\bm)(\bx-\bm)^\top$. 

Similar to the linear case in \eqref{eq:epsilon_choose}, we choose the regularization parameters using dynamic smoothing. In this case, the update for $\epsilon$ becomes
\begin{equation}\label{eq:epsilon_choose_affine}
 \epsilon_k=\min(\epsilon_{k-1}, q_\gamma(\{\dist(\bx,A^{(k)})\}_{\bx\in\calX})).
 \end{equation}
Again, $\gamma$ is a hyperparameter that we discuss more in our theoretical results. We refer to the resulting algorithm as Affine Fast Median Subspace with Dynamic Smoothing (AFMS-DS). The method with fixed regularization is just referred to as Affine FMS (AFMS).

\begin{remark}
    AFMS is rotation and translation equivariant, in the sense that if $T_{\epsilon}^{(a)}(A) = (L,\bm)$, then $T_{\epsilon}^{(a)}(\bR A + \ba) = (\bR L,\bR \bm + \ba)$ for any $\bR^\top \bR = \bI_D$ and $\ba \in \R^D$. Furthermore, FMS is a special case of AFMS with the mean parameter fixed to be zero.
\end{remark}

\section{Theoretical Analysis}\label{sec:theoreticalanalysis}

This section establishes theoretical guarantees for the (A)FMS-DS algorithms by showing that the algorithms recover an underlying linear or affine subspace under specific inlier-outlier models. We emphasize that, in the case of FMS-DS, we achieve a global result, namely, the algorithm converges to the underlying subspace from any initialization. The main theorem for FMS-DS is presented in Section~\ref{sec:mainresults}, with further discussion of this result presented in Sections \ref{subsubsec:assump_discuss} and \ref{subsubsec:theory_ext}. We then establish a local convergence theorem for AFMS-DS in Section~\ref{sec:affine_analysis}.

\subsection{Main Result for Linear Subspace Recovery}\label{sec:mainresults}

This section establishes a linear convergence guarantee for Algorithm~\ref{alg:IRLS}. The main result, Theorem~\ref{thm:global}, shows that the algorithm linearly converges to an underlying subspace from any initialization if certain conditions on the data are satisfied. 

To state our theorem, we first define a regularized objective $F_{\epsilon}(L)$:
\begin{align}\label{eq:Fepsilon}
F_{\epsilon}(L)&=\sum_{\bx\in\calX:\dist(\bx,L)>\epsilon}\dist(\bx,L) + \sum_{\bx\in\calX:\dist(\bx,L)\leq \epsilon}\left(\frac{1}{2}\epsilon+\frac{\dist(\bx,L)^2}{2\epsilon}\right).
\end{align}
This is a natural extension of \eqref{eq:problem2} as we view it as the objective function for IRLS with a fixed regularization parameter \citep{JMLR:v15:zhang14a,lerman2015robust,daubechies2010iteratively,lerman2018fast,kummerle2020iteratively,peng2022global}. 

Throughout this section, we consider the noiseless inlier--outlier model
\[
\calX=\calX_{\mathrm{in}}\cup\calX_{\mathrm{out}},
\qquad \calX_{\mathrm{in}}\subset L_\star,
\qquad L_\star\in\scrG(D,d).
\]
Thus, every inlier lies exactly on the underlying $d$-dimensional linear subspace $L_\star$. Next, we introduce the assumptions needed for our main theorem; we discuss them further in Section~\ref{subsubsec:assump_discuss}. These assumptions are restrictions on the inlier--outlier dataset $\calX$. The first assumption requires that low-dimensional subspaces that are not $L_\star$ may not contain a significant number of points.

\begin{assumption}\label{assump:lowerdim}
    Let $\gamma_\star = |\calX_{\mathrm{in}}| / |\calX|$ be the percentage of inliers and $\gamma$ be the hyperparameter for FMS-DS. For all ($d-1$)-dimensional linear subspaces $L_0\subset L_\star$ and $d$-dimensional linear subspaces $L\neq L_\star$, $|\calX \cap (L_0\cup L)|/|\calX| < \gamma \leq \gamma_\star/2$.
\end{assumption} 

\begin{remark}
    Assumption \ref{assump:lowerdim} can be weakened by instead assuming that $L=L_\star$ is the unique subspace such that $q_{\gamma}(\{\dist(\bx,L)\}_{\bx\in\calX})=0$. However, in this case, we can only establish convergence to $L_\star$, and the linear convergence rate guarantee no longer holds. 
\end{remark}

The second assumption requires the definition of two summary statistics $\Sin$ and $\Sout$. 
\begin{align}\label{eq:Sin}
\Sin&=\min_{\bu\in L_\star:\|\bu\|=1}\sum_{\bx\in\calX_{\mathrm{in}}}|\bu^\top\bx|, \ 
\Sout=\sup_{\substack{L\in\scrG(D,d)\\\epsilon>0}}\left\|\sum_{\bx\in\calX_{\mathrm{out}}}\frac{\bP_L\bx\bx^\top\bP_{L^\perp}}{\max(\dist(\bx,L),\epsilon)}\right\|_2.
\end{align}
Heuristically, the statistic $\Sin$ measures how well-spread the inliers are on the underlying subspace, while the statistic $\Sout$ shows how well-aligned the outliers are along any $d$-dimensional subspace. Both statistics are somewhat natural. A large $\Sin$ prevents inliers from lying on or close to low-dimensional subspaces of $L_\star$. The statistic $\Sin$ is zero when there is a direction in $L_\star$ that is orthogonal to all inliers, and it is large when every direction in $L_\star$ has many inliers significantly correlated with it. On the other hand, having a small $\Sout$ is also natural, since it prevents outliers from being too low-dimensional. We remark that the $\Sin$ statistic is the same as the permeance statistic in \cite{lerman2015robust}, and the $\Sout$ statistic extends the alignment statistic of \cite{maunu2019well} to all smoothing parameters. 

In line with the discussion in the previous paragraph, the second assumption ensures that the outlier statistic $S_{\mathrm{out}}$ is sufficiently smaller than the inlier statistic $S_{\mathrm{in}}$ and an additional regularized inlier statistic. This holds in a wide variety of cases, see the discussions in \cite{maunu2019well,maunu2019robustsubspacerecoveryadversarial} as well as Section \ref{subsubsec:assump_discuss}.
\begin{assumption}\label{assump:global}
For the inlier-outlier dataset $\calX$, there exists a $\theta_0 \in (0,\pi/2)$, the following deterministic condition is  satisfied: 
{\begin{equation}\label{eq:Singlobal}\ \cos\theta_0 \Sin\geq {3}\sqrt{d}\Sout.
\end{equation} }
\end{assumption}

{
We also impose the following assumption, which is necessary to ensure that the iterates remain within a bounded neighborhood around $L_\star$ for $k \geq 1$. This assumption depends on a fixed parameter $\alpha_0$ in $(0,\gamma/\gamma_\star)$. For simplicity,  in the concrete examples presented in Section~\ref{subsubsec:assump_discuss} we set  
$\alpha_0 = 0.5\cdot\gamma/\gamma_\star.$ 

\begin{assumption}\label{assump:global2}
There exists $\alpha_0 \in (0,\gamma/\gamma_\star)$ such that the following conditions hold for all subspaces $L \in \scrG(D, d)$:

\begin{itemize}
    \item Dominance of a quantile of the inlier distances over their mean: There exists $\beta_1>0$ such that 
    \begin{equation}\label{eq:global21}
     \beta_1 \cdot   q_{\alpha_0}(\{\dist(\bx, L)\}_{\bx \in \calX_{\mathrm{in}}}) \geq \mathrm{mean}(\{\dist(\bx, L)\}_{\bx \in \calX_{\mathrm{in}}}).
    \end{equation}

    \item Possible dominance of a quantile of the outlier distances over a quantile of the inlier norms: 
    There exists $\beta_2\geq 1$ such that if  $\gamma \leq 1-(1-\alpha_0)\gamma_\star$, then 
    \begin{equation}\label{eq:global22}
        \beta_2 \cdot q_{\frac{\gamma-\alpha_0\gamma_\star}{1 - \gamma_\star}}
        \left(\{\dist(\bx, L)\}_{\bx \in \calX_{\mathrm{out}}}\right)
        \geq
        q_{\alpha_0}\left(\left\{\|\bx\|\right\}_{\bx \in \calX_{\mathrm{in}}}\right).
    \end{equation}
Here, $q_{\alpha_0}(\{\|\bx\|\}_{\bx\in\calX_{\mathrm{in}}})$ is the empirical $\alpha_0$-quantile of the inlier norms, using the convention in Section~\ref{subsec:IRLS_derivation}.
If $\gamma> 1-(1-\alpha_0)\gamma_\star$  or $\calX_{\mathrm{out}}$ is empty, then \eqref{eq:global22} cannot hold, since the quantile percentage is larger than 1. In this case, we still define $\beta_2=1$ and use this parameter in the next condition.

    \item Dominance of the inliers over the outliers in a spectral sense:  
  For  $\theta_0$ defined in Assumption~\ref{assump:global} and $\beta :={10\beta_1}{\beta_2}$,       \begin{equation}\label{eq:global23}
        \sin \theta_0 \cdot 
        \min_{\calX_{\mathrm{in}}' \subseteq \calX_{\mathrm{in}} \,:\, |\calX_{\mathrm{in}}'| \geq \alpha_0|\calX_{\mathrm{in}}|} 
        \sigma_d\left(\sum_{\bx \in \calX_{\mathrm{in}}'} \bx \bx^\top \right)
        \geq 2\beta \cdot \sigma_1\left(\sum_{\bx \in \calX_{\mathrm{out}}} \bx \bx^\top \right).
    \end{equation}
\end{itemize}
\end{assumption}
Here, \eqref{eq:global21} implies that the inliers are well-distributed over $L_\star$, capturing a property analogous to a lower bound on $\Sin$. Equation~\eqref{eq:global22} ensures that no subspace in $\scrG(D, d)$ contains a certain fraction of the outliers, capturing a property analogous to Assumption~\ref{assump:lowerdim}. Lastly, \eqref{eq:global23} guarantees a spectral dominance of the inliers over the outliers.

\begin{remark}
    The appearance of $\cos(\theta_0)$ and $\sin(\theta_0)$ in Assumptions \ref{assump:global} and \ref{assump:global2} introduces a tradeoff in this parameter. Namely, this means that $\theta_0$ should be chosen not too close to $0$ or $\pi/2$, since in \eqref{eq:Singlobal} or \eqref{eq:global23} the left hand sides become zero. For simplicity, we set $\theta_0 = \pi/4$ in our later analysis of these conditions in Section \ref{subsubsec:assump_discuss}.
\end{remark} 
}

With these assumptions, we are now ready to present our main theorem. This theorem guarantees the global convergence of FMS-DS with a linear rate.
\begin{theorem}\label{thm:global}[Global Linear Convergence of FMS-DS]
Let $\calX$ be a dataset satisfying Assumptions \ref{assump:lowerdim}, \ref{assump:global}, and \ref{assump:global2} for the given constant $\gamma>0$. Then, the sequence $L^{(k)}$ generated by Algorithm~\ref{alg:IRLS} with dynamic smoothing hyperparameter $\gamma$ converges to $L_\star$ in the sense that $\lim_{k \to \infty}\|\bP_{L^{(k)}}-\bP_{L_\star}\|_2 =0$. In addition, the sequence $F_{\epsilon_k}(L^{(k)})$ is nonincreasing and there exist $C\geq1$ and $0<c<1$ such that 
\begin{equation}\label{eq:theorem1}F(L^{(k)}) - F(L_\star) \leq Cc^k (F(L^{(0)}) - F(L_\star)).\end{equation}
\end{theorem}
{We refer the reader to \eqref{eq:theorem1_c} for the rigorous statement.}
To prove Theorem~\ref{thm:global}, we rely on several lemmas. To improve readability, the proof of these lemmas is deferred to Section \ref{app:linthm_lemmaproofs}.

\begin{proof}[Proof of Theorem~\ref{thm:global}] 

The algorithm stops at the first $k$ with $\epsilon_k=0$, which gives $L^{(k)}=L_\star$ under Assumption~\ref{assump:lowerdim}, or $[A^{(k)}]=[A_\star]$ under Assumption~\ref{assumption3}; all updates analyzed below precede termination and have $\epsilon_k>0$.

We begin the proof with two lemmas.
The first lemma is a simple monotonicity property on the regularized objective function $F_{\epsilon}$ defined in \eqref{eq:Fepsilon}.
\begin{lemma}[Properties of smoothed objective function]\label{lemma:smooth_properties}
For any $\epsilon_{2}<\epsilon_1$, $F_{\epsilon_{2}}(L)\leq F_{\epsilon_1}(L)$. In addition, $F_0(L)=F(L)$.
\end{lemma}
The second lemma shows that the smoothed objective value is nonincreasing over iterations of Algorithm~\ref{alg:IRLS}, where the amount of decrease is quantified.
\begin{lemma}[Decrease over iterations]\label{lemma:decrease}
For all $k\geq0$ with $\epsilon_k>0$, we have
\begin{equation}\label{eq:decrease_equation0}
F_{\epsilon_k}(L^{(k)})-F_{\epsilon_k}(L^{(k+1)})\geq \frac{\|\bP_{L^{(k)}}\bS_{L^{(k)},\epsilon_k}\bQ_{L^{(k)}}\|_F^2}{2\|\bS_{L^{(k)},\epsilon_k}\|},
\end{equation}
where 
\begin{equation}\label{eq:decrease_equation}
\bS_{L,\epsilon}=\sum_{\bx\in\calX}\frac{\bx\bx^\top}{\max(\epsilon,\dist(\bx,L))}.
\end{equation}
\end{lemma}

Using the fact that $\epsilon_k$ is nonincreasing, Lemmas~\ref{lemma:smooth_properties} and ~\ref{lemma:decrease} imply that the sequence $(F_{\epsilon_k}(L^{(k)}))_{k \in \mathbb N}$ is nonincreasing:
\begin{equation}\label{eq:local_convergence}
F_{\epsilon_k}(L^{(k)})-F_{\epsilon_{k+1}}(L^{(k+1)})\geq F_{\epsilon_k}(L^{(k)})-F_{\epsilon_{k}}(L^{(k+1)})\geq 0.
\end{equation}
As a result, by basic monotonic convergence theory, the sequence $F_{\epsilon_k}(L^{(k)})$ must converge. Consequently, $F_{\epsilon_k}(L^{(k)})-F_{\epsilon_{k+1}}(L^{(k+1)})$ converges to zero.

The next two lemmas imply that the sequence $\epsilon_k$ converges to zero. The third lemma helps bound the right-hand side (RHS) of \eqref{eq:decrease_equation0} by a quantity proportional to $\epsilon_k$, but depending also on $L$, whereas the fourth lemma removes the dependence on $L$. 
We note that the first lower bound in the next lemma is on the Frobenius norm of $\bP_{L}\bS_{L,\epsilon}\bQ_{L}$, which is a Riemannian gradient of $F_{\epsilon}(L)$, and we thus view it as establishing nonzero gradient.
\begin{lemma}[Nonzero gradient, local guarantee]\label{lemma:gradient} If $L \in G(D,d)$, $\epsilon>0$ and at least half of the inliers have distance at least $\epsilon$ from $L$, that is, $|\{\bx\in\calX_{\mathrm{in}}:\dist(\bx,L)\geq\epsilon\}|\geq \frac{1}{2}|\calX_{\mathrm{in}}|$, then
\[
\|\bP_{L}\bS_{L,\epsilon}\bQ_{L}\|_F\geq  \frac{\cos\theta_1(L,L_\star)}{2\sqrt{d}}\Sin -\Sout \ \text{ and } \ \|\bS_{L,\epsilon}\|\leq \frac{\|\bX\|^2}{\epsilon}.
\]\end{lemma}

Under Assumption~\ref{assump:lowerdim}, the quantile convention and $\epsilon_k\leq q_\gamma(\{\dist(\bx,L^{(k)})\}_{\bx\in\calX})$ imply
\[
|\{\bx\in\calX_{\mathrm{in}}:\dist(\bx,L^{(k)})<\epsilon_k\}|
\leq \gamma|\calX|
\leq \tfrac12|\calX_{\mathrm{in}}|.
\]
Consequently, the cutoff condition in Lemma~\ref{lemma:gradient} holds at the iterates whenever $\epsilon_k>0$, including when some distances equal $\epsilon_k$.

The fourth lemma shows that under Assumptions~\ref{assump:lowerdim}--~\ref{assump:global2}, all iterations of FMS lie in a neighborhood of the true subspace $L_\star$. The size of this neighborhood is controlled by $\theta_0$, which appears in Assumption~\ref{assump:global2}.  
\begin{lemma}[$L^{(k)}$ is in a local neighborhood of $L_\star$]\label{lem:assumpholds2}
Under Assumptions~\ref{assump:lowerdim}, \ref{assump:global}, and~\ref{assump:global2}, for all iterations $L^{(k)}$ with $k\geq1$, $\theta_1(L^{(k)},L_\star)\leq \theta_0$.
\end{lemma}

Assumption~\ref{assump:lowerdim} gives $\Sin>0$. Set
\[
\Delta:=\frac{\Sin\cos\theta_0}{2\sqrt d}-\Sout
\geq\frac{\Sin\cos\theta_0}{6\sqrt d}>0.
\]
The inequality uses Assumption~\ref{assump:global}.
Combining \eqref{eq:local_convergence}, Lemmas~\ref{lemma:decrease}-\ref{lem:assumpholds2}, and \eqref{eq:Singlobal}, we have for $k\geq1$ 
\begin{align}\label{eq:local_convergence2}
F_{\epsilon_k}(L^{(k)})-F_{\epsilon_{k+1}}(L^{(k+1)})
\geq \frac{\|\bP_{L^{(k)}}\bS_{L^{(k)},\epsilon_k}\bQ_{L^{(k)}}\|_F^2}{2\|\bS_{L^{(k)},\epsilon_k}\|}
\geq \epsilon_k\frac{\left(\frac{\Sin\cos\theta_0}{2\sqrt{d}}- \Sout\right)^2}{2\|\bX\|^2}=\epsilon_k\frac{\Delta^2}{2\|\bX\|^2}.
\end{align}
We concluded from \eqref{eq:local_convergence} that $F_{\epsilon_k}(L^{(k)})-F_{\epsilon_{k+1}}(L^{(k+1)})$ converges to zero. Thus \eqref{eq:local_convergence2} implies that $\epsilon_k$ converges to zero as well.

{
We will divide our proof of convergence into two parts. In the first part, we give a shorter proof that establishes the convergence of $L^{(k)}$, and in the second part, we give a longer proof that establishes the convergence rate of $F(L^{(k)})$. Both parts depend on the following lemma. It establishes a bound on the ratio of outlier and inlier objective values based on $\Sin$ and $\Sout$. Let the inlier and outlier objectives be written as 
$$F_{\mathrm{in}}(L)=\sum_{\bx\in\calX_{\mathrm{in}}}\dist(\bx,L) \text{ and } F_{\mathrm{out}}(L)=\sum_{\bx\in\calX_{\mathrm{out}}}\dist(\bx,L).$$
\begin{lemma}[Bound on the objective value]\label{lemma:objectivevalue_bound}
If $L\neq L_\star$ and the principal angles between $L$ and $L_\star$ are $\theta_1,\cdots,\theta_d$ and $F_{\mathrm{in}}$ and $F_{\mathrm{out}}$ are as above, then
\begin{align}\label{eq:lemma:objectivevalue_boundb1}
&\frac{|F_{\mathrm{out}}(L)-F_{\mathrm{out}}(L_\star)|}{F_{\mathrm{in}}(L)-F_{\mathrm{in}}(L_\star)} \leq \frac{\sqrt{d} \Sout\sum_{j=1}^d\theta_j}{\Sin \sum_{j=1}^d\sin\theta_j }.
\end{align}
\end{lemma}

\noindent
\textbf{Proof of convergence:} By the definition~\eqref{eq:epsilon_choose} and compactness of the Grassmannian, there exists a subsequence $k_\ell\to\infty$ such that
\[
\epsilon_{k_\ell}=q_\gamma(\{\dist(\bx,L^{(k_\ell)})\}_{\bx\in\calX}),
\qquad L^{(k_\ell)}\longrightarrow\widehat L.
\]
Since $\epsilon_{k_\ell}\to0$, continuity of the empirical quantile and Assumption~\ref{assump:lowerdim} imply that $\widehat L=L_\star$. By the monotonicity of $F_{\epsilon_k}(L^{(k)})$,
\[
\lim_{k\to\infty}F_{\epsilon_k}(L^{(k)})
=\lim_{\ell\to\infty}F_{\epsilon_{k_\ell}}(L^{(k_\ell)})
=F(L_\star).
\]
We now make explicit why this implies convergence of the inlier objective. Since $\calX_{\mathrm{in}}\subset L_\star$, we have $F_{\mathrm{in}}(L_\star)=0$. Moreover, Assumption~\ref{assump:lowerdim} ensures $\Sin>0$. Lemma~\ref{lemma:objectivevalue_bound} and $\theta\leq(\pi/2)\sin\theta$ for $\theta\in[0,\pi/2]$ give
\[
|F_{\mathrm{out}}(L)-F_{\mathrm{out}}(L_\star)|
\leq \rho F_{\mathrm{in}}(L),
\qquad
\rho:=\frac{\pi\sqrt d\,\Sout}{2\Sin}.
\]
By Assumption~\ref{assump:global}, $\rho\leq\pi\cos\theta_0/6\leq\pi/6<1$. Therefore,
\begin{align*}
0\leq(1-\rho)F_{\mathrm{in}}(L^{(k)})
&\leq F(L^{(k)})-F(L_\star)\\
&\leq F_{\epsilon_k}(L^{(k)})-F(L_\star)
\longrightarrow0.
\end{align*}
Hence $F_{\mathrm{in}}(L^{(k)})\to0$. Finally, the inlier estimate~\eqref{eq:dist_lb} yields
\[
F_{\mathrm{in}}(L^{(k)})
\geq\Sin\sin\theta_1(L^{(k)},L_\star)
=\Sin\|\bP_{L^{(k)}}-\bP_{L_\star}\|_2,
\]
so $L^{(k)}\to L_\star$ in the claimed sense.

\noindent
\textbf{Proof of convergence rate:}  We now proceed with the proof of linear convergence of $F(L^{(k)})$ by showing that  $F_{\epsilon_k}(L^{(k)})$ converges linearly.} To establish that $F_{\epsilon_k}(L^{(k)})$ converges linearly, it suffices to show the existence of a constant $C_0>0$ such that 
\begin{equation}\label{eq:local_convergence3}
F_{\epsilon_k}(L^{(k)})-F(L_\star)\leq\epsilon_k \underbrace{\Bigg(\Big(1+\frac{\pi \sqrt{d}\Sout}{{2} \Sin}\Big)|\calX_{\mathrm{in}}|/c_3 +|\calX|/2\Bigg)}_{C_0},
\end{equation}
where $c_3$ will be specified in Lemma \ref{lemma:assumption1}.
The combination of \eqref{eq:local_convergence2} and \eqref{eq:local_convergence3} implies for $k\geq1$
\[
F_{\epsilon_k}(L^{(k)})-F_{\epsilon_{k+1}}(L^{(k+1)})\geq \frac{F_{\epsilon_k}(L^{(k)})-F(L_\star)}{C_0}\cdot \frac{\Delta^2}{2\|\bX\|^2}.
\]
{This implies $F_{\epsilon_{k+1}}(L^{(k+1)}) - F(L_\star)\leq (1-\frac{\Delta^2}{2C_0\|\bX\|^2})(F_{\epsilon_{k}}(L^{(k)}) - F(L_\star))$. As a result, using monotonicity for the first update,
\begin{equation}\label{eq:natural_bound_k}
    F_{\epsilon_{k}}(L^{(k)}) - F(L_\star)\leq \Big(1-\frac{\Delta^2}{2C_0\|\bX\|^2}\Big)^{\max\{k-1,0\}}(F_{\epsilon_{0}}(L^{(0)}) - F(L_\star)),
\end{equation}
and $F_{\epsilon_k}(L^{(k)})$ converges linearly to $F(L_\star)$. We can flip this to linear convergence in $F$ as well. Using Lemma \ref{lemma:objectivevalue_bound},
\begin{align*}
    F(L^{(k)})-F(L_\star)&=F_{\mathrm{in}}(L^{(k)})-F_{\mathrm{in}}(L_\star)+F_{\mathrm{out}}(L^{(k)})-F_{\mathrm{out}}(L_\star) \\
    &\geq\left(1-\frac{\Sout\sum_{j=1}^d\theta_j}{\frac{\sum_{j=1}^d\sin\theta_j}{\sqrt{d}} \Sin}\right) \left(F_{\mathrm{in}}(L^{(k)})-F_{\mathrm{in}}(L_\star)\right)\\&\geq  \left(1-\frac{\sqrt{d}\pi\Sout}{2\Sin}\right) \left(F_{\mathrm{in}}(L^{(k)})-F_{\mathrm{in}}(L_\star)\right)\geq  \left(1-\frac{\sqrt{d}\pi\Sout}{2\Sin}\right) \frac{|\calX_{\mathrm{in}}|}{2}\epsilon_k.
\end{align*} 
Combining this with the natural bound in \eqref{eq:natural_bound_k}, \[F(L^{(0)})-F(L_\star)\geq \frac{\left(1-\frac{\sqrt{d}\pi\Sout}{2\Sin}\right) \frac{|\calX_{\mathrm{in}}|}{2}}{\left(1-\frac{\sqrt{d}\pi\Sout}{2\Sin}\right) \frac{|\calX_{\mathrm{in}}|}{2}+\frac{|\calX|}{2}}(F_{\epsilon_0}(L^{(0)})-F(L_\star)),\] 
we obtain
\begin{align}\label{eq:theorem1_c}F(L^{(k)}) - F(L_\star)&\leq F_{\epsilon_k}(L^{(k)}) - F(L_\star)\leq \left(1-\frac{\Delta^2}{2C_0\|\bX\|^2}\right)^{\max\{k-1,0\}} (F_{\epsilon_0}(L^{(0)}) - F(L_\star)) \\ \nonumber 
&\leq \frac{\left(1-\frac{\sqrt{d}\pi\Sout}{2\Sin}\right) \frac{|\calX_{\mathrm{in}}|}{2}+\frac{|\calX|}{2}}{\left(1-\frac{\sqrt{d}\pi\Sout}{2\Sin}\right) \frac{|\calX_{\mathrm{in}}|}{2}} \left(1-\frac{\Delta^2}{2C_0\|\bX\|^2}\right)^{\max\{k-1,0\}} (F(L^{(0)}) - F(L_\star)).
\end{align}

We also observe that
\[
\sum_{\bx\in\calX_{\mathrm{out}}}\frac{\bP_{L}\bx\bx^\top\bP_{L^\perp}}{\max(\dist(\bx,L),\epsilon)}
= \bZ_1^\top \bZ_2,
\]
where $\bZ_1$ is a matrix with $|\calX_{\mathrm{out}}|$ rows given by $\bP_L \bx$ for $\bx \in \calX_{\mathrm{out}}$, and $\bZ_2$ is a matrix with $|\calX_{\mathrm{out}}|$ rows consisting of vectors of norm at most one
\[
\frac{\bP_{L^\perp}\bx}{\max(\dist(\bx,L),\epsilon)}.
\]
Consequently, $\|\bZ_1\| \le \|\bX\|$ and $\|\bZ_2\| \le \|\bZ_2\|_F \le \sqrt{|\calX_{\mathrm{out}}|} \le \sqrt{|\calX|}$, and
\[
\Sout \leq \|\bZ_1\|\|\bZ_2\|\leq \|\bX\|\sqrt{|\calX|}.
\]
Cauchy--Schwarz gives $\Sin\leq\sqrt{\nin}\|\bX\|$, hence $\Delta^2\leq\nin\|\bX\|^2/(4d)$. Moreover, since $C_0\geq|\calX|/2$, it follows that
\[
\frac{1}{2}\leq c:=1-\frac{\Delta^2}{2C_0\|\bX\|^2}<1.
\]
Therefore, \eqref{eq:theorem1_c} implies that the convergence is linear. Since $c^{\max\{k-1,0\}}\leq c^{-1}c^k$, the theorem holds with the prefactor in \eqref{eq:theorem1_c} multiplied by $c^{-1}$.

We establish \eqref{eq:local_convergence3} through our final lemma, which establishes} a connection between the quantiles of the set of distances and the inlier objective value that is useful for controlling the iterates $\epsilon_k$.
\begin{lemma}\label{lemma:assumption1} Under Assumption 1, there exists $c_3>0$ such that for all $L\in \scrG(D,d)$,
\begin{equation}\label{eq:assumption12}
q_{\gamma}(\{\dist(\bx,L)\}_{\bx\in\calX})\geq c_3\frac{F_{\mathrm{in}}(L)}{|\calX_{\mathrm{in}}|}.
\end{equation}
\end{lemma}

To prove \eqref{eq:local_convergence3}, we use the following exact expression for the smoothing error when $\epsilon_k>0$:
\begin{align}\nonumber
F_{\epsilon_k}(L^{(k)})-F(L^{(k)})
&=\sum_{\substack{\bx\in\calX\\\dist(\bx,L^{(k)})<\epsilon_k}}
\frac{(\epsilon_k-\dist(\bx,L^{(k)}))^2}{2\epsilon_k}\\\label{eq:natural_bound}
&\leq \frac{\epsilon_k}{2}
|\{\bx\in\calX:\dist(\bx,L^{(k)})<\epsilon_k\}|
\leq \frac{\epsilon_k|\calX|}{2}.
\end{align}
Distances equal to $\epsilon_k$ contribute zero to the smoothing error; thus no assumption excluding ties is needed. In fact, the quantile rule bounds the cardinality in~\eqref{eq:natural_bound} by $\gamma|\calX|$, although the weaker final bound is sufficient here.

Second, Lemma \ref{lemma:objectivevalue_bound} implies
\begin{align*}
    F(L^{(k)})-F(L_\star)&=F_{\mathrm{in}}(L^{(k)})-F_{\mathrm{in}}(L_\star)+F_{\mathrm{out}}(L^{(k)})-F_{\mathrm{out}}(L_\star) \\
    &\leq\left(1+\frac{\Sout\sum_{j=1}^d\theta_j}{\frac{\sum_{j=1}^d\sin\theta_j}{\sqrt{d}} \Sin}\right) \left(F_{\mathrm{in}}(L^{(k)})-F_{\mathrm{in}}(L_\star)\right).
\end{align*}
Third, when $\epsilon_k = q_{\gamma}(\{\dist(\bx,L^{(k)})\}_{\bx\in\calX})$, Lemma \ref{lemma:assumption1} implies
\[
F_{\mathrm{in}}(L^{(k)})-F_{\mathrm{in}}(L_\star)= F_{\mathrm{in}}(L^{(k)})\leq |\calX_{\mathrm{in}}|\epsilon_{k}/c_3.
\]
Combining the three equations above,
\begin{align}\label{eq:local_convergence4}
F_{\epsilon_k}(L^{(k)})-F(L_\star)&\leq (F(L^{(k)})-F(L_\star))+\epsilon_k|\calX|/2
\\ \nonumber
&\leq \left(1+\frac{\Sout\sum_{j=1}^d\theta_j}{\frac{\sum_{j=1}^d\sin\theta_j}{\sqrt{d}} \Sin}\right) \left(F_{\mathrm{in}}(L^{(k)})-F_{\mathrm{in}}(L_\star)\right) + \epsilon_k|\calX|/2
\\ \nonumber
&\leq \left(\left(1+\frac{\Sout\sum_{j=1}^d\theta_j}{\frac{\sum_{j=1}^d\sin\theta_j}{\sqrt{d}} \Sin}\right)|\calX_{\mathrm{in}}|/c_3 +|\calX|/2\right)\epsilon_k.
\end{align}

Using $\theta \leq \frac{\pi}{2} \sin \theta$ for $\theta \in [0, \pi/2]$, and the definition of $C_0$ in \eqref{eq:local_convergence3},
\[
\left(\left(1+\frac{\Sout\sum_{j=1}^d\theta_j}{\frac{\sum_{j=1}^d\sin\theta_j}{\sqrt{d}} \Sin}\right)|\calX_{\mathrm{in}}|/c_3 +|\calX|/2\right)\leq C_0
\] 
and \eqref{eq:local_convergence3} holds.
When $\epsilon_k\neq q_{\gamma}(\{\dist(\bx,L^{(k)})\}_{\bx\in\calX})$, then there exists $l<k$ such that $\epsilon_k=\epsilon_l = q_{\gamma}(\{\dist(\bx,L^{(l)})\}_{\bx\in\calX})$ and we have
\[
F_{\epsilon_k}(L^{(k)})-F(L_\star)\leq F_{\epsilon_l}(L^{(l)})-F(L_\star)\leq C_0\epsilon_l=C_0\epsilon_k,
\]
where the second inequality follows from \eqref{eq:local_convergence4}.

\end{proof}

We conclude with a few remarks.

\begin{remark}[Choice of $c_3$ in Lemma~\ref{lemma:assumption1}]
By the proof of Lemma~\ref{lemma:assumption1}, it is sufficient to find $c_3$ such that no more than $\gamma/2\gamma_\star$ percentage of the inliers are in the set
 \[
\{\bx\in \calX_{\mathrm{in}}: |\bu_1^\top\bx|\leq c_3\cdot\mathrm{mean}_{\bx\in\calX_{\mathrm{in}}}\|\bx\|\},
\]
for any $\|\bu_1\|=1$, where $\mathrm{mean}$ is the average value of the set, and no more than $\gamma/2(1-\gamma_\star)$ percentage of the outliers are contained in the region
 \[
 \{\bx\in \calX_{\mathrm{out}}: |\bx^\top\bu_2|\leq c_3\cdot\mathrm{mean}_{\bx\in\calX_{\mathrm{in}}}\|\bx\|\}
 \]
 for any $\|\bu_2\|=1$.
\end{remark}

\begin{remark}
    Compared to recent guarantees for IRLS  \citep{daubechies2010iteratively,kummerle2018harmonic,kummerle2020iteratively,peng2022global}, our work represents a step forward in two key ways. We analyze an IRLS algorithm on a Riemannian manifold, $\scrG(D,d)$, and establish global convergence for a nonconvex problem, both of which are novel results in the literature. In contrast, most analyses rely on convex optimization to ensure global convergence. {Existing works on IRLS with nonconvex $\ell_p$ minimization ($p<1$) by  \cite{daubechies2010iteratively,kummerle2018harmonic,yang2022towards,peng2022global} achieve only local convergence.}
\end{remark}

\subsubsection{Discussion of Assumptions}
\label{subsubsec:assump_discuss}

Assumption \ref{assump:lowerdim} ensures that $L_\star$ contains more points than all other low-dimensional subspaces by some margin. In the Generalized Haystack Model discussed below,  it is satisfied almost surely
when $\nin>4d$.

Assumptions \ref{assump:global} and \ref{assump:global2} represent generic conditions on the inlier-outlier dataset, $\calX$, that can be satisfied in a wide variety of settings. We can illustrate distinct examples where the conditions of these assumptions hold using data models inspired by \cite{maunu2019well,maunu2019robustsubspacerecoveryadversarial}. To simplify our setting, we will assume Gaussian inliers and consider two models of outliers: Gaussian outliers and arbitrary outliers on the sphere. { To further simplify, we set $\gamma = \gamma_*/2$, and we verify Assumptions~\ref{assump:global} and~\ref{assump:global2} with $\theta_0=\pi/4$ and $\alpha_0=\gamma / (2\gamma_\star) = 1/4$.}

The case of Gaussian inliers and Gaussian outliers is the Generalized Haystack Model \citep{maunu2019well}. In this model, we assume that $\bx\in\calX$ are i.i.d. sampled from a distribution $\mu=\alpha_{\mathrm{in}}\mu_{\mathrm{in}}+\alpha_{\mathrm{out}}\mu_{\mathrm{out}}$, that is, it is sampled from $\mu_{\mathrm{in}}$ with probability $\alpha_{\mathrm{in}}$ and sampled from $\mu_{\mathrm{out}}$ with probability $\alpha_{\mathrm{out}}$, where $\alpha_{\mathrm{in}}+\alpha_{\mathrm{out}}=1$. In particular, $\mu_{\mathrm{in}}$ is the distribution of inliers and is assumed to be $N(0,\bSigma_{\mathrm{in}})$, and $\mu_{\mathrm{out}}$ is the distribution of outliers and is assumed to be $N(0,\bSigma_{\mathrm{out}})$. In addition, we assume that $C\sigma_{\mathrm{in}}^2\bP_{L_\star}/d \psdgeq \bSigma_{\mathrm{in}}\psdgeq \sigma_{\mathrm{in}}^2\bP_{L_\star}/d$ and $c \sigma_{\mathrm{out}}^2\bI/D \psdleq \bSigma_{\mathrm{out}}\psdleq \sigma_{\mathrm{out}}^2\bI/D$, i.e., their condition numbers are $O(1)$. 
{Under this model, we verify that Assumptions~\ref{assump:global} and~\ref{assump:global2} hold, as stated in the following proposition.

\begin{proposition}[Generalized Haystack Model]\label{prop:haystack_assump}

Suppose that inliers and outliers follow the Generalized Haystack Model with fixed positive population mixture proportions $\alpha_{\mathrm{in}}$ and $\alpha_{\mathrm{out}}$, where $\alpha_{\mathrm{in}}+\alpha_{\mathrm{out}}=1$. Assume $n\gtrsim D^3\log D$, $\alpha_0=1/4$, $\theta_0=\pi/4$, and, for a sufficiently large absolute constant $C$,
\[
\alpha_{\mathrm{in}}\geq\frac43
\min\left(\frac{C}{D-d},\,0.01\right).
\]
For each realization, set $\gamma=\gamma_\star/2$, where $\gamma_\star=\nin/n$. Then Assumptions~\ref{assump:global} and~\ref{assump:global2} hold with high probability as $n,D\to\infty$, provided that
\begin{equation}\label{eq:prop:haystack_assump}
\frac{\alpha_{\mathrm{in}}}{\alpha_{\mathrm{out}}}
\gtrsim \frac{\sigma_{\mathrm{out}}}{\sigma_{\mathrm{in}}}
\cdot\frac{d}{\sqrt{D(D-d)}}
+\frac{\sigma_{\mathrm{out}}^2}{\sigma_{\mathrm{in}}^2}
\cdot\frac{d}{D}.
\end{equation}

\end{proposition}
The hypotheses above concern deterministic population quantities. The implicit constant in $n\gtrsim D^3\log D$ may depend on the fixed positive mixture proportions and the fixed bounds on the covariance condition numbers; it does not depend on $n,D$, or the variance scales. The choice $\gamma=\gamma_\star/2$ remains a sample-dependent oracle tuning choice for this model illustration; it is not a prescription for estimating the unknown inlier fraction. In Section~\ref{sec:haystack}, we condition on the component labels and use concentration of the component counts to transfer~\eqref{eq:prop:haystack_assump} to the empirical ratio, with an adjustment of the absolute constant.

\noindent\textbf{Sample-size requirement.}
The condition $n\gtrsim D^3\log D$ is a sufficient condition used in the probabilistic verification of Assumptions~\ref{assump:global} and~\ref{assump:global2}. The proof in Section~\ref{sec:haystack} uses covering arguments to obtain simultaneous control of the outlier-alignment statistic and the empirical quantile and spectral quantities in these assumptions. We do not claim that the resulting dependence on $D$ is optimal, and it may be restrictive in high-dimensional applications. A possible route to weaker requirements is to sharpen these uniform estimates and retain their dependence on the intrinsic dimension $d$ and mixture proportions, rather than summarize them by a single bound on the ambient dimension $D$. Establishing a smaller sample-size guarantee also requires more intricate analysis.

For a proof, see Section~\ref{sec:haystack}. This can be compared with the bound for small sample sizes derived for a different nonconvex method in~\cite{maunu2019well}, which states that
\[
\frac{n_{\mathrm{in}}}{n_{\mathrm{out}}} \gtrsim \frac{\sigma_{\mathrm{out}}}{\sigma_{\mathrm{in}}} \cdot \frac{d}{\sqrt{D(D - d)}}.
\]
The two bounds are of the same order when $\sigma_{\mathrm{out}} \approx \sigma_{\mathrm{in}}$ or $\sigma_{\mathrm{out}} \leq \sigma_{\mathrm{in}}$, and the additional term changes the sufficient ratio by a factor of order $1+(\sigma_{\mathrm{out}}/\sigma_{\mathrm{in}})\sqrt{(D-d)/D}$ when $\sigma_{\mathrm{out}}\geq\sigma_{\mathrm{in}}$. One could also consider an asymptotic regime as in \cite{lerman2018fast,maunu2019well}, where it is possible to take the ratio $\nin/\nout$ to zero for fixed $d$ and $D$ as $n \to \infty$ in the Haystack model, where $\bSigma_{\mathrm{in}} = \bP_{L_\star}/d$ and $\bSigma_{\mathrm{out}} = \bI/D$.

In the case of adversarial outliers inspired by \cite{maunu2019robustsubspacerecoveryadversarial}, we allow outliers to follow an arbitrary distribution on the sphere, and we assume that inliers follow a Gaussian distribution $N(\bzero, \sigma_{\mathrm{in}}^2\bP_{L_\star}/d)$.  We refer to this as an adversarial model because outliers can follow any distribution. We restrict outliers to the sphere because our objective is not invariant to scale.  We  also assume that $\gamma_\star>4/5$, so that in Assumption \ref{assump:global2} we only need to show that \eqref{eq:global21} and \eqref{eq:global23} hold. The following proposition verifies Assumptions~\ref{assump:global} and~\ref{assump:global2}.

\begin{proposition}[Adversarial Outliers]\label{prop:adv_assump}
    Suppose that $\calX_{\mathrm{out}}\subset S^{D-1}$ and the inliers follow the Gaussian model, $\gamma = \gamma_\star/2$, $\alpha_0 = 1/4$, $\theta_0 = \pi/4$, and $\gamma_\star>4/5$. Then, Assumptions \ref{assump:global} and \ref{assump:global2} hold with high probability as $\nin\to\infty$, with $\nin\gtrsim d^2\log(2d)$, provided that 
    \begin{equation}\label{eq:adv_condition}
       \frac{\nin}{\nout} \gtrsim  \frac{d}{\sigma_{\mathrm{in}}^2}+\frac{d}{\sigma_{\mathrm{in}}}.  
    \end{equation}
\end{proposition}
The proof of this proposition is given in Section \ref{subsec:prop_adv_proof}. We see that both Assumptions \ref{assump:global} and \ref{assump:global2} are satisfied once $\nin/\nout = O(d)$. This matches the result of \cite{maunu2019robustsubspacerecoveryadversarial}. 
}

\subsubsection{Extensions of Theory}
\label{subsubsec:theory_ext}

There are a couple of straightforward extensions of our theoretical results that we do not include in this paper.

\noindent
\textbf{Stability to noise:} The FMS-DS algorithm exhibits empirical stability to noise. Our analysis can be extended to show that FMS-DS converges to a neighborhood of $L_\star$ since any $L$  outside this neighborhood cannot be a fixed point. Technical analysis is deferred to future work.

\noindent
\textbf{Fixed smoothing parameters:}
The FMS algorithm \citep{lerman2018fast} uses the fixed value $\epsilon_k=\epsilon$ for $k\geq 1$. However, they are only able to guarantee convergence for the Haystack model. A corollary of Theorem \ref{thm:global} allows us to give a more general guarantee of finite-time approximate recovery by FMS with a quantile-compatible fixed smoothing parameter under our generic conditions. Its proof follows from the proof of Lemma~\ref{lemma:gradient}, which also implies that any stationary point of $F_\epsilon$ satisfying $\theta_1(L,L_\star)\leq\theta_0$ has inlier median residual at most $\epsilon$.
\begin{corollary}[Global convergence of FMS with fixed $\epsilon$]\label{cor:FMS}
Under the Assumptions \ref{assump:lowerdim}, \ref{assump:global}, and \ref{assump:global2} and if $\epsilon=q_\gamma(\{\dist(\bx,L^{(0)})\}_{\bx\in\calX})>0$, then in
$2+\left\lceil 2\big\|\sum_{\bx\in\calX}\bx\bx^\top\big\|\frac{F_{\epsilon}(L^{(0)})-F(L_\star)}{\epsilon\Delta^2}\right\rceil$ 
iterations, FMS reaches a subspace $L^{(k)}$ satisfying
$q_{0.5}(\{\dist(\bx,L^{(k)})\}_{\bx\in\calX_{\mathrm{in}}})\leq \epsilon.$
\end{corollary}
\begin{proof}
With $\epsilon$ fixed at the initial quantile, the argument of
Lemma~\ref{lem:assumpholds2} shows that
$\theta_1(L^{(k)},L_\star)\leq\theta_0$ for all $k\geq1$.
Until the inlier median residual is at most $\epsilon$,
Lemmas~\ref{lemma:decrease} and~\ref{lemma:gradient} imply
$F_\epsilon(L^{(k)})-F_\epsilon(L^{(k+1)})
\geq \frac{\epsilon\Delta^2}{2\|\bX\|^2}.
$ On the other hand, $F_\epsilon$ is nonincreasing and bounded below
by $F(L_\star)$, so the total decrease cannot exceed
$F_\epsilon(L^{(0)})-F(L_\star)$.
Dividing this total by the minimum decrease per update, and
allowing for the first update and integer rounding, gives the
stated iteration bound.
\end{proof}

\subsection{Affine Subspace Recovery}\label{sec:affine_analysis}

This section presents a local theoretical guarantee for the AFMS-DS algorithm in Section~\ref{sec:affineIRLS}. While weaker than Theorem~\ref{thm:global} due to stricter assumptions, it is, to our knowledge, the first theoretical guarantee for robust affine subspace estimation.

We now present our assumptions for the main theorem on affine subspace recovery. First, we assume without loss of generality (WLOG) that $\bm_\star=0$ due to the translation equivariance of the IRLS method.
Second, the following assumption generalizes Assumption \ref{assump:lowerdim} to the affine setting.
\begin{assumption}\label{assumption3}
    For all ($d-1$)-dimensional affine subspaces $[A_0]\subset [A_\star]$ and $d$-dimensional affine subspaces $[A]\neq [A_\star]$ in $\scrA(D,d)$, $|\calX \cap ([A_0]\cup [A])|/|\calX| < \gamma \leq \gamma_\star/2$.
\end{assumption} 
Finally, we also have a generalization of our generic condition in Assumption \ref{assump:global} to the affine setting. To do this, we define a notion of distance between a representative $A = (L, \bm) \in \scrG(D,d) \times \R^D$ and an equivalence class $[A'] = [(L', \bm')] \in \scrA(D,d)$:
\begin{equation}\label{eq:distAAprime}
\dist(A, [A'])=\sqrt{\sum_{i=1}^d\theta_i^2(L,L')+\dist^2(\bm,[A'])},
\end{equation}
where $\dist(\bm,[A'])=\dist(\bm-\bm',L')$. We discuss this more in Section \ref{subsubsec:affassump}. Due to the translation equivariance of the AFMS method, without loss of generality, we can assume that $[A_{\star}]$ is a linear subspace, so that $[A_\star] = L_\star$. This simplification allows us to write $\dist(A_\star,[A])=\dist((L_\star, 0), [A])$. 
\begin{assumption}\label{assumption4}
There exists $0<c_0<\pi/2$ such that for any $[A] \in \scrA(D,d)$ satisfying $\dist(A_\star,[A])\leq c_0$, 
\begin{align}\label{eq:assumption4_1}
\cos c_0 &{\min_{\calX_{\mathrm{in}}' \subseteq \calX_{\mathrm{in}} \,:\, |\calX_{\mathrm{in}}'| \geq |\calX_{\mathrm{in}}|/2}}\sum_{\bx\in\calX_{\mathrm{in}}'}\dist(\bx,[A])\geq \\ \nonumber
&4\sqrt2\max \Big(\frac{\pi}{2}\sum_{\bx\in\calX_{\mathrm{out}}} \dist(\bP_{L_\star}\bx, [A]), \left(|\calX_{\mathrm{out}}|\sum_{j=1}^d\theta_j(L,L_\star)^4\right)^{1/2}
\left\|\bP_{L_\star^\perp}\sum_{\bx\in\calX_{\mathrm{out}}}\bx\bx^\top\bP_{L_\star^\perp}\right\|_2^{1/2}
\Big).
\end{align}

\end{assumption} 
In Assumption~\ref{assumption4}, \eqref{eq:assumption4_1} implies that the influence of the inliers is larger than that of the outliers. %
Compared with the assumptions for Theorem~\ref{thm:global}, these assumptions are more restrictive. 

Our main theoretical result on AFMS-DS  guarantees that it converges locally and can be viewed as a generalization of Theorem~\ref{thm:global}. 
\begin{theorem}\label{thm:affine}[Local convergence of AFMS-DS] 
Under Assumptions~\ref{assumption3} and~\ref{assumption4} and $\dist(A_\star,[A^{(k)}])\leq c_0$ 
 for all $k\geq 1$, where $A^{(k)}$ is the sequence generated by AFMS-DS, $[A^{(k)}]$ converges to $[A_\star]$. In addition, the sequence $F_{\epsilon_k}^{(a)}(A^{(k)})$ is nonincreasing and converges to $F^{(a)}(A_\star)$ linearly. 
\end{theorem}

The proof is a natural extension of the argument presented in Theorem~\ref{thm:global} and is provided in Section~\ref{sec:proof_aff}. However, in the affine setting, Assumption \ref{assumption4} becomes more challenging to simplify to a more interpretable form. We conjecture that further analysis may yield results analogous to Assumptions~\ref{assump:global} and \ref{assump:global2}, potentially leading to global convergence guarantees. We leave this investigation for future work.

Our result in Theorem \ref{thm:affine} relies on an additional condition versus Theorem \ref{thm:global}, namely that $\dist(A_\star, [A^{(k)}]) \leq c_0$ for all $k \geq 1$, but it does not have an analog of Assumption \ref{assump:global2}. Monotonicity gives $F^{(a)}(A^{(k)})\leq F_{\epsilon_k}^{(a)}(A^{(k)})\leq F_{\epsilon_0}^{(a)}(A^{(0)})$. Thus, initialization ensures the required confinement whenever this sublevel set of $F^{(a)}$ is contained in the $c_0$-neighborhood of $[A_\star]$.

\subsubsection{Discussion of Assumptions 4 and 5}
\label{subsubsec:affassump}
Following the discussion in Section~\ref{subsubsec:assump_discuss}, we analyze Assumptions~\ref{assumption3} and~\ref{assumption4} under both the adversarial and Haystack models. Unlike the earlier spherical setting, we do not consider spherized  models here, as they are not well-suited for affine subspace recovery.

Instead, we adopt a probabilistic framework in which a fraction $\alpha_{\mathrm{in}}$ of the samples are inliers and $\alpha_{\mathrm{out}} = 1 - \alpha_{\mathrm{in}}$ are outliers. Inliers are drawn from the distribution $\mu_{\mathrm{in}} = \mathcal{N}(\bm_\star, \bSigma_{\mathrm{in}}/d)$, where $\bSigma_{\mathrm{in}}$ is a positive semidefinite matrix with range in $L_\star$ and $\sigma_d(\bSigma_{\mathrm{in}}) \geq 1$. Outliers are either drawn from a general distribution $\mu_{\mathrm{out}}$ (affine adversarial model), or from a Gaussian $\mathcal{N}(\bm_\star, \bSigma_{\mathrm{out}}/D)$ with $\bSigma_{\mathrm{out}} \psdleq \sigma_{\mathrm{out}}^2 \bI$ (Affine Haystack model). As in the proof of our main theorem, we can assume without loss of generality that $\bm_\star = \bzero$ and identify $[A_\star]$ with $L_\star$. 
In Propositions~\ref{prop:affine_adversarial} and~\ref{prop:affine_haystack}, take $\gamma=\gamma_\star/2$. The conclusions hold with probability tending to one as $n\to\infty$, for fixed $d,D$ and fixed positive mixture proportions, and $C$ is a sufficiently large absolute constant.

\begin{proposition}[Affine Adversarial Model]\label{prop:affine_adversarial}
Under the affine adversarial model, Assumptions~\ref{assumption3} and~\ref{assumption4} hold if
\[
\frac{\alpha_{\mathrm{in}}}{\alpha_{\mathrm{out}}}
>\frac{C}{\cos c_0}\max\left(
\sqrt d\int_{\bx\sim\mu_{\mathrm{out}}}\|P_{L_\star}\bx\|,\ 1,\
c_0\sqrt d\left\|\int_{\bx\sim\mu_{\mathrm{out}}}
P_{L_\star^\perp}\bx\bx^\top P_{L_\star^\perp}\right\|_2^{1/2}
\right).
\]
\end{proposition}

Proposition~\ref{prop:affine_adversarial} suggests that if $c_0$ is chosen sufficiently small and the distribution $\mu$ has support of magnitude $O(1)$, then the outlier ratio must be bounded by $O(1/\sqrt{d})$.  This is better than the linear case, where the outlier ratio is bounded by $O(1/d)$, because the analysis of the affine case does not use the extra Assumption \ref{assump:global2}; instead, we assume that the initialization is sufficiently good, specifically by requiring that $c_0$ is small enough.

\begin{proposition}[Affine Haystack Model]\label{prop:affine_haystack}
Under the Affine Haystack model, and assuming $\sigma_{\mathrm{out}}\leq\sqrt{D/d}$, Assumption~\ref{assumption3} and equation~\eqref{eq:assumption4_1} hold provided that
\[
\frac{\alpha_{\mathrm{in}}}{\alpha_{\mathrm{out}}}
>\frac{C}{\cos c_0}\max\left(1,\,c_0\sigma_{\mathrm{out}}\sqrt{d/D}\right).
\]
\end{proposition}

We note that the required inlier fraction for the Affine Haystack model is $O(1)$, in contrast to the $O(d/\sqrt{D(D-d)})$ requirement in the Generalized Haystack Model discussed in Section~\ref{subsubsec:assump_discuss}. However, the Affine Haystack model allows outliers to have magnitudes up to a factor of $\sqrt{D/d}$ larger.

The $O(1)$ inlier requirement is intuitive. For instance, in the case of estimating a $0$-dimensional affine subspace (i.e., computing the median), the breakdown point is $0.5$, meaning that at least half of the data must be inliers for robust recovery. However, since the outliers in the Affine Haystack model are relatively benign, we expect the breakdown point of AFMS to be lower. We leave a precise characterization of this threshold to future work.

\textbf{Discussion of Distance:} For our analysis, we assume that the set of inliers $\sX_{\mathrm{in}}$ lies on a $d$-dimensional affine subspace $[A_\star]=[(L_\star,\bm_\star)]$.  This defines a distance between a representative for $[A]$ and the affine subspace $[A']$. To our knowledge, this notion of distance is new in the literature and may be of independent interest. This is not a true metric on the affine Grassmannian since it depends on the choice of $\bm$ in the representation for $[A]$. Note, however, that $\dist(A, [A']) = 0$ if and only if $[A] = [A']$.

\section{Numerical Simulations}
\label{sec:numerics}

In this section, we present experiments on synthetic and real data. Our first experiment in Section \ref{subsec:exp1} compares the FMS algorithm with various competitive RSR methods. The second and third experiments in Sections \ref{subsec:exp2} and \ref{subsec:exp3} investigate the performance of FMS under various regularization strategies, highlighting the effectiveness of the proposed dynamic smoothing approach in escaping saddle points. The final experiment in Section \ref{sec:practical} demonstrates the practical utility of FMS for dimensionality reduction in training neural networks. The code used for our experiments and implementations is available at: \url{https://github.com/alp-del/Robust-PCA}.

\subsection{Experiment 1: Subspace Recovery Across Methods}
\label{subsec:exp1}

For the first experiment, we evaluate the performance of four competitive RSR algorithms on synthetic datasets generated under a semi-adversarial inlier-outlier model. The algorithms chosen are RANSAC~\citep{maunu2019robustsubspacerecoveryadversarial}, STE~\citep{yu2024subspace}, TME~\citep{tyler1987distribution,zhang2016robust}, and FMS. We also include PCA as a baseline subspace recovery method. Comparisons with other RSR algorithms have been conducted in prior work and consistently show inferior performance across a range of regimes (see, e.g., \cite{lerman2018overview}), especially when one is interested in exact recovery. Although RANSAC and TME were not originally developed for RSR, we use their RSR-adapted versions in our experiments. Finally, we note that the most common implementation of DPCP \citep{NEURIPS2018_af21d0c9} is equivalent to FMS, but targets the orthogonal complement of the desired subspace. As a result, it yields the same solution as FMS when using the same initialization and parameters.

The inlier-outlier models used here aim to demonstrate the performance of these RSR algorithms in a semi-adversarial setting. The datasets consist of two components: inliers lying on a subspace of dimension \( d \in \{3, 10, {50}\} \) and sampled from a standard Gaussian distribution supported on that subspace, and outliers lying on a separate subspace of dimension \( d_{\mathrm{out}} \in \{1, 5, 10, {50}\} \) and independently sampled from a standard Gaussian distribution supported on the outlier subspace.
WLOG, we assume that the ambient dimension is $D=d+d_{\mathrm{out}}$. As a final step, we normalize all points so that they lie on the unit sphere in $\mathbb{R}^{D}$ to remove effects of scale. 

The goal is to compare the performance of the RSR algorithms in these semi-adversarial settings. Therefore, we fix the total number of data points at $n=160$, and we vary the proportion of inliers and outliers to see where each algorithm breaks down.

Since we are comparing iterative algorithms, we ensure a fair comparison of running times by setting the maximum number of iterations to 200 for the RANSAC, FMS, TME, and STE methods. In particular, we emphasize that RANSAC is executed with only 200 candidate samples. This choice is motivated by the fact that the computational complexity of a single iteration of FMS or STE is comparable to that of evaluating one candidate sample in RANSAC.

For the STE method, we use a gamma parameter of 2, and for both STE and TME we use regularization of $\epsilon = 10^{-10}$. We run FMS with fixed $\epsilon = 10^{-10}$, as well as dynamic smoothing with two choices of $\gamma=0.5$ and $\gamma = 0 .1$. To ensure statistical reliability, we generate 200 datasets for each pair $d$ and $d_{\mathrm{out}}$, and the average performance is reported. FMS is initialized using the PCA subspace, and TME and STE are initialized with the sample covariance.

Let the orthonormal matrix \( \bP \in \mathbb{R}^{D \times d} \) represent the basis of the true inlier subspace, and let \( \widetilde{\bP} \in \mathbb{R}^{D \times d} \) denote the basis of the estimated subspace. We measure the subspace estimation error by \( \left\| \widetilde{\bP} \widetilde{\bP}^\top - \bP \bP^\top \right\|_2,\) which corresponds to the sine of the largest principal angle between the true subspace and its estimate. Due to the fact that we are in a noiseless regime, we plot all errors on a log scale. Since error is reported as the geometric mean of the errors over the 200 datasets, the result is the average log-error.

\begin{figure}[!ht]
  \centering
  \setlength{\tabcolsep}{3mm}
  \renewcommand{\arraystretch}{0}

\begin{tabular}{@{}>{\centering\arraybackslash}m{.05\linewidth}
                @{\hspace{2.5mm}}
                >{\centering\arraybackslash}m{.3\linewidth}
                @{\hspace{2.5mm}}
                >{\centering\arraybackslash}m{.3\linewidth}
                @{\hspace{2.5mm}}
                >{\centering\arraybackslash}m{.3\linewidth}@{}}
      & \textbf{$d = 3$}
      & \textbf{$d = 10$} & \textbf{$d = 50$} \\  \noalign{\vskip 3mm} 
      \rowlabel{$d_{\mathrm{out}} = 1$} &
        \plot{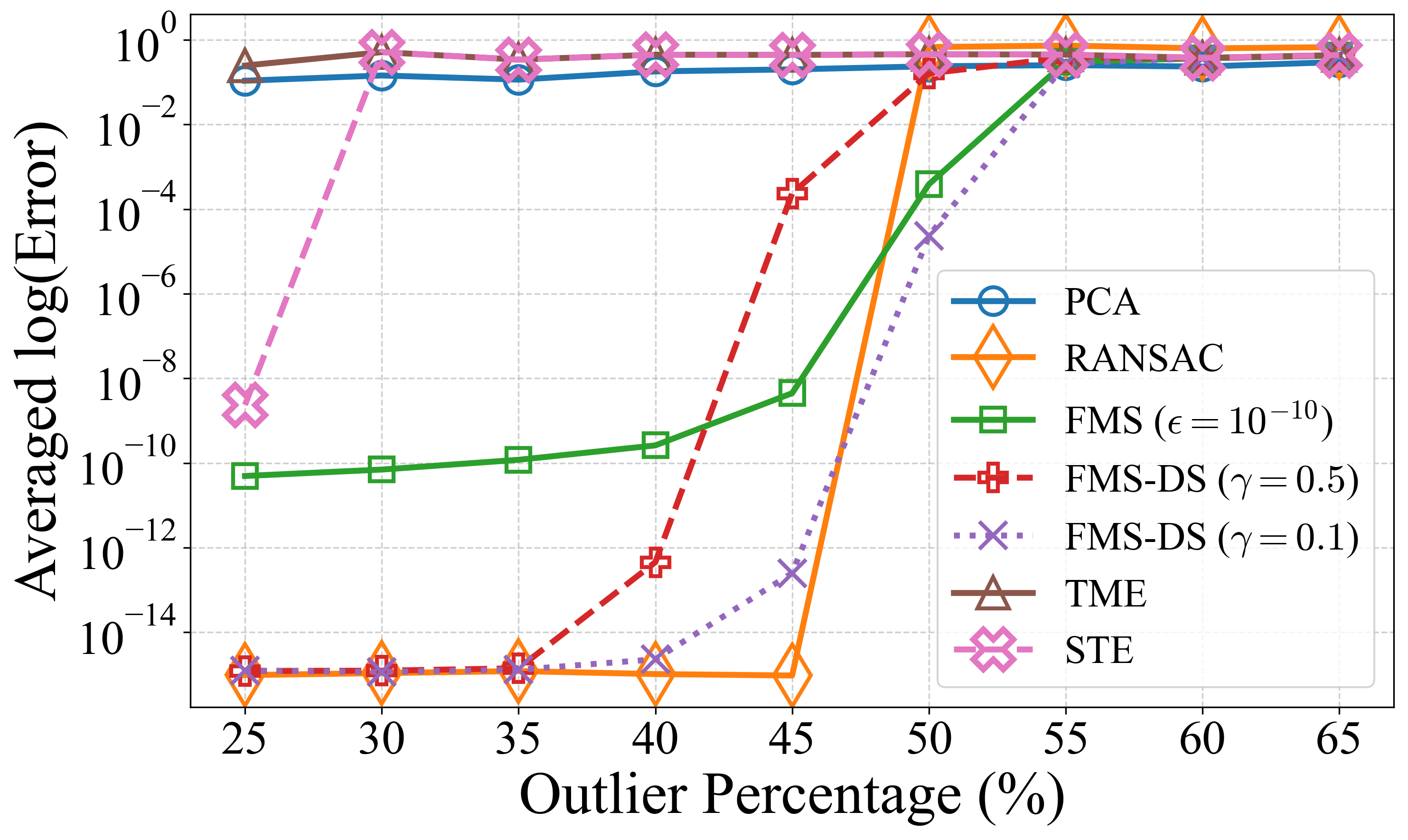} &
        \plot{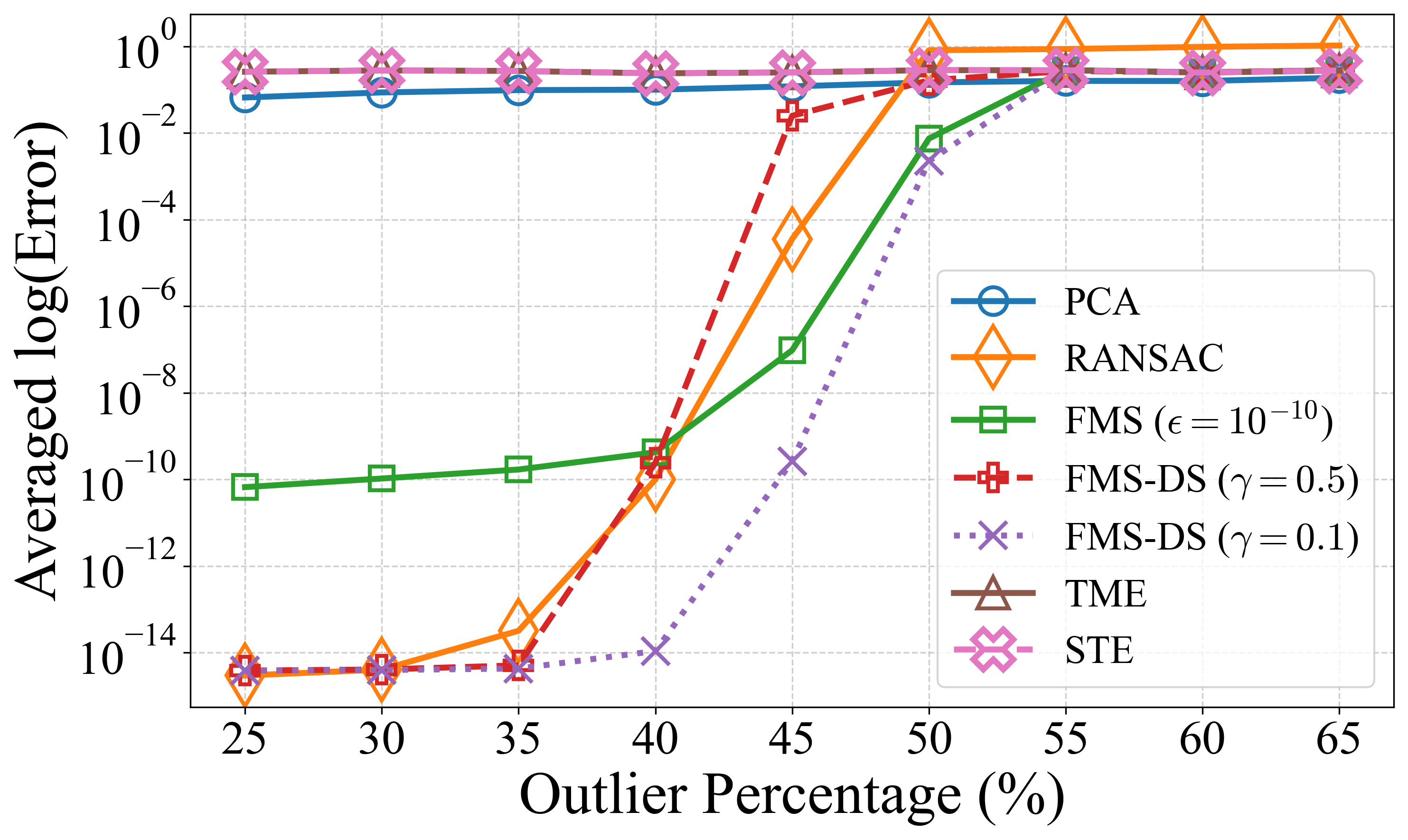}&
        \plot{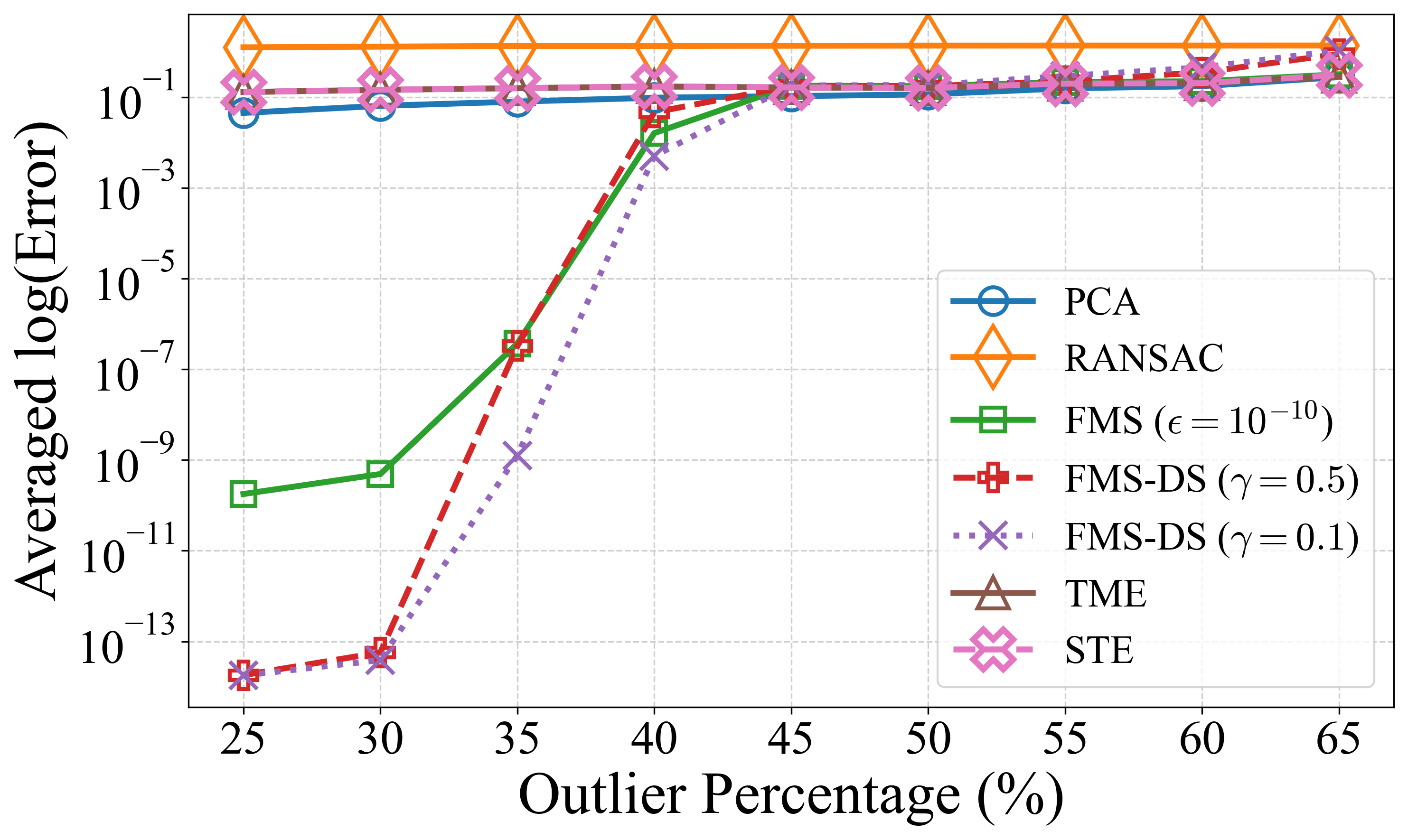} \\
      \noalign{\vskip 6pt}  %

      \rowlabel{$d_{\mathrm{out}} = 5$} &
        \plot{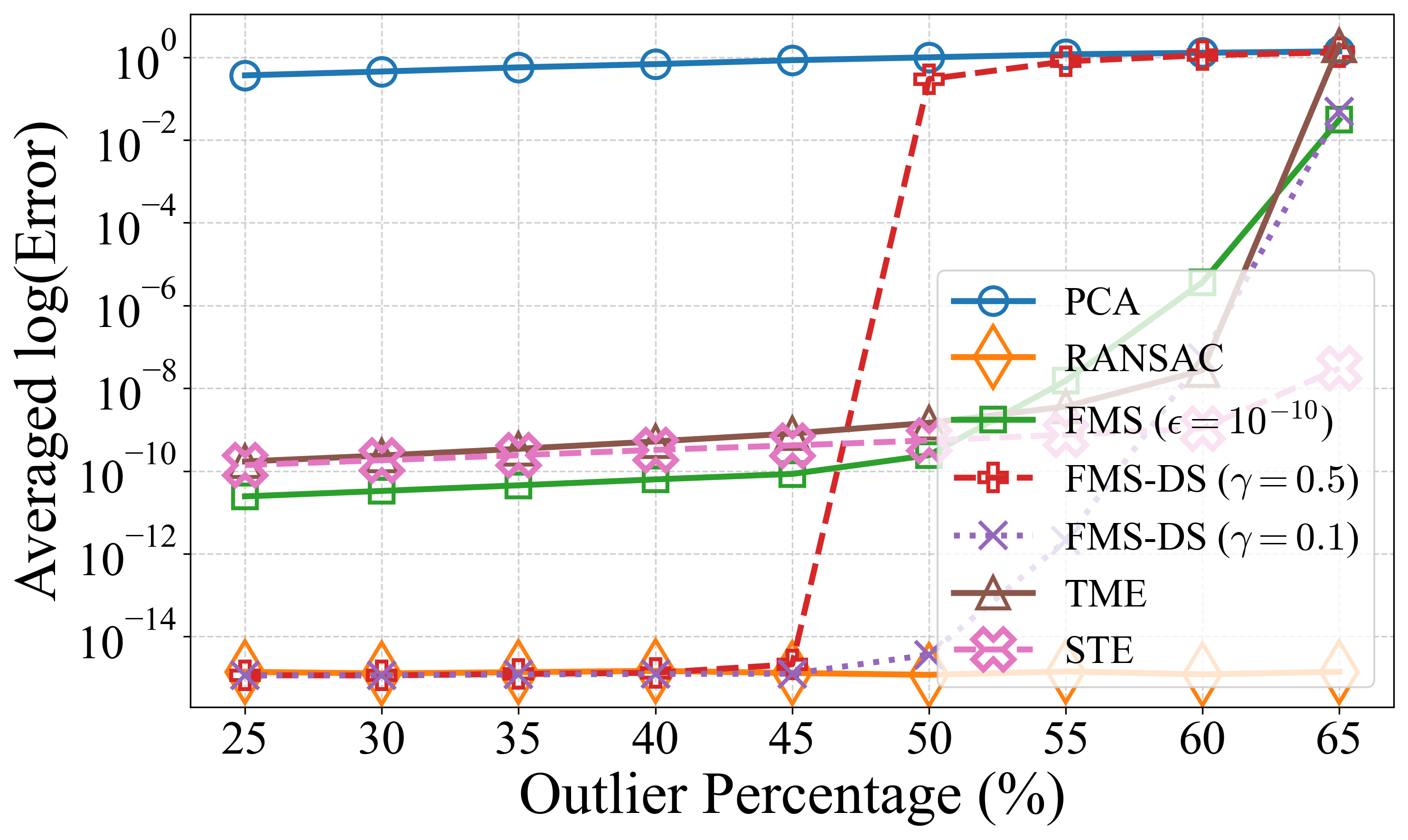} &
        \plot{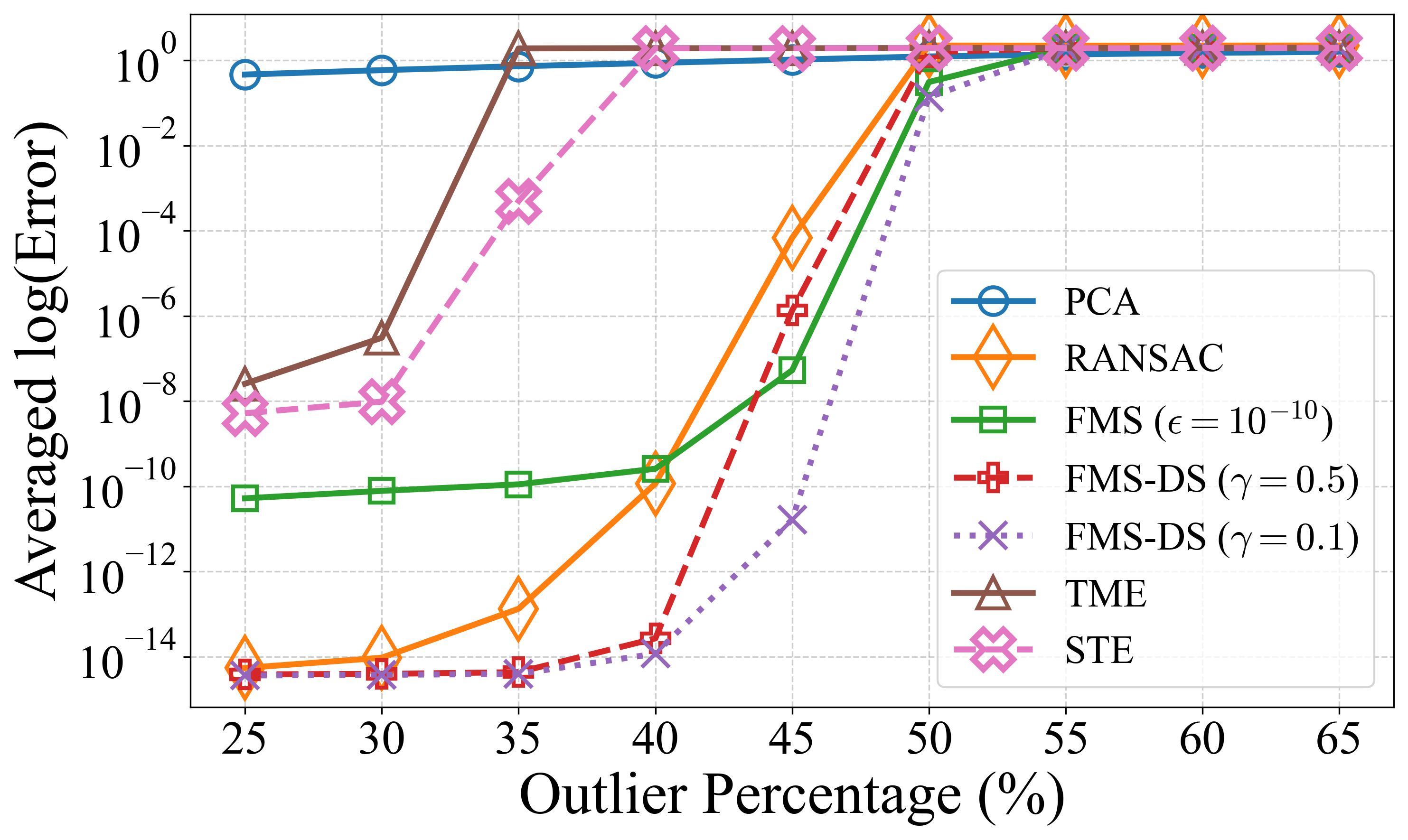}&
        \plot{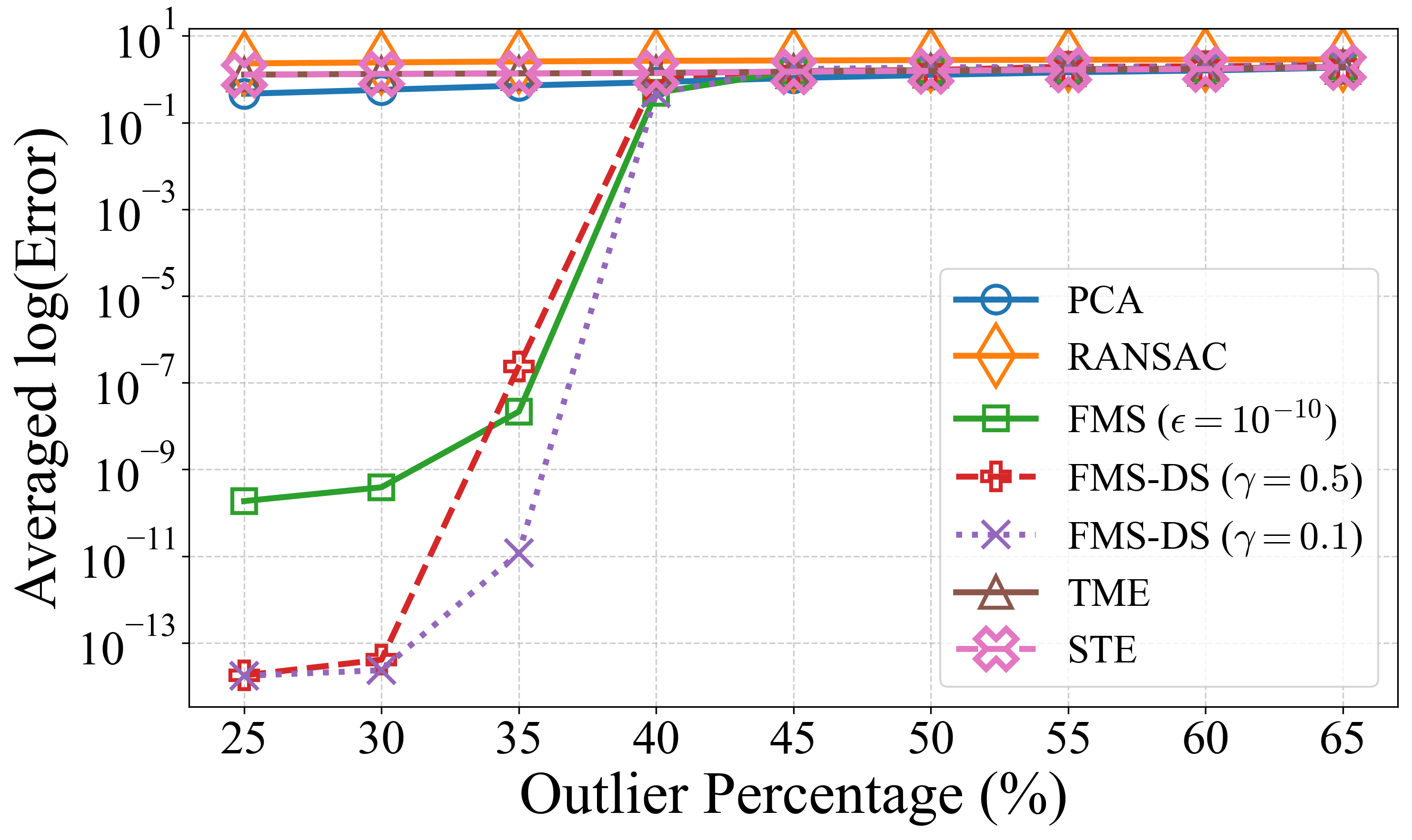} \\
      \noalign{\vskip 6pt}

      \rowlabel{$d_{\mathrm{out}} = 10$} &
        \plot{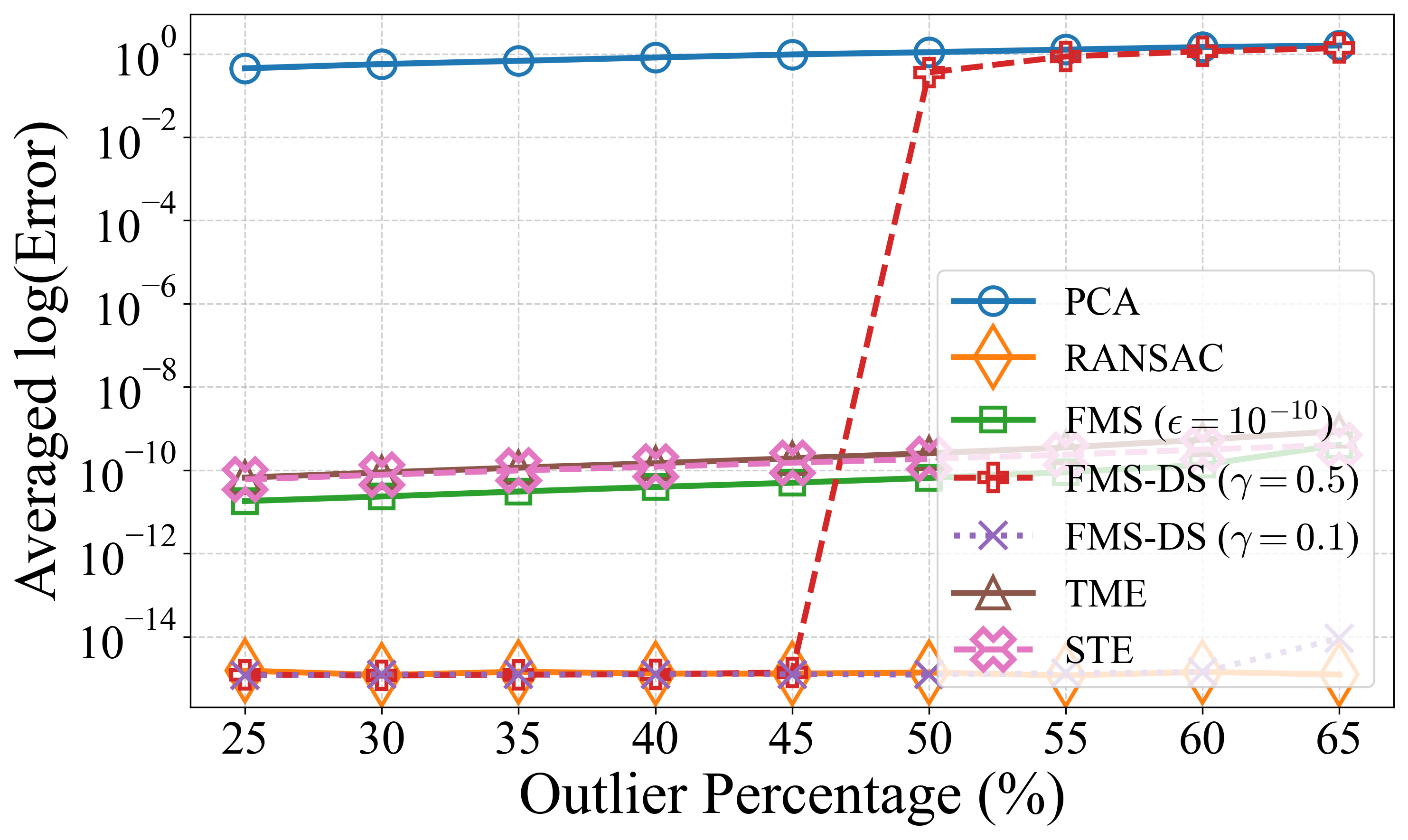} &
        \plot{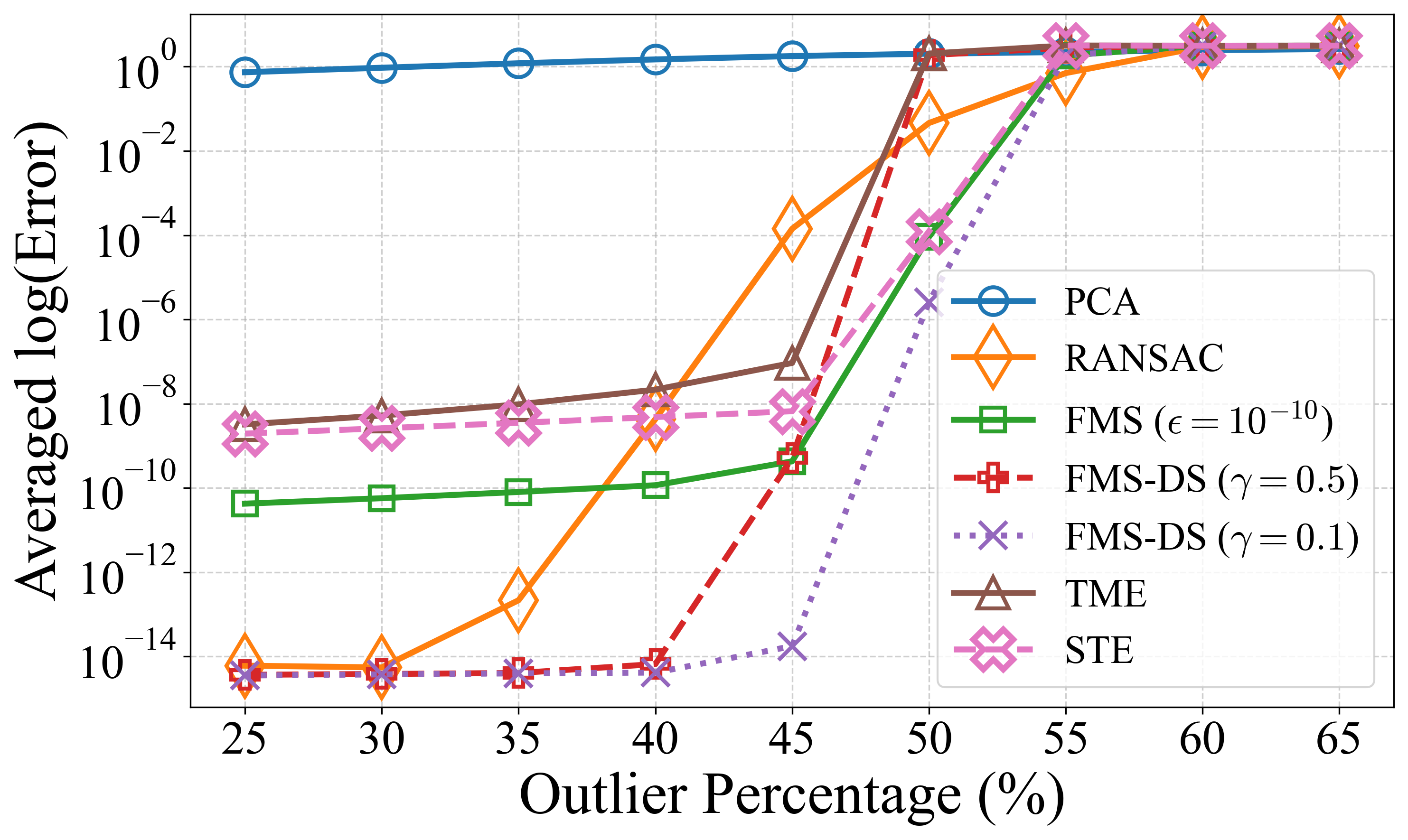}&
        \plot{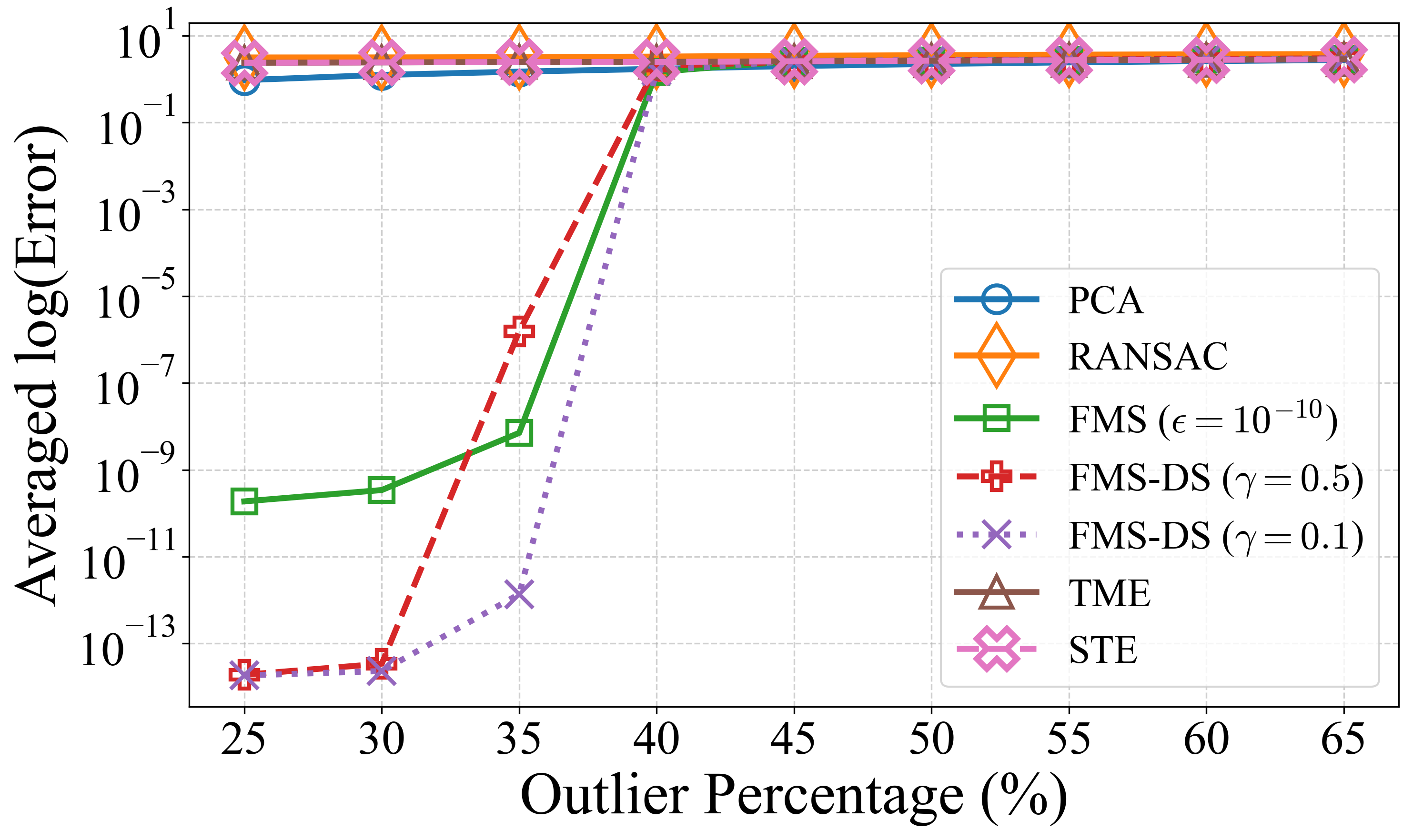} \\   \noalign{\vskip 6pt}

      \rowlabel{$d_{\mathrm{out}} = 50$} &
        \plot{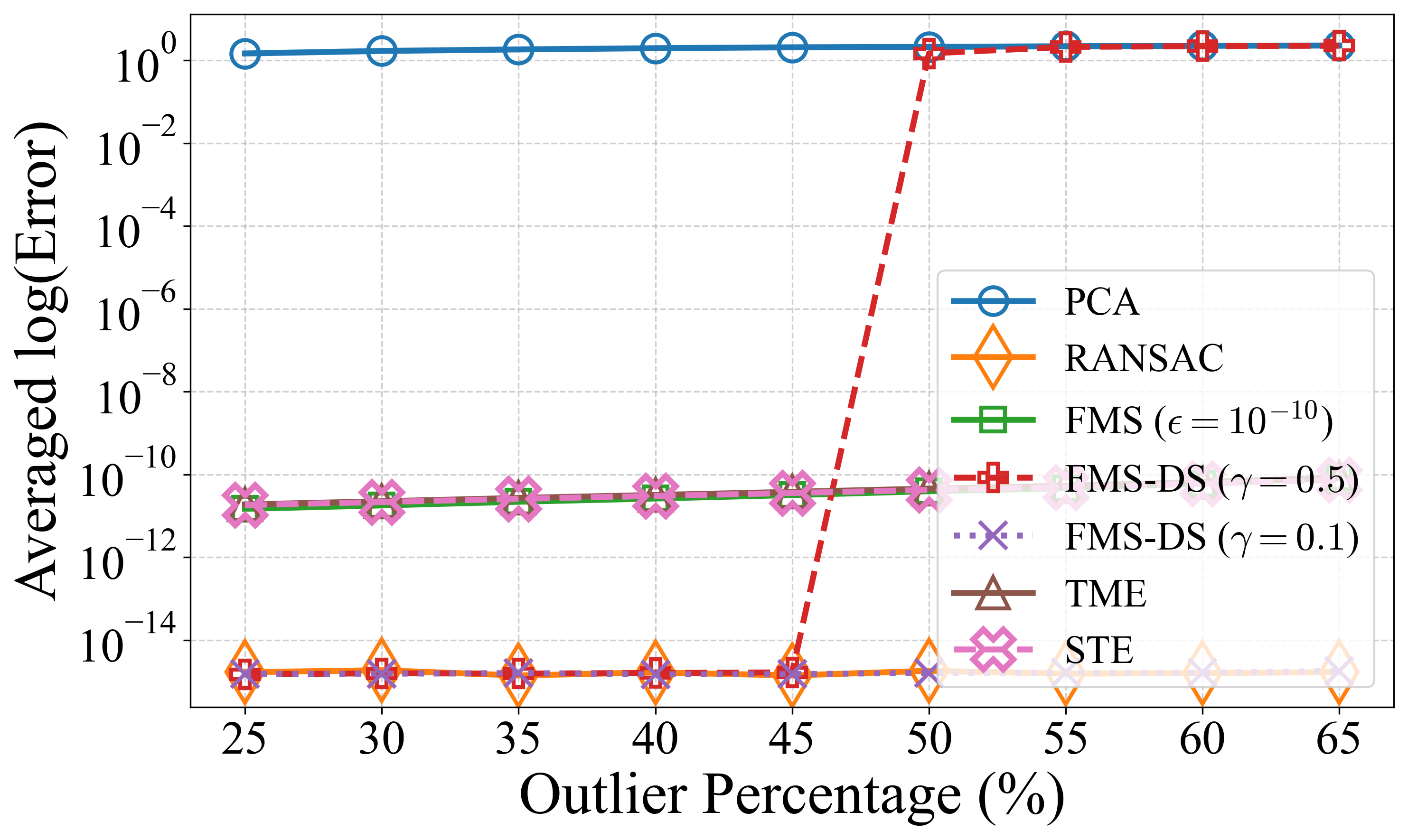} &
        \plot{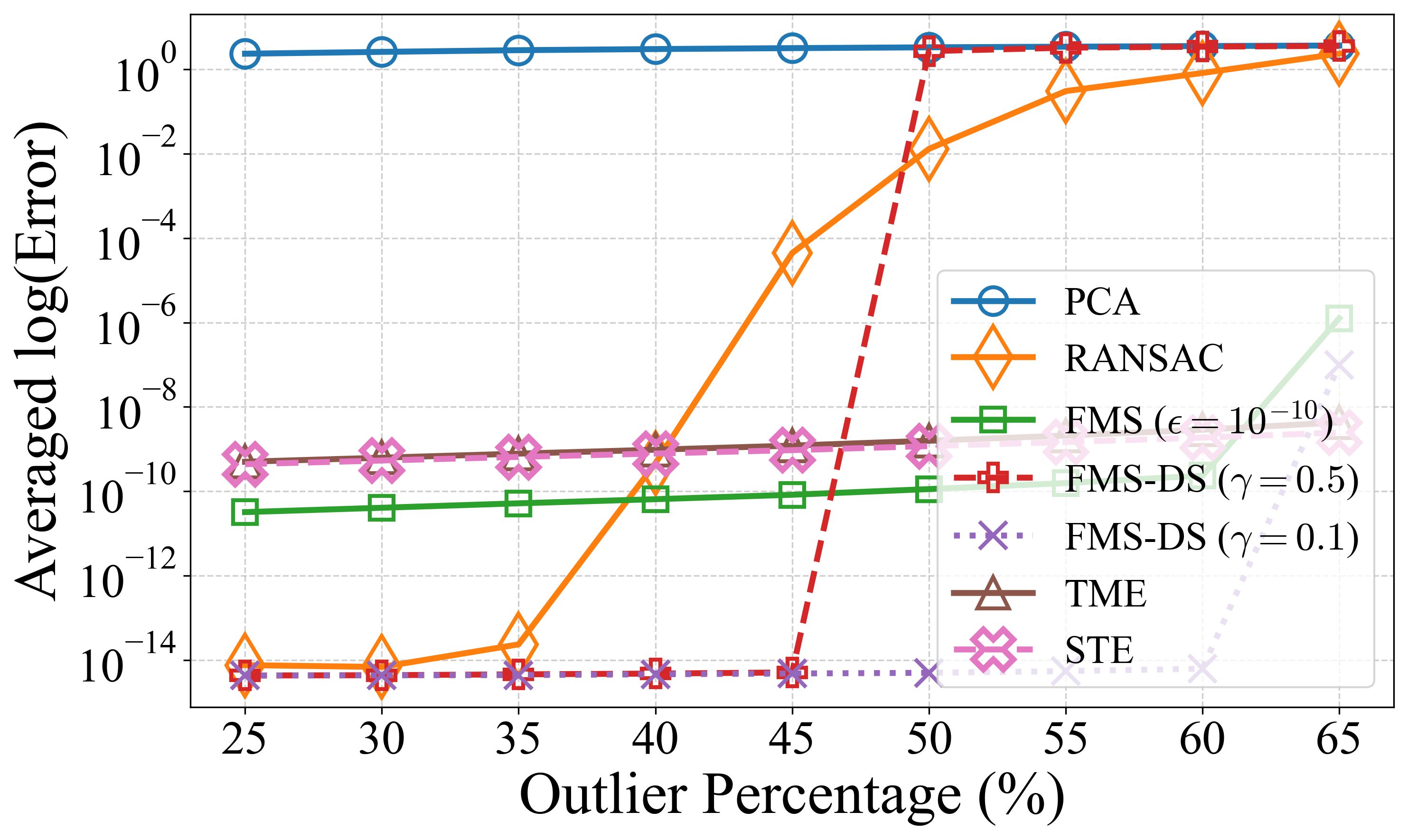}&
        \plot{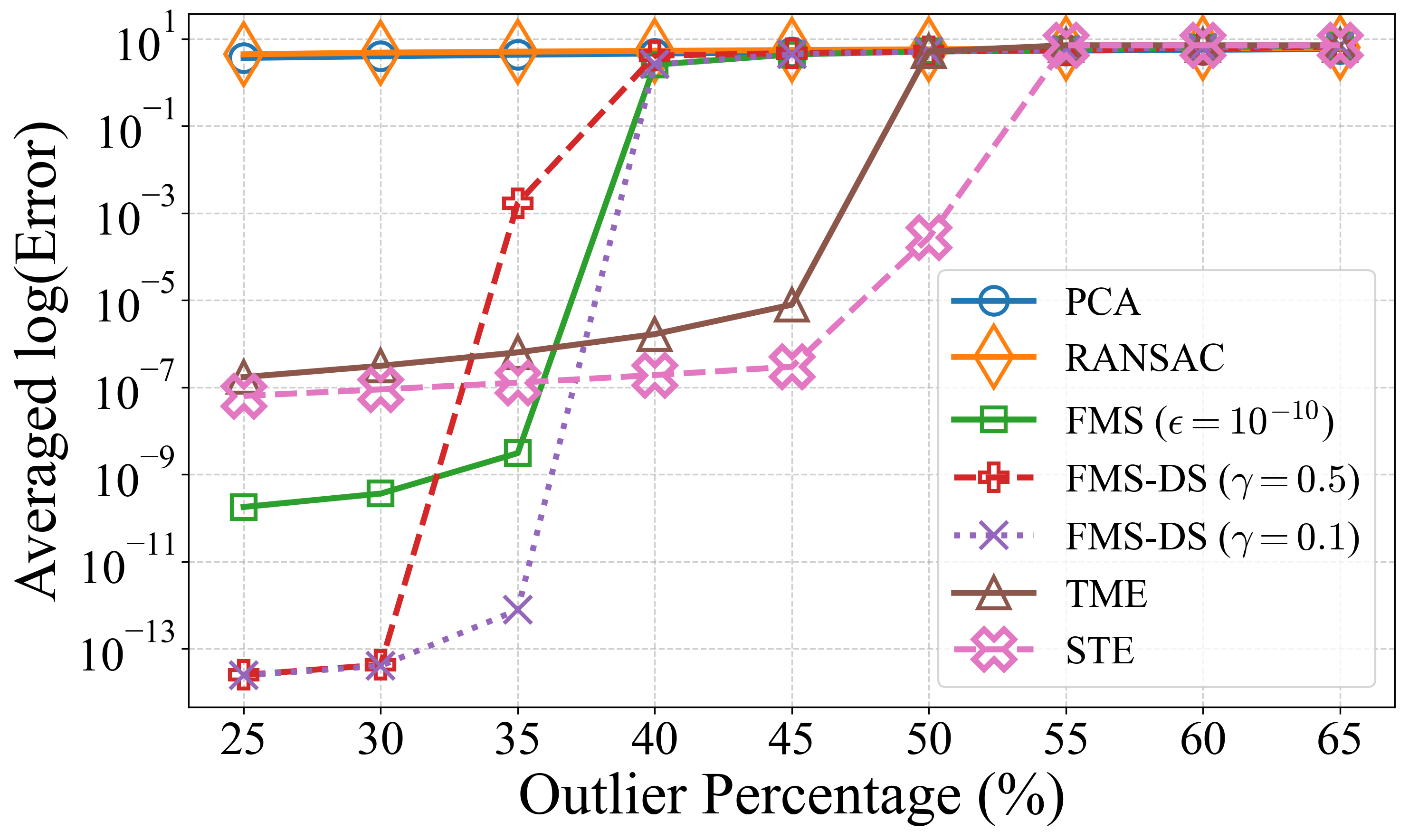} \\
  \end{tabular}
  
  \caption{Performance comparison of PCA, TME, RANSAC, STE, and FMS on synthetic data, reporting the geometric mean of the error over 200 repetitions on a log scale. This corresponds to the averaged log-error. We vary outlier dimension across rows and inlier dimension across columns as follows: top row: outlier dimension \( d_{\mathrm{out}} = 1 \), middle rows: \( d_{\mathrm{out}} = 5,10 \), bottom row: \( d_{\mathrm{out}} = 50 \). Left column: inlier dimension \( d = 3 \), middle column: \( d = 10 \), right column: \(d=50\). As we can see, FMS-DS with small $\gamma$ performs well across a range of settings. RANSAC performs well for small $d$, but fails for larger $d$ due to the fact that its runtime is exponential in $d$, and we cap the number of iterations.} \label{fig:exp1}
  \end{figure}
   
As illustrated in Figure~\ref{fig:exp1}, FMS performs comparably to STE and TME and usually outperforms them, particularly when the inlier dimension is higher. RANSAC, which follows a fundamentally different approach, performs well in low-dimensional settings but degrades significantly as dimensionality increases. This is because RANSAC estimates candidate subspaces from randomly selected sets of $d$ samples, and its success depends on all $d$ samples being inliers. This condition becomes exponentially unlikely as $d$ grows larger.

{
The only setting in which FMS underperforms relative to STE and TME is the scenario $d = d_{\mathrm{out}} = 50$. This behavior is likely due to the relatively small sample size ($n = 160$) compared to the ambient dimension, which results in a data matrix with a large condition number. In Figure~\ref{fig:exp1_2}, we repeat the simulation with larger sample sizes of $n = 200, 300, 400$. The results in Figure~\ref{fig:exp1_2} show that as $n$ increases, the performance of FMS and FMS-DS improves substantially and becomes comparable to that of STE and TME, while RANSAC consistently performs poorly across all settings.

\begin{figure}[!ht]
  \centering
  \setlength{\tabcolsep}{3mm}
  \renewcommand{\arraystretch}{0}

\begin{tabular}{@{}>{\centering\arraybackslash}m{.05\linewidth}
                @{\hspace{2.5mm}}
                >{\centering\arraybackslash}m{.3\linewidth}
                @{\hspace{2.5mm}}
                >{\centering\arraybackslash}m{.3\linewidth}
                @{\hspace{2.5mm}}
                >{\centering\arraybackslash}m{.3\linewidth}@{}}
      & \textbf{$n=200$}
      & \textbf{$n=300$} & \textbf{$n=400$} \\  \noalign{\vskip 3mm}

      \rowlabel{$d=d_{\mathrm{out}} = 50$} &
        \plot{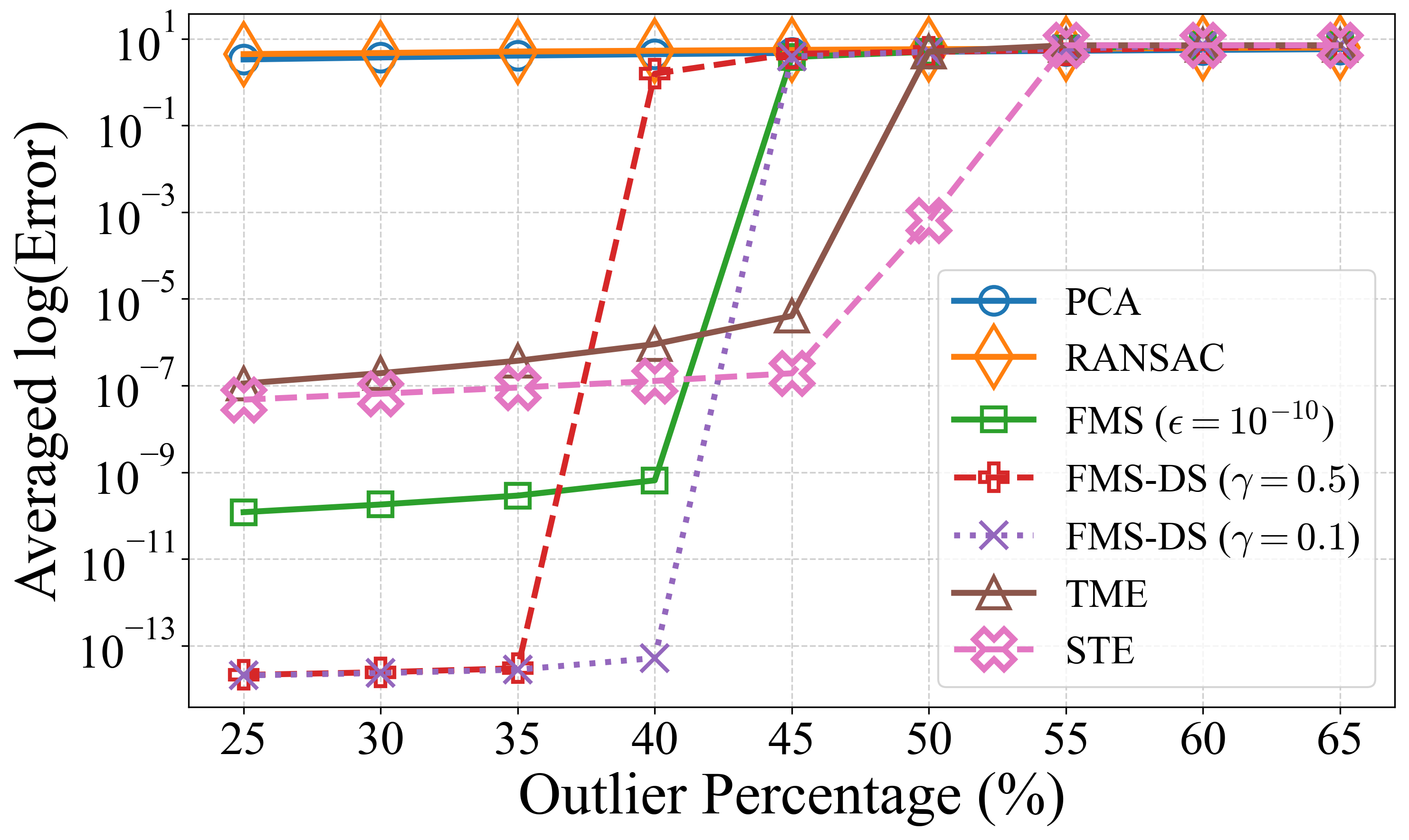} &
        \plot{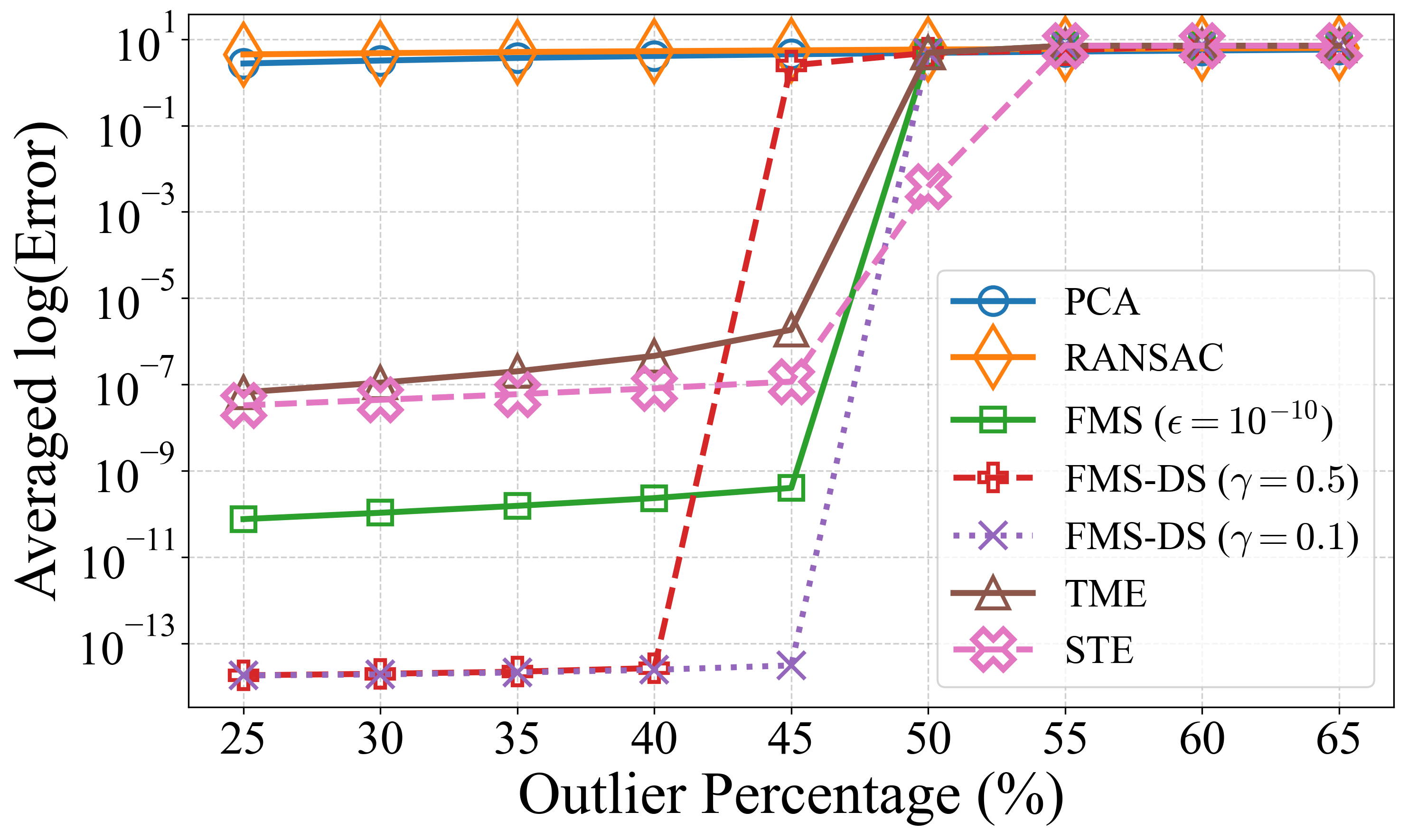}&
        \plot{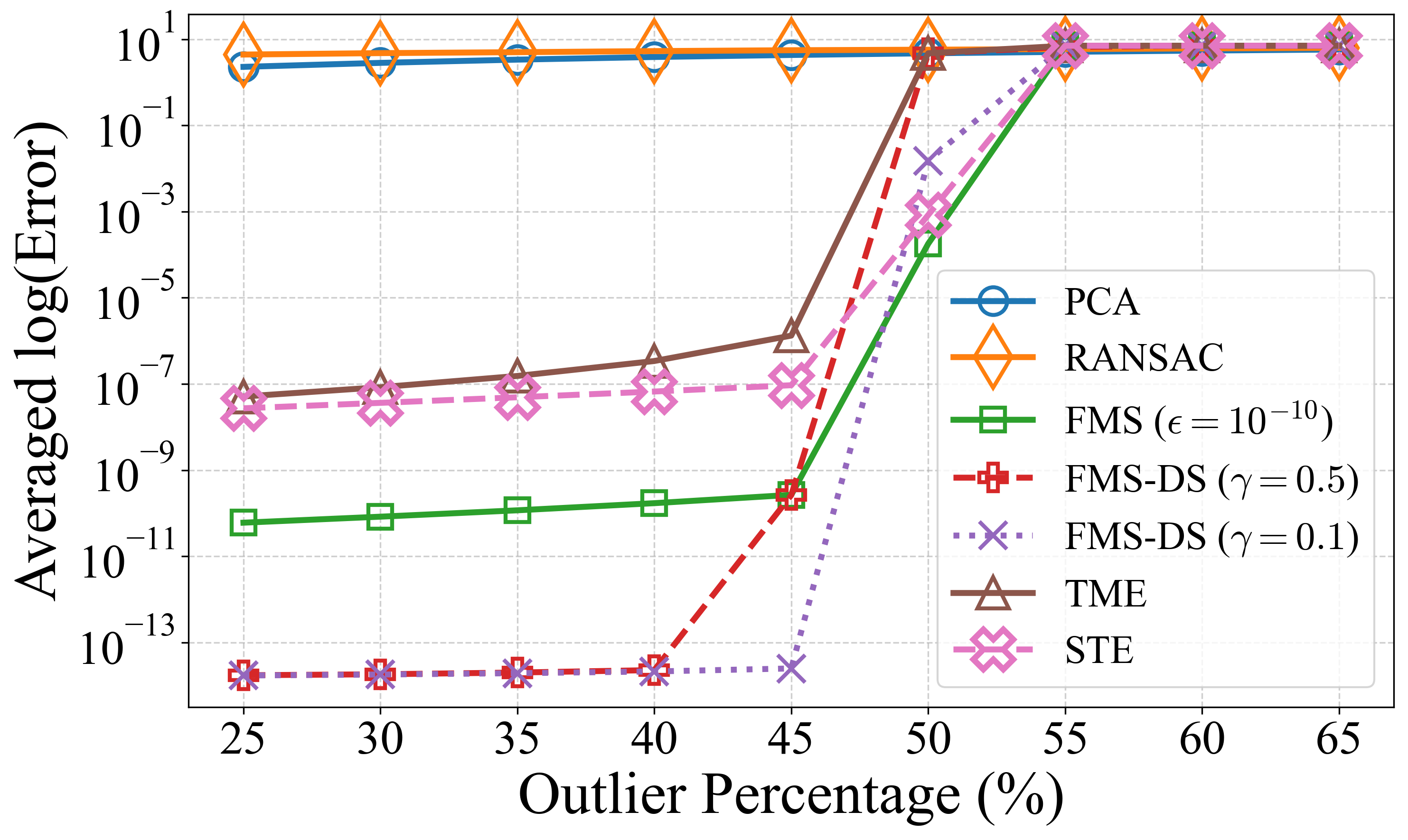} \\
  
  \end{tabular}
  
  \caption{{Performance comparison of PCA, TME, RANSAC, STE, and FMS on synthetic data, reporting the geometric mean of the error over 200 repetitions on a log scale. This corresponds to the averaged log-error. The inlier and outlier dimensions are fixed at 
$d = 50$ and $d_{\mathrm{out}} = 50$, while the number of samples $n$ varies: 
left: $n = 200$, middle: $n = 300$, right: $n = 400$. 
As $n$ increases, the performance of FMS and FMS-DS improves substantially and becomes comparable to STE and TME, 
whereas RANSAC performs poorly across all values of $n$.}} \label{fig:exp1_2}
  \end{figure}}

Since STE, FMS, and TME are applied with fixed regularization parameters, as no dynamic smoothing schemes are available for them, they do not reach machine precision and have performance comparable to FMS with fixed regularization $\epsilon = 10^{-10}$. On the other hand, applying FMS with fixed $\epsilon=10^{-15}$, which we do not display to prevent clutter, has the same performance as FMS-DS with $\gamma = 0.1$, and thus it can reach machine precision.

\subsection{Experiment 2: Effect of Regularization in FMS}
\label{subsec:exp2}

We investigate how different regularization strategies influence the performance of FMS in a setting where the initialization is chosen adversarially. In particular, we aim to demonstrate that dynamic smoothing can avoid saddle points that FMS with fixed regularization may become stuck at. 

We compare the following variants:
\begin{itemize}
    \item FMS with fixed regularization parameters of $\epsilon=10^{-3}, 10^{-10}$, and $10^{-15}$.
    \item FMS-DS where $\epsilon$ is chosen by \eqref{eq:epsilon_choose} with \( \gamma = 0.1 \) and $0.5$. Here, we set the initial value $\epsilon^{(0)}$ to be the $\gamma$-quantile of the $\{\dist(\bx,L_{0})\}_{\bx\in\calX}$, where $L_{0}$ represents the initial subspace.
\end{itemize}
Similar to Experiment 1, the synthetic dataset is constructed with an inlier subspace dimension of \( d = 3 \) and an outlier dimension of \( d_{\mathrm{out}} = 1 \). The ambient space has dimension \( D = d + d_{\mathrm{out}} = 4 \). The dataset contains a total of 200 data points. To demonstrate that the dynamic smoothing strategy can escape local minima or stationary points, we initialize the algorithm at a stationary point of the FMS objective---specifically, a subspace with two directions in the inlier subspace and one direction orthogonal to it. We again generate 200 datasets in this model and average the results for each algorithm. In this case, FMS is initialized with two directions within the inlier subspace and one direction orthogonal to it that contains the outliers.

Figures~\ref{fig:exp2} and \ref{fig:exp2_3} report the results of this experiment. In Figure \ref{fig:exp2}, we report the averaged log-error of the different methods as a function of outlier percentage. We also report the failure rate, which corresponds to FMS getting stuck at a bad stationary point, as a function of outlier percentage in Figure \ref{fig:exp2_3}.
As shown in these figures, FMS-DS with a larger $\gamma$ significantly outperforms both fixed regularization approaches, regardless of the fixed regularization strength. 

It is worth noting that FMS with a large regularization parameter consistently results in higher error due to the overly strong regularization. On the other hand, while FMS with small regularization can perform well in some cases, it begins to fail when the outlier percentage is sufficiently large (around $15\%$), in part because it becomes trapped in stationary points. In contrast, FMS with dynamic smoothing achieves consistently better performance across a range of outlier ratios.

\begin{figure}[ht]
    \centering
    \includegraphics[width=0.49\linewidth]{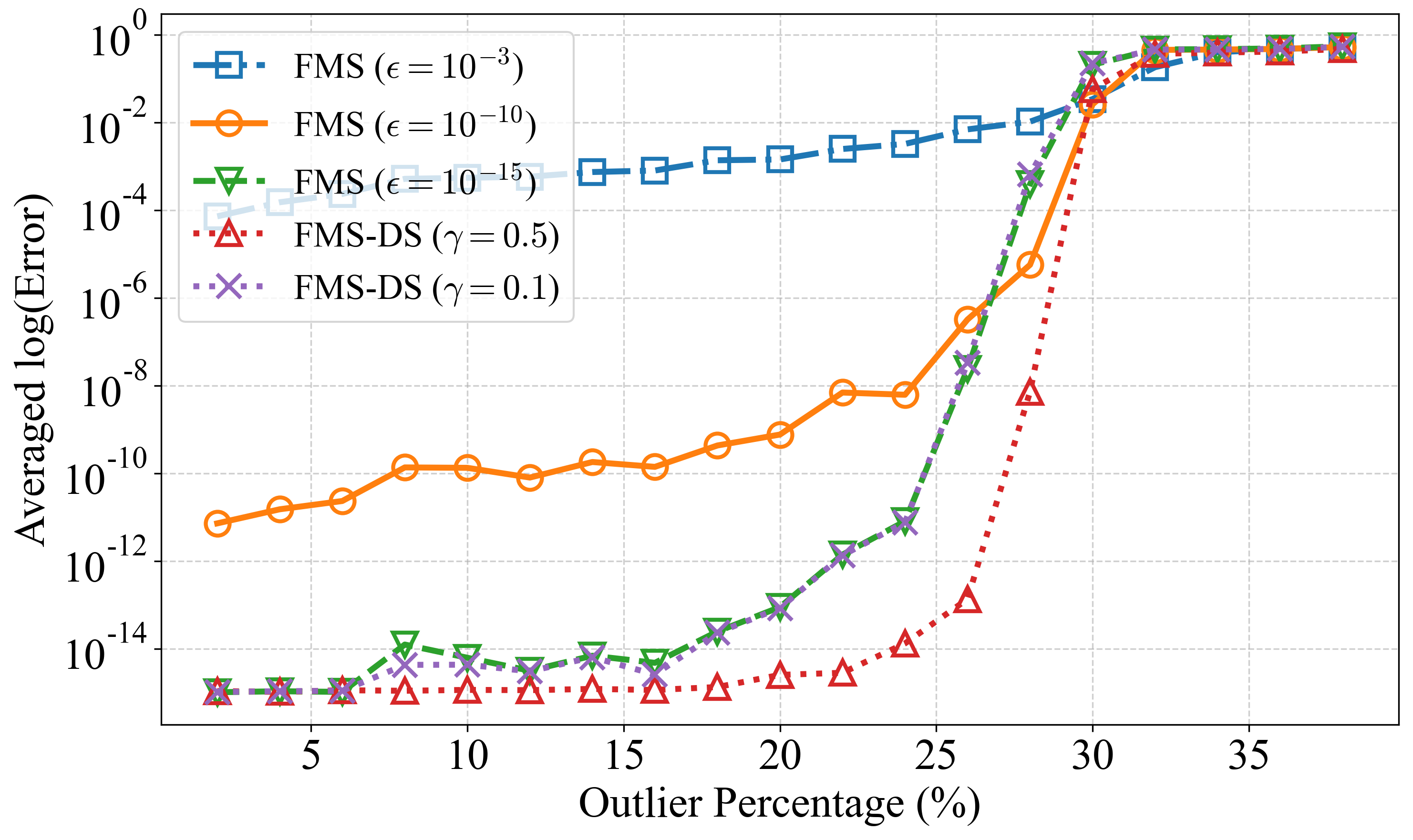}
    \caption{Averaged log-error versus outlier percentage for the experiment with adversarial initialization. Here, the inlier subspace dimension is 3 and the outlier subspace is a 1-dimensional subspace orthogonal to it. FMS is initialized with two directions within the inlier subspace and one direction orthogonal to it. 
    }
    \label{fig:exp2}
\end{figure}

\begin{figure}[ht]
    \centering
    \begin{minipage}{.49\linewidth}
        \centering
        Full View \\ \vspace{.1cm}
        \includegraphics[width=\linewidth]{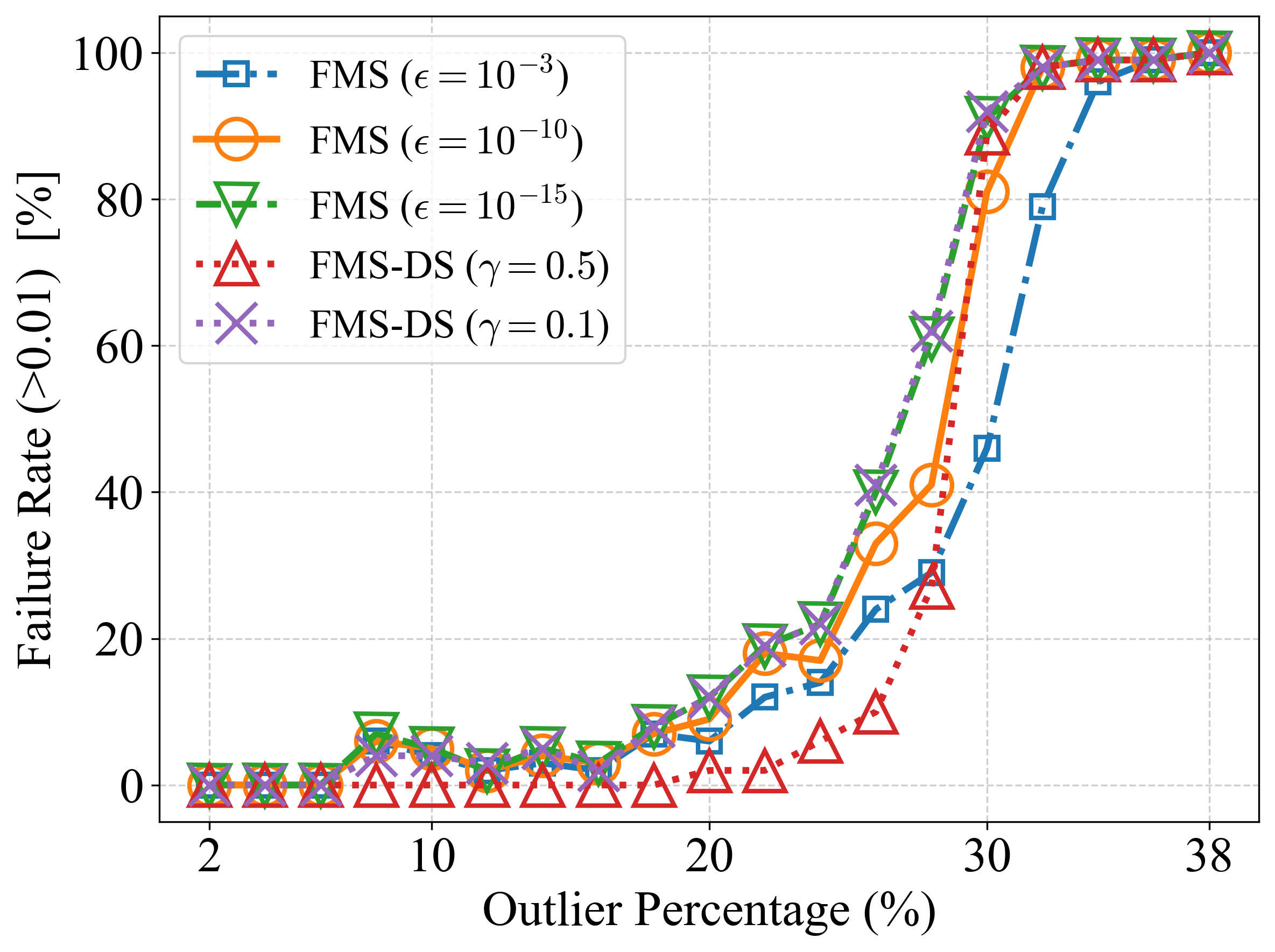}
    \end{minipage}
    \begin{minipage}{.49\linewidth}
        \centering
        Zoomed-in View \\ \vspace{.1cm}
        \includegraphics[width=\linewidth]{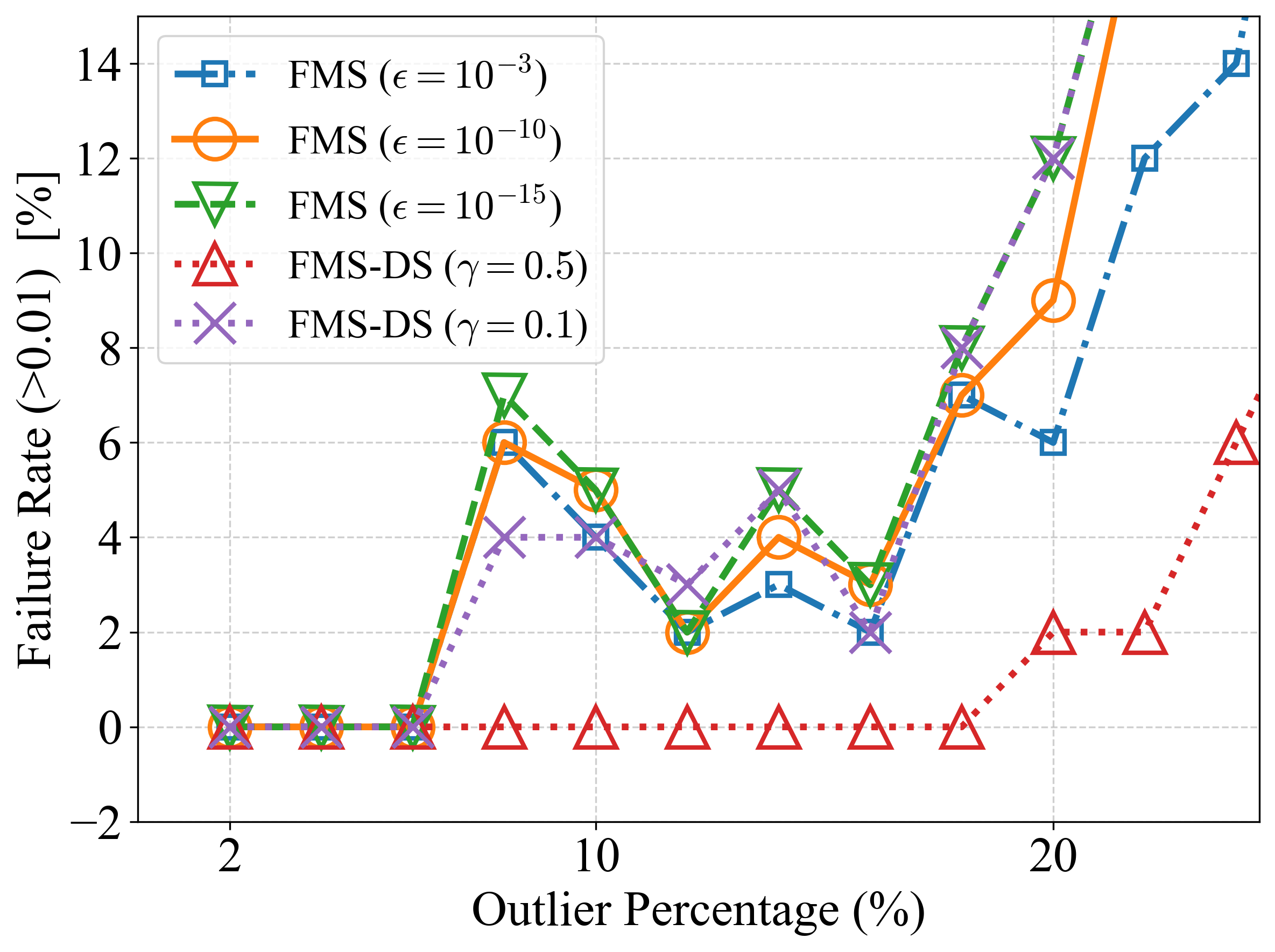}
    \end{minipage}
    \caption{Failure rate of the various FMS algorithms under the same setting as Figure \ref{fig:exp2}, for different outlier percentages. Left: full range of outlier percentages. Right: zoomed-in view.}
    \label{fig:exp2_3}
\end{figure}

\subsection{Experiment 3: Iteration-Wise Error Analysis}
\label{subsec:exp3}

We further investigate the convergence behavior of FMS by plotting the estimation error over iterations for each regularization strategy. We follow the settings of \( d \) and \( d_{\mathrm{out}} \) from Experiment 2, but fix the number of inliers at 100 and outliers at 30. The algorithm is initialized using the PCA subspace as in Experiment 1, or the orthogonal subspace as in Experiment 2. 

In Figure \ref{fig:exp3}, we plot the log-error versus iteration for FMS with the different regularization schemes. The left display of Figure~\ref{fig:exp3} uses PCA initialization and shows that it is comparable to FMS with a fixed small regularization of $\epsilon = 10^{-15}$. The rate of convergence is slightly slower than that of FMS with small regularization ($\epsilon = 10^{-15}$).  In contrast, for FMS with larger fixed regularization, the estimation error plateaus once it reaches the same order of magnitude as the regularization parameter. The right display of Figure \ref{fig:exp3} demonstrates the performance of the various regularization strategies with the poor initialization of Experiment 2. As we can see, FMS-DS with $\gamma = 0.5$ and FMS with fixed $\epsilon = 10^{-3}$ are more effective at escaping the saddle point. However, the fixed regularization is not able to adapt and plateaus at the fixed error of $\approx 10^{-3}$, while FMS-DS converges to the global solution. FMS-DS with $\gamma=0.1$ and FMS with small fixed $\epsilon=10^{-15}$ perform comparably.

\begin{figure}[ht]
     \centering
    \begin{minipage}{0.49\linewidth}
        \centering
        PCA Initialization (Exp.~1) \\ \vspace{0.1cm} 
        \includegraphics[width=\linewidth]{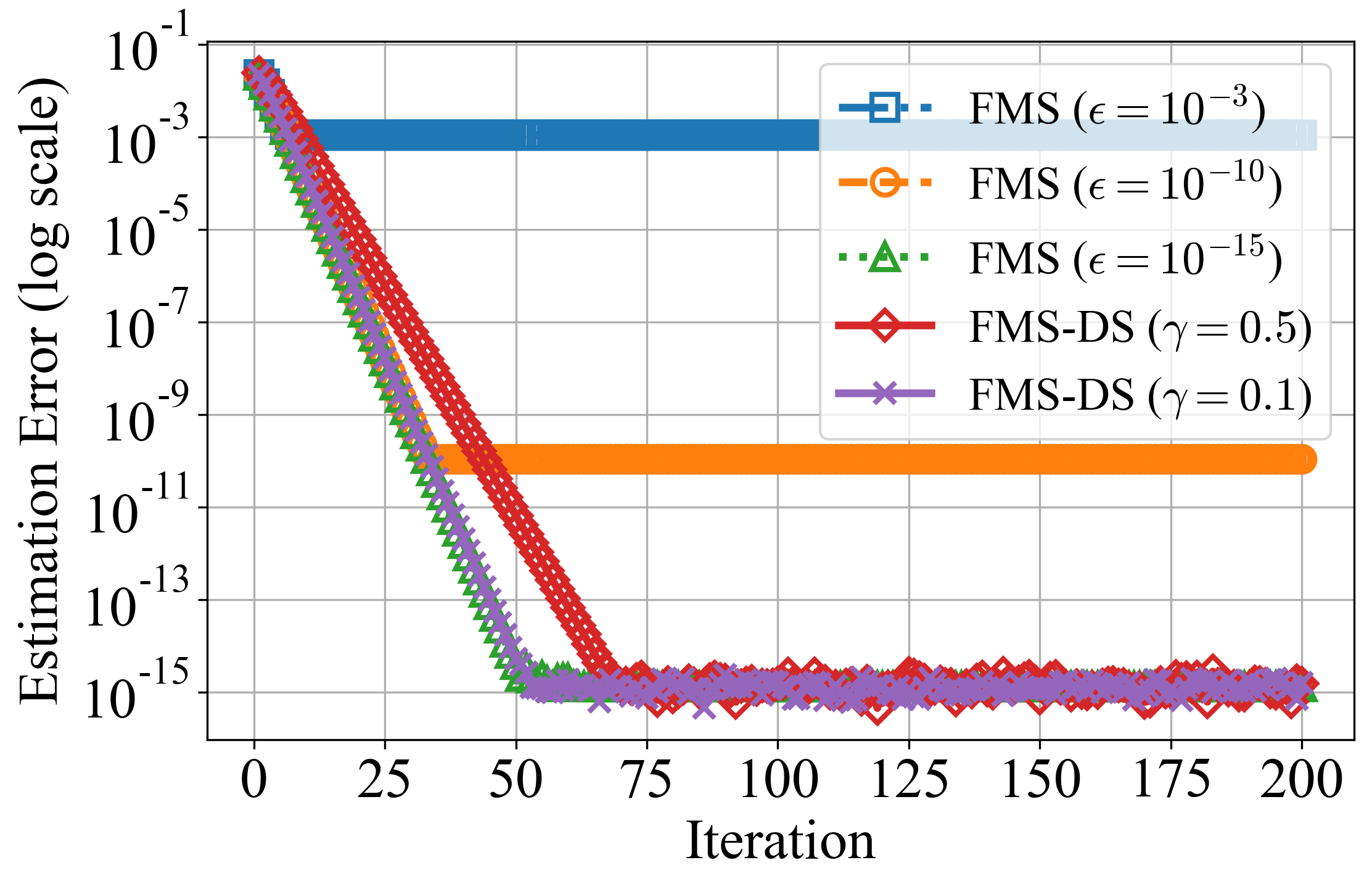}
    \end{minipage}
    \begin{minipage}{0.49\linewidth}
        \centering
        Adversarial Initialization (Exp.~2) \\ \vspace{0.1cm}
        \includegraphics[width=\linewidth]{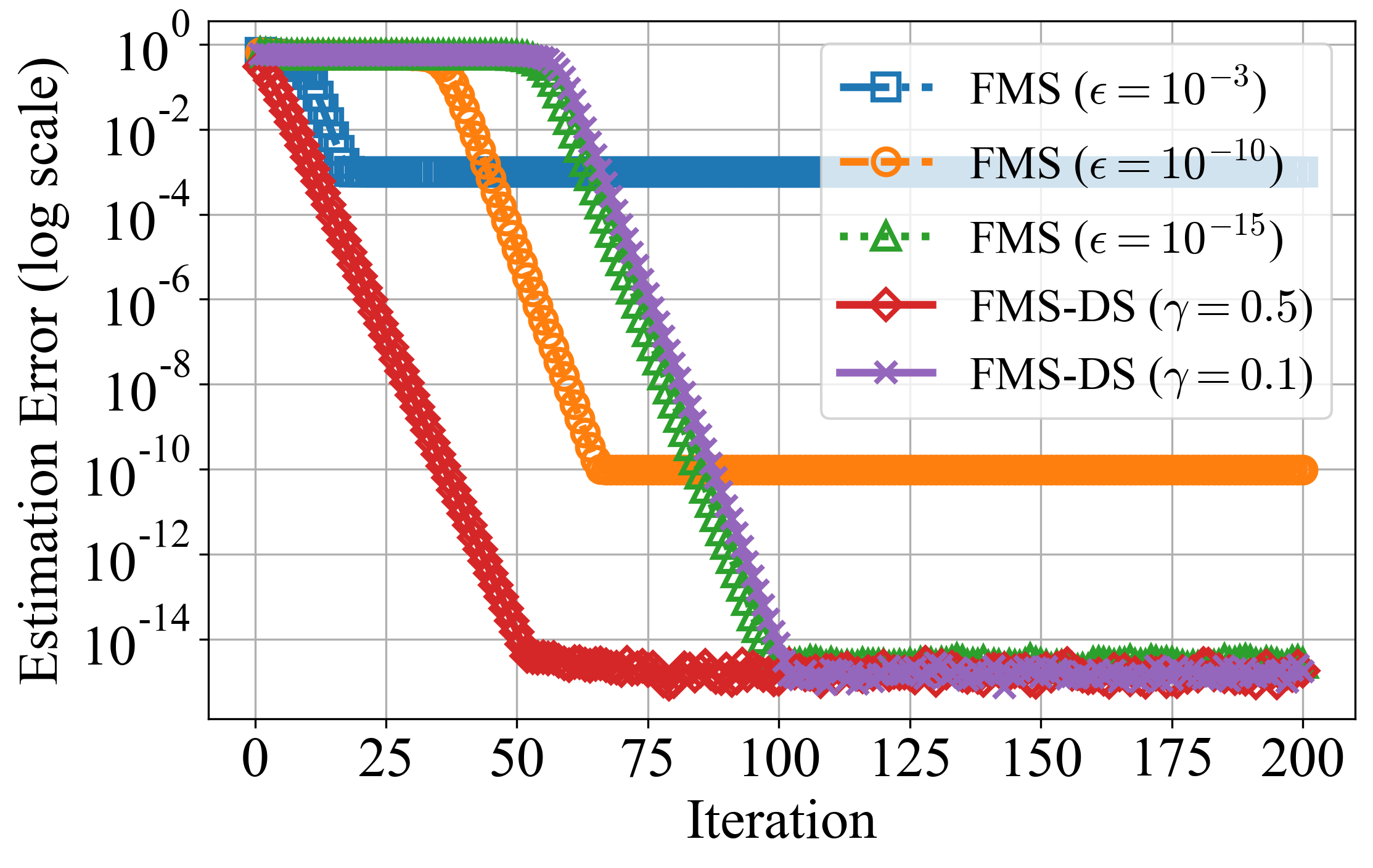}
    \end{minipage}
    
    \caption{Averaged log-error versus iteration for FMS with different regularization strategies. On the left, we demonstrate convergence with PCA initialization as in Experiment 1, and on the right, we demonstrate convergence with the orthogonal initialization of Experiment 2. The left plot demonstrates that in settings with good initialization and lack of bad stationary points, FMS with small fixed $\epsilon=10^{-15}$ and FMS-DS with small $\gamma=0.1$ both perform well. On the other hand, for bad initialization, we see that FMS with larger $\epsilon$ and FMS-DS with larger $\gamma$ escape the saddle point faster.}
    \label{fig:exp3}
\end{figure}

\subsection{A Practical Demonstration of the FMS Algorithm}\label{sec:practical}

While the previous section provides a new theoretical analysis of the FMS algorithm and extends it to affine settings, this analysis applies only in inlier-outlier settings. However, it is typically unknown a priori whether a dataset satisfies this condition. 
To further motivate the FMS algorithm and its role in robust subspace estimation, we present an application related to low-dimensional neural network training. 

A line of work studies subspace-constrained optimization. One example is neural network training, where \cite{li2022low} use Dynamic Linear Dimensionality Reduction (DLDR) to show that training in a PCA subspace can improve generalization. 
Unfortunately, this procedure of subspace estimation lies at odds with the current understanding of neural network training. 
{
We argue that using vanilla PCA may be misguided since recent evidence indicates
that the stochastic gradient noise encountered in neural network training
is heavy-tailed
\citep{simsekli2019tailindex,zhou2020towards,gurbuzbalaban2021heavy}.
Specifically, one can decompose the stochastic gradient at each iteration into
a deterministic component (the expected gradient) and a mean-zero noise term
given by the deviation of the stochastic gradient from its expectation.
While earlier work modeled this noise in the vanishing step-size limit as Gaussian,
leading to an SDE with Brownian increments \citep{jastrzkebski2017three},
empirical evidence suggests instead that this noise exhibits heavy-tailed behavior,
often consistent with $\alpha$-stable laws rather than Gaussian tails
\citep{simsekli2019tailindex}.
}

The implementation details of DLDR are as follows \citep{li2022low}. First, during training with SGD, we sample \( t \) steps of neural network parameters \( \{\bw_1, \bw_2, \dots, \bw_t\} \subset \R^D \). The weights are then centered by subtracting the mean, and the PCA subspace is computed.  
The network is then retrained, with gradients projected to the PCA subspace.

Motivated by this, we replace the PCA subspace in the previous procedure with FMS and AFMS. For completeness, we also report the application of TME, PCA, and spherical PCA (SPCA), which computes the PCA subspace of the normalized dataset ${\bx/|\bx| : \bx \in \calX}$ \citep{Locantore1999}. We do not report dynamic smoothing, as it did not change the performance of FMS. In particular, dynamic smoothing typically affects the algorithm's rate of convergence and precision, but neither is essential in this example.

Specifically, we trained CIFAR-10 \citep{Krizhevsky2009LearningML} on ResNet-20, CIFAR-100 \citep{Krizhevsky2009LearningML} on ResNet-32, and Tiny ImageNet \citep{imagenet_cvpr09} on ResNet-18. The \textbf{CIFAR-10} dataset consists of 60,000 color images of size 32×32 pixels, categorized into 10 classes. It contains 50,000 training images and 10,000 test images. The \textbf{CIFAR-100} dataset is an extension of CIFAR-10, containing 100 classes instead of 10. Each class has 600 images, making a total of 60,000 images. The dataset consists of 50,000 training images and 10,000 test images of size 32×32 pixels. The \textbf{Tiny ImageNet} dataset contains 200 classes with images resized to 64×64 pixels. It includes 100,000 training images and 10,000 validation images. 

For all three datasets, the images were normalized using their channel-wise means and variances. The deep neural networks were trained with the SGD optimizer using a weight decay of 1e-4, a momentum of 0.9, and a batch size of 128. The learning rate was initially set to 0.1 and decayed to 0.01 over 100 epochs. Model parameters were sampled at the end of every epoch, after which FMS, PCA, and TME were applied to estimate the subspace. When training within this subspace using projected-SGD, we maintained the same batch size and momentum as in the original SGD setup. The learning rate was set to 1 initially and decayed to 0.1 over 30 epochs.

To assess robustness, we introduced additional label corruption by randomly selecting a fraction of the training data and assigning random labels to them. The testing accuracies of our simulations are visualized in Figure \ref{fig:low3}. The first row illustrates the accuracies of all methods applied to three datasets reduced to 20 dimensions, while the second row shows the accuracies when 15\% of the labels are randomly corrupted. Note that AFMS and FMS have the same performance in all scenarios. PCA and TME also perform similarly. As shown in Figure~\ref{fig:low3}, label corruption significantly degrades the performance of SGD. In contrast, training in low-dimensional subspaces demonstrates that FMS outperforms the other methods.

\begin{figure}[htbp]
  \centering
  \begin{minipage}[b]{0.31\linewidth}
    \centering
    CIFAR-10\\ \vspace{.1cm}
    \refstepcounter{subfig} 
    \includegraphics[width=\linewidth]{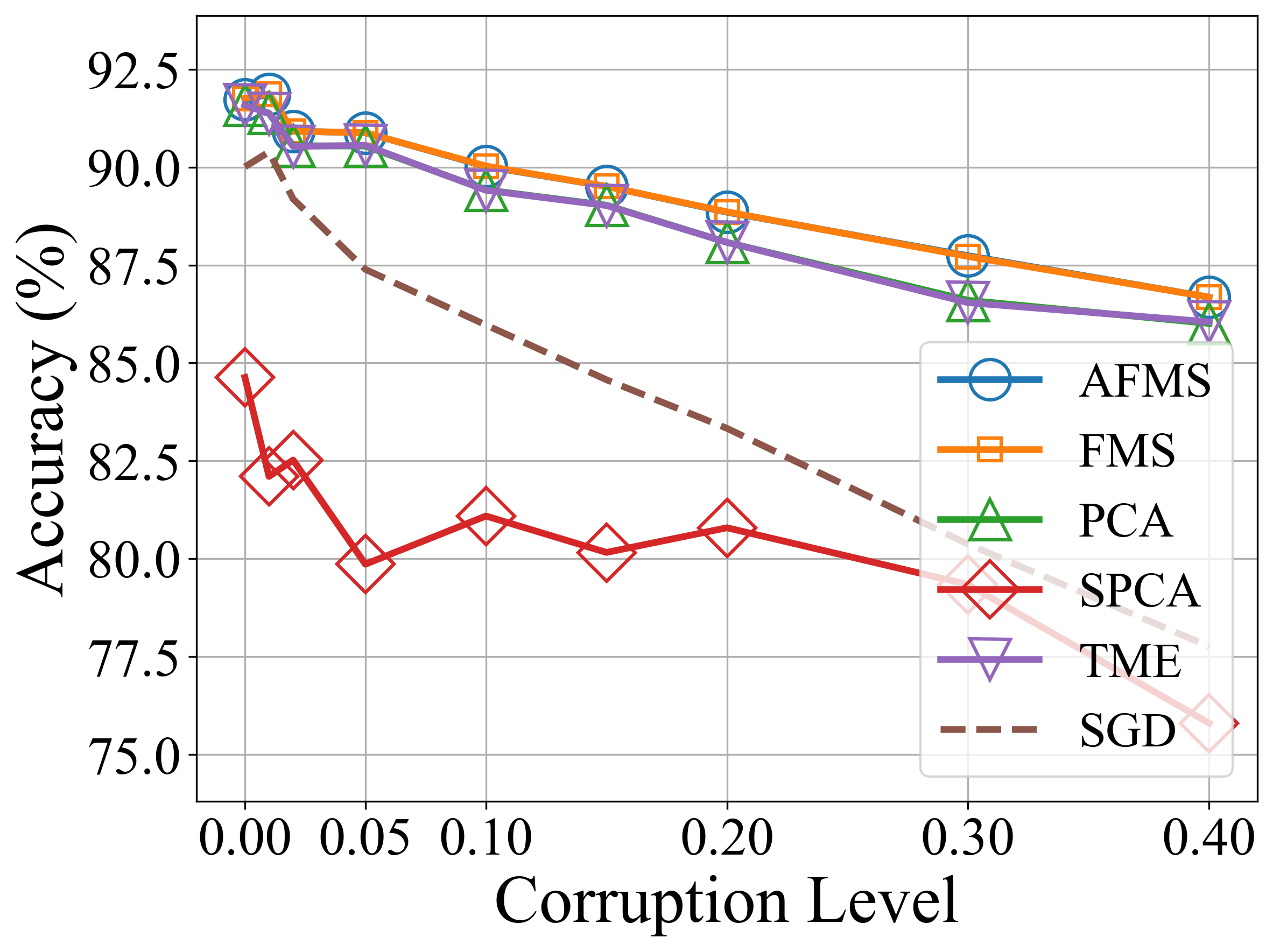}
    \includegraphics[width=\linewidth]{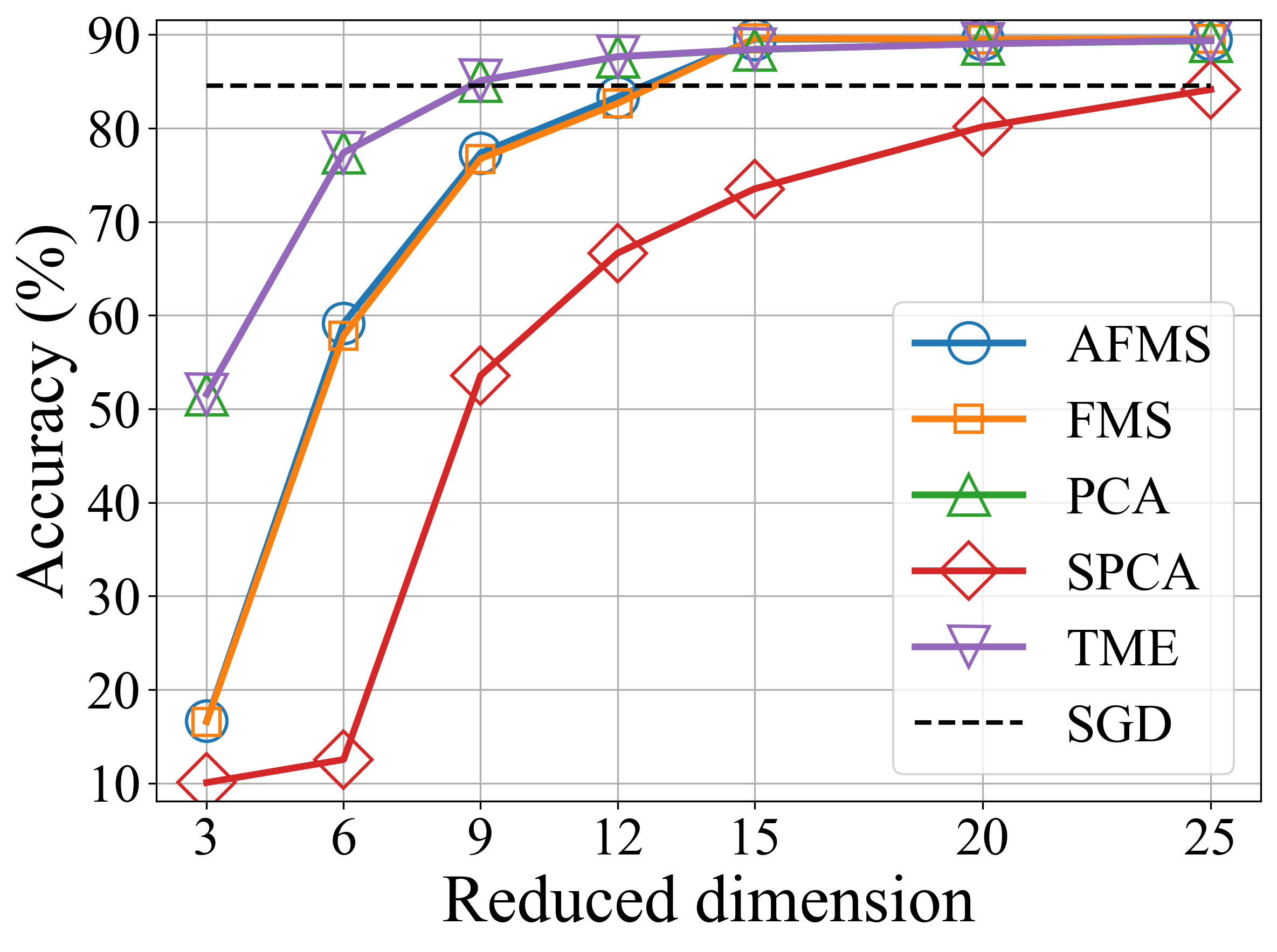}
    \label{fig:sub1}
  \end{minipage}\hfill
  \begin{minipage}[b]{0.31\linewidth}
    \centering
    CIFAR-100\\ \vspace{.1cm}
    \refstepcounter{subfig} 
    \includegraphics[width=\linewidth]{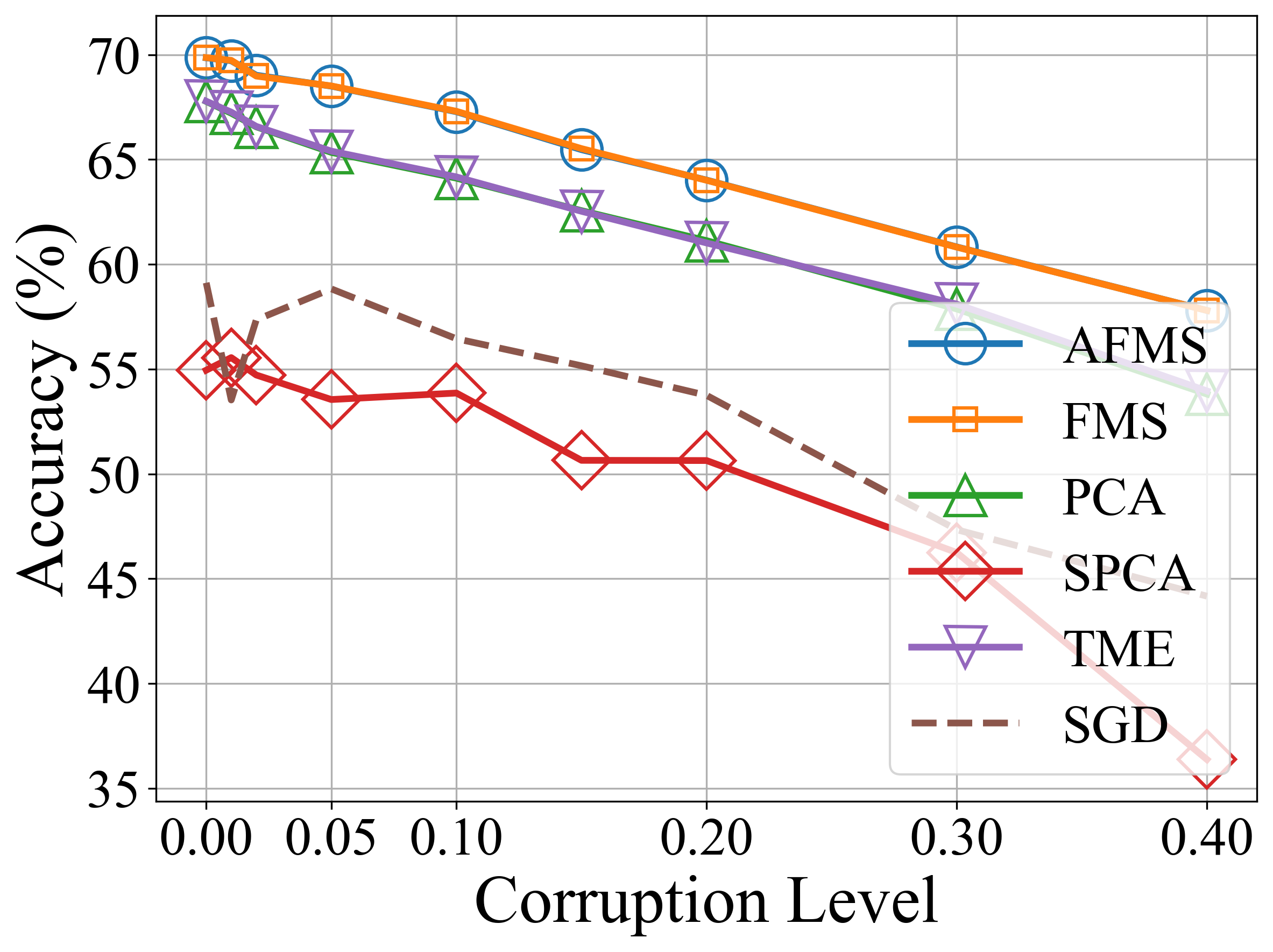}
    \includegraphics[width=\linewidth]{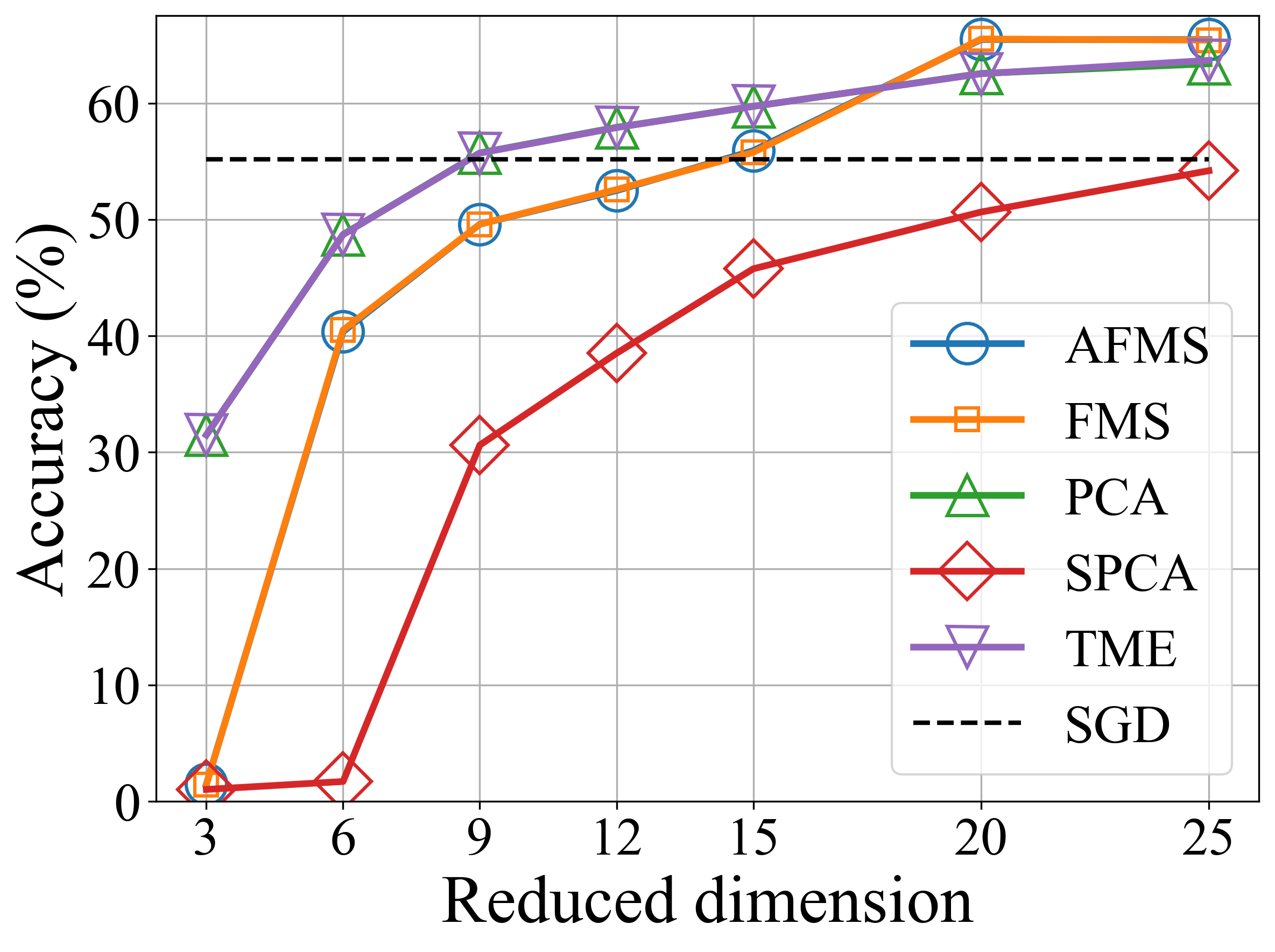}
    \label{fig:sub2}
  \end{minipage}\hfill
  \begin{minipage}[b]{0.31\linewidth}
    \centering
    Tiny ImageNet\\ \vspace{.1cm}
    \refstepcounter{subfig} 
    \includegraphics[width=\linewidth]{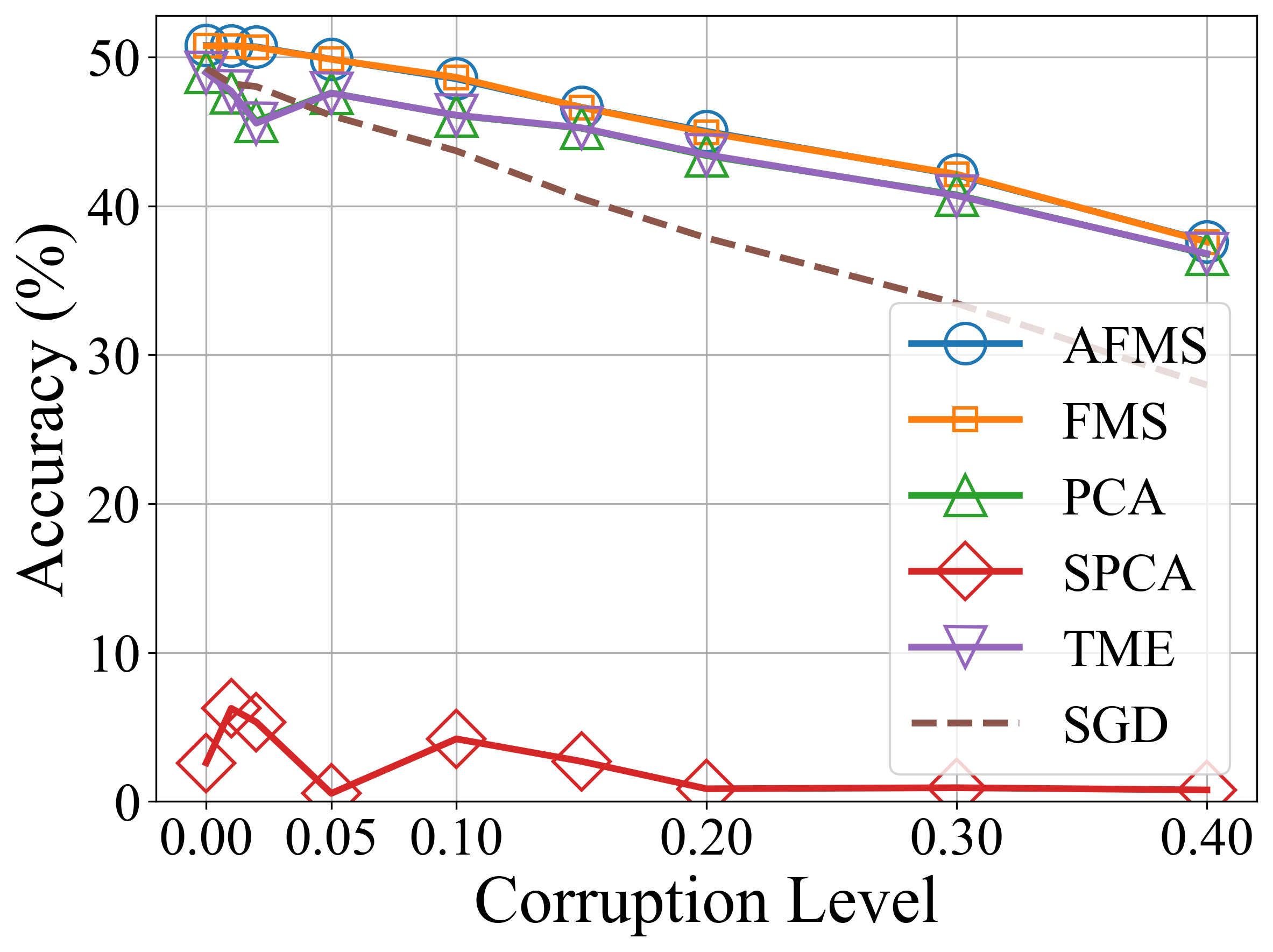}
    \includegraphics[width=\linewidth]{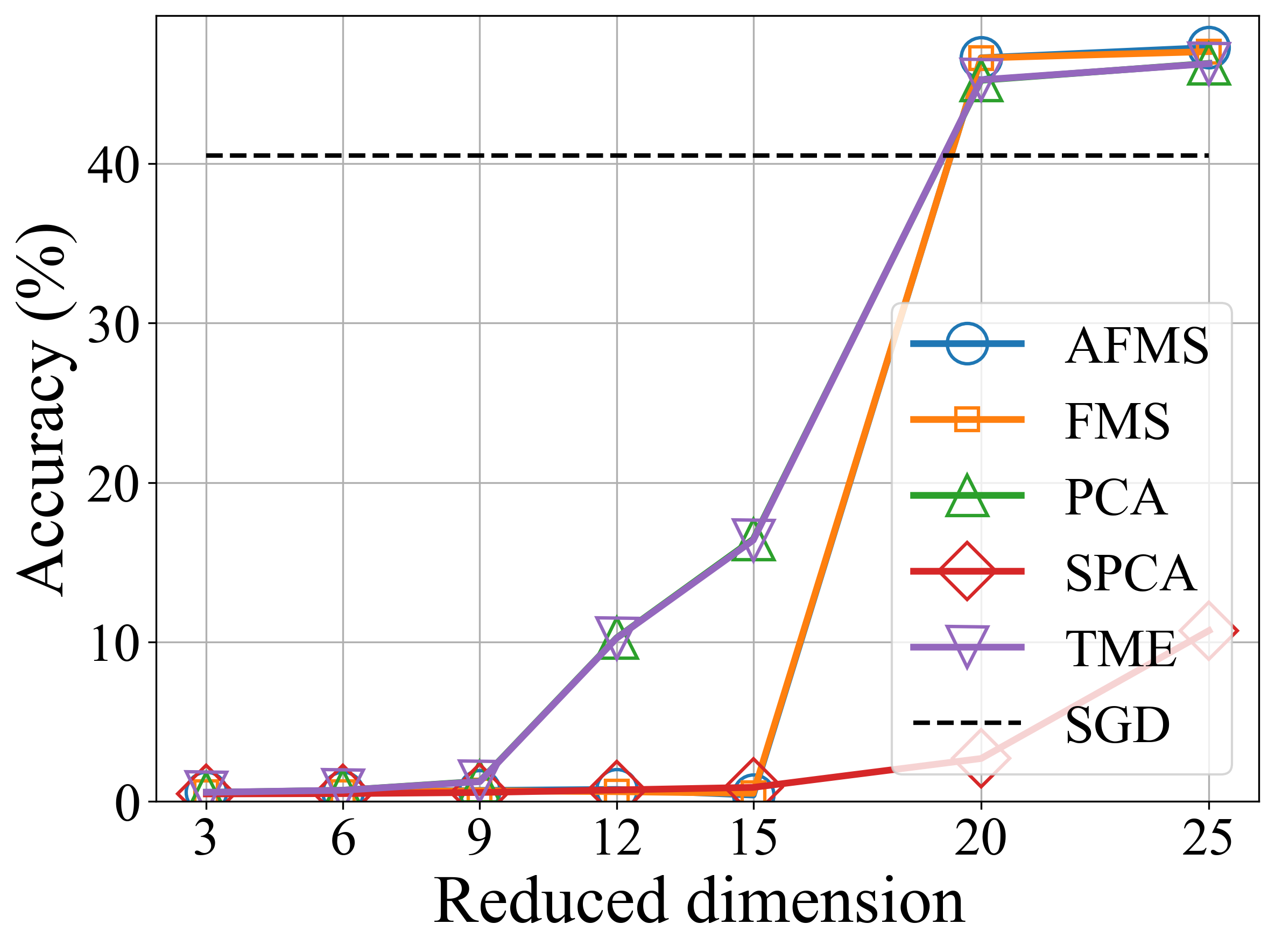}
    \label{fig:sub3}
  \end{minipage}
  \caption{Comparison of PCA, SPCA, TME, FMS, and AFMS for low-dimensional neural network training using three different datasets (left: CIFAR-10, center: CIFAR-100, right: Tiny ImageNet). Top row: subspace dimension fixed at 20 and varying corruption, bottom row: corruption level fixed at 15\% and varying subspace dimension. For sufficiently high dimensions, the subspace recovered by (A)FMS outperforms those obtained by other RSR methods.}
  \label{fig:low3}
\end{figure}

\section{Additional Proofs}
\label{app:proofs}

In this section, we provide additional proofs that were deferred for readability. Section~\ref{app:linthm_lemmaproofs} contains the proofs of Lemmas~\ref{lemma:smooth_properties}–\ref{lemma:assumption1}, which were stated during the proof of Theorem~\ref{thm:global}. Section~\ref{sec:proof_aff} presents the proof of Theorem~\ref{thm:affine}, which builds upon the ideas used in the proof of Theorem~\ref{thm:global}.
Lastly, Section \ref{sec:haystack} provides the proof of Proposition \ref{prop:haystack_assump}, Section \ref{subsec:prop_adv_proof} the proof of Proposition \ref{prop:adv_assump}, and Section \ref{sec:affine_models} provides the proofs of Propositions \ref{prop:affine_adversarial} and \ref{prop:affine_haystack}.

\subsection{Proof of Supplementary Lemmas for Theorem \ref{thm:global}}
\label{app:linthm_lemmaproofs}

\begin{proof}[Proof of Lemma \ref{lemma:smooth_properties}]
    Assume that $\epsilon_{2} < \epsilon_1$. We remind ourselves that
    \begin{align*}
        F_{\epsilon}(L)&=\sum_{\bx\in\calX:\dist(\bx,L)>\epsilon}\dist(\bx,L)+ \sum_{\bx\in\calX:\dist(\bx,L)\leq \epsilon}\left(\frac{1}{2}\epsilon+\frac{\dist(\bx,L)^2}{2\epsilon}\right).
    \end{align*}
    We proceed by cases for each $\bx$.

    If $\dist(\bx, L) > \epsilon_1$, then the corresponding term in $F_{\epsilon_{2}}$ is equal to the term in $F_{\epsilon_1}$. Similarly, if $\dist(\bx, L) \leq \epsilon_{2}$, the term is no larger in $F_{\epsilon_2}$ than in $F_{\epsilon_1}$, since $\epsilon/2+\dist(\bx,L)^2/(2\epsilon)$ is nondecreasing for $\epsilon\geq\dist(\bx,L)$.  It remains to study the case where $\epsilon_{2} < \dist(\bx, L) \leq \epsilon_1$. The result follows from the fact that
    \[
        |z| \leq \frac{\epsilon}{2} + \frac{z^2}{2\epsilon}
    \]
    for all $\epsilon > 0$. In particular, WLOG assume $z > 0$, and the inequality is equivalent to
    \[
        z^2 + \epsilon^2 - 2z\epsilon = (z-\epsilon)^2 \geq 0.
    \]
\end{proof}

\begin{proof}[Proof of Lemma~\ref{lemma:decrease}]
Before proving Lemma \ref{lemma:decrease}, we prove the following lemma on the quadratic growth of the trace function $f: \scrG(D,d)\rightarrow\reals$ defined by {$f(L)=\tr(\bP_L \bSigma\bP_L )$}.
\begin{lemma}\label{lemma:quadratic}
For a positive semi-definite matrix $\bSigma \in \reals^{D\times D}$ and  $\hat{L}=\argmax_{L}f(L)$, 
\[
f(\hat{L})-f(L)\geq \frac{\|\bP_L\bSigma\bP_{L^\perp}\|_F^2}{\|\bSigma\|}.
\]
\end{lemma}

\begin{proof}[Proof of Lemma~\ref{lemma:quadratic}]
By approximation with $\bSigma+\tau\bI$, $\tau>0$, it suffices to prove the result for positive definite $\bSigma$.
In the proof, we denote
\begin{equation*}
[\bSigma]_{L_1,L_2}=\bU_{L_1}^\top\bSigma \bU_{L_2}. 
\end{equation*}
and thus write
\begin{equation}   
\label{eq:matrix_subscript}
\bSigma=\begin{pmatrix}
[\bSigma]_{L,L} & [\bSigma]_{L,L^\perp}\\
[\bSigma]_{L^\perp,L} & [\bSigma]_{L^\perp,L^\perp}
\end{pmatrix}.
\end{equation}
Inspired by \cite{lerman_zhang2024}, we define \[\bSigma_*:=\begin{pmatrix}
[\bSigma]_{L,L} & [\bSigma]_{L,L^\perp}\\
[\bSigma]_{L^\perp,L} & [\bSigma]_{L^\perp,L} ([\bSigma]_{L,L})^{-1}[\bSigma]_{L,L^\perp}
\end{pmatrix}.
\] 
We note that $\bSigma-\bSigma_*$ is positive semi-definite. Indeed, three of its blocks are zero and the bottom-right block is $[\bSigma]_{L^\perp,L^\perp}-[\bSigma]_{L^\perp,L} ([\bSigma]_{L,L})^{-1}[\bSigma]_{L,L^\perp}$, which is positive semi-definite \cite[Theorem 1.3.3]{bhatia2009positive}. Since $f(\hat{L})$ is the sum of the top $d$ eigenvalues of $\bSigma$,
\begin{align}
f(\hat{L})=\sum_{i=1}^d\lambda_i(\bSigma)\geq \sum_{i=1}^d\lambda_i(\bSigma_*)=\tr(\bSigma_*)=\tr([\bSigma]_{L,L} )+\tr([\bSigma]_{L^\perp,L} ([\bSigma]_{L,L})^{-1}[\bSigma]_{L,L^\perp}),
\end{align} 
where the first inequality follows from the fact that $\bSigma-\bSigma_*$ is positive semidefinite and the second equality follows from the fact that $\rank(\bSigma_*)\leq d$, as
\[
\bSigma_*=\begin{pmatrix}
[\bSigma]_{L,L} \\
[\bSigma]_{L^\perp,L} 
\end{pmatrix} \begin{pmatrix}
\bI &
([\bSigma]_{L,L})^{-1}[\bSigma]_{L,L^\perp}.\end{pmatrix}
\]

In addition, 
\begin{align*}
    \tr([\bSigma]_{L^\perp,L} ([\bSigma]_{L,L})^{-1}[\bSigma]_{L,L^\perp})&=\|[\bSigma]_{L^\perp,L} ([\bSigma]_{L,L})^{-1/2}\|_F^2\\
    &\geq \|[\bSigma]_{L^\perp,L}\|_F^2/\|[\bSigma]_{L,L}\|\\
    &\geq  \|[\bSigma]_{L^\perp,L}\|_F^2/\|\bSigma\|.
\end{align*}
Recalling that  $f(L)=\tr([\bSigma]_{L,L} )$, the lemma is proved.
\end{proof}

We now proceed with the proof of Lemma \ref{lemma:decrease}. We use an auxiliary function defined as follows:
\begin{multline}
\label{eq:Gepsilon}
G_{\epsilon}(L,L_0)=\sum_{\bx_i\in\calX:\dist(\bx_i,L_0)>\epsilon}\Big(\frac{1}{2}\dist(\bx_i,L_0)+\frac{\dist(\bx_i,L)^2}{2\dist(\bx_i,L_0)}\Big)\\
+\sum_{\bx_i\in\calX:\dist(\bx_i,L_0)\leq \epsilon}\Big(\frac{1}{2}\epsilon+\frac{\dist(\bx_i,L)^2}{2\epsilon}\Big).
\end{multline}
Our proof is then based on Lemmas~\ref{lemma:smooth_properties} and \ref{lemma:quadratic} as well as the following observations (see Section 5.1 of \cite{lerman2018fast} for a proof):
\begin{align}\label{eq:MM_properties} 
&F_{\epsilon}(L)=G_{\epsilon}(L,L)\\\nonumber
&\text{$F_{\epsilon}(L)\leq G_{\epsilon}(L,L_0)$ with equality when $L=L_0$}\\\nonumber
&L^{(k+1)}=\argmin_L G_{\epsilon_k}(L,L^{(k)})
\end{align}

From the observations in \eqref{eq:MM_properties} we have
\[
F_{\epsilon_{k+1}}(L^{(k+1)})\leq F_{\epsilon_k}(L^{(k+1)})\leq  G_{\epsilon_k}(L^{(k+1)},L^{(k)})\leq G_{\epsilon_k}(L^{(k)},L^{(k)})=F_{\epsilon_k}(L^{(k)}).
\]
This implies that
\begin{equation}\label{eq:decrease0}
F_{\epsilon_k}(L^{(k)})-F_{\epsilon_k}(L^{(k+1)})\geq G_{\epsilon_k}(L^{(k)},L^{(k)})-G_{\epsilon_k}(L^{(k+1)},L^{(k)}).
\end{equation}
Now, using the fact that
\[
    \dist(\bx_i, L)^2 = \tr(\bP_{L^\perp} \bx_i \bx_i^\top \bP_{L^\perp}),
\]
we can rewrite
\begin{align*}
    G_{\epsilon_k}(L,L_0) &= \sum_{\bx\in\calX:\dist(\bx,L_0)>\epsilon_k}\left(\frac{1}{2}\dist(\bx_i,L_0)+\frac{\dist(\bx_i,L)^2}{2\dist(\bx_i,L_0)}\right) \ + \\
    &\sum_{\bx\in\calX:\dist(\bx,L_0)\leq \epsilon_k}\left(\frac{1}{2}\epsilon_k+\frac{\dist(\bx_i,L)^2}{2\epsilon_k}\right) \\
    &= \tr \left(\sum_i \frac{\bP_{L^\perp} \bx_i \bx_i^\top \bP_{L^\perp}}{2 \max(\dist(\bx_i, L_0), \epsilon_k)} \right) + C(L_0) \\
    &= \frac{1}{2} \tr(\bP_{L^\perp}\bS_{L_0,\epsilon_k}\bP_{L^\perp}) + C(L_0) \\
    &= \frac{1}{2}\Big[\tr(\bS_{L_0,\epsilon_k} )-\tr(\bP_{L}\bS_{L_0,\epsilon_k}\bP_{L})\Big] + C(L_0)
\end{align*}{
for some term $C(L_0)$ that depends on $L_0$, $\epsilon_k$, and $\calX$. %
So we can apply Lemma \ref{lemma:quadratic} to find
\begin{align*}
&     G_{\epsilon_k}(L^{(k)}, L^{(k)}) - G_{\epsilon_k}(L^{(k+1)},L^{(k)})  
     \\=& \frac{1}{2}\left(-\tr(\bP_{L^{(k)}}\bS_{L^{(k)},\epsilon_k}\bP_{L^{(k)}}))+\max_{L} \tr(\bP_{L}\bS_{L^{(k)},\epsilon_k}\bP_{L}))\right)
     \geq \frac{1}{2} \frac{\|\bP_{L^{(k)}} \bS_{L^{(k)},\epsilon_k} \bP_{L^{(k)^\perp}}\|_F^2}{\|\bS_{L^{(k)},\epsilon_k}\|}.
\end{align*}
}
Lemma \ref{lemma:decrease} follows by the previous display and \eqref{eq:MM_properties}.
\end{proof}

 \begin{proof}[Proof of Lemma~\ref{lemma:gradient}]
 In this proof, we let $\bS_{L,\epsilon}=\bS_{L,\epsilon,\mathrm{in}}+\bS_{L,\epsilon,\mathrm{out}}$, where
\[
\bS_{L,\epsilon,\mathrm{in}}=\sum_{\bx\in\calX_{\mathrm{in}}}\frac{\bx\bx^\top }{\max(\dist(\bx,L),\epsilon)},\,\,\bS_{L,\epsilon,\mathrm{out}}=\sum_{\bx\in\calX_{\mathrm{out}}}\frac{\bx\bx^\top}{\max(\dist(\bx,L),\epsilon)},
\]
and we will investigate $\bP_L\bS_{L,\epsilon,\mathrm{in}}\bP_{L^\perp}$ and $\bP_L\bS_{L,\epsilon,\mathrm{out}}\bP_{L^\perp}$ separately.

Let the principal vectors between $L$ and $L_\star$ be $\{\bu_j\}_{j=1}^d$ and $\{\bv_j\}_{j=1}^d$, corresponding to $L$ and $L_\star$, respectively, and let $\{\bw_j\}_{j=1}^d$ be unit vectors for positive angles, and zero vectors for zero angles, in $L^\perp$ such that $\bv_j = \cos\theta_j\, \bu_j + \sin\theta_j\, \bw_j$. The bound is immediate when $\theta_1=\pi/2$; below assume $\theta_1<\pi/2$. Then we will show that there exists $1\leq j\leq d$ such that 
\[
\bu_j^\top\bS_{L,\epsilon,\mathrm{in}}\bw_j\geq \frac{\cos\theta_1}{2\sqrt{d}} \Sin,
\]
{and $|\bu_j^\top\bS_{L,\epsilon,\mathrm{out}}\bw_j|\leq \Sout$ for all $1 \leq j \leq d$.}

We begin by analyzing $\bS_{L,\epsilon,\mathrm{out}}$. By the definition of $\Sout$,  
\[
\|\bP_L \bS_{L,\epsilon,\mathrm{out}} \bP_{L^\perp}\|_2\leq\Sout
\]
and as a result, for all $1\leq j\leq d$,
\begin{equation}\label{eq:Sout_lb}
|\bu_j^\top\bS_{L,\epsilon,\mathrm{out}}\bw_j|\leq\Sout.
\end{equation}

We now analyze $\bP_L \bS_{L,\epsilon,\mathrm{in}} \bP_{L^\perp}$. We first  investigate $\bu_j^\top\bx\bx^\top\bw_j$ for any $\bx\in\calX_{\mathrm{in}}$:
\begin{align}
\label{eq:gradient_in1}
\bu_j^\top\bx=\bu_j^\top \left( \sum_{i=1}^d \bv_i \bv_i^\top \right)\bx = \cos\theta_j\bv_j^\top\bx, \quad 
\bw_j^\top\bx=\bw_j^\top \left(\sum_{i=1}^d \bv_i \bv_i^\top \right) \bx = \sin \theta_j \bv_j^\top \bx.
\end{align} 
{The first equalities follow from the fact that $\bx \in L_\star$ and the second from $\bu_j,\bw_j \perp \bv_i$ for $i \neq j$.}
Next, we note that  $P_{L^\perp}\bx=\sum_{j=1}^d\bw_j\bw_j^\top\bx$, and thus
\begin{equation}\label{eq:gradient_in2}
\dist(\bx,L)=\|P_{L^\perp}\bx\|=\sqrt{\sum_{j=1}^d(\bw_j^\top\bx)^2}=\sqrt{\sum_{j=1}^d\sin^2\theta_j(\bv_j^\top\bx)^2}.
\end{equation}
To proceed with our lower bound of $\dist(\bx,L)$, we note that  $\sqrt{\sum_{j=1}^d \sin^2 \theta_j c_j}$ is a concave function of the variables $\sin^2 \theta_j$. Consequently, using the result of Section 7.2.3 of \cite{lerman2015robust}
\begin{align}\label{eq:dist_lb}
\sum_{\bx\in\calX_{\mathrm{in}}}\dist(\bx,L)&=\sum_{\bx\in\calX_{\mathrm{in}}}\sqrt{\sum_{j=1}^d \sin^2 \theta_j \cdot (\bv_j^\top\bx)^2}  \geq \Sin\sqrt{\sum_{j=1}^d \sin^2\theta_j}.
\end{align}

Applying \eqref{eq:gradient_in1} and \eqref{eq:gradient_in2}, for an inlier $\bx$, we have that 
\begin{align*}
\frac{\dist^2(\bx,L)}{\max(\dist(\bx,L),\epsilon)}&=\frac{\sum_{j=1}^d  \sin^2\theta_j\cdot(\bv_j^\top\bx)^2}{\max(\dist(\bx,L),\epsilon)}\\ &= \sum_{j=1}^d \tan\theta_j \frac{\sin\theta_j\cos\theta_j(\bv_j^\top\bx)^2}{\max(\dist(\bx,L),\epsilon)} \\ &=\sum_{j=1}^d \tan(\theta_j) \, \frac{\bu_j^\top\bx\bx^\top\bw_j}{\max(\dist(\bx,L),\epsilon)}.
\end{align*}
Next, let us analyze
\(
\sum_{j=1}^d \tan \theta_j \, \bu_j^\top \bS_{L,\epsilon,\mathrm{in}}\bw_j .\)
Recall that, {by assumption}, $\epsilon$ is chosen such that for at least half of the inliers, $\dist(\bx,L)\geq \epsilon$. Using this and \eqref{eq:dist_lb},
\begin{align}\label{eq:gradient_in3}
\sum_{j=1}^d \tan \theta_j \bu_j{^\top}\bS_{L,\epsilon,\mathrm{in}}\bw_j &=\sum_{\bx\in\calX_{\mathrm{in}}} \frac{\dist^2(\bx,L)}{\max(\dist(\bx,L),\epsilon)} \\ \nonumber
&\geq \sum_{\bx\in\calX_{\mathrm{in}}:\dist(\bx,L)\geq\epsilon} \dist(\bx,L) \\ \nonumber
&\geq 
 \frac{1}{2}\sum_{\bx\in\calX_{\mathrm{in}}} \dist(\bx,L) \\ \nonumber
 &\geq \frac{1}{2}\Sin \sqrt{\sum_{j=1}^d\sin^2\theta_j}.\nonumber
\end{align}

On the other hand, {because the principal angles are ordered such that $\theta_1\geq\cdots\geq \theta_d$,}
\begin{equation}\label{eq:gradient_in4}
\frac{\sqrt{\sum_{j=1}^d\sin^2\theta_j}}{\sum_{j=1}^d \tan\theta_j }\geq \frac{\frac{1}{\sqrt{d}}{\sum_{j=1}^d\sin\theta_j}}{\frac{1}{\cos\theta_1}\sum_{j=1}^d \sin\theta_j } =\frac{\cos\theta_1}{\sqrt{d}}.
\end{equation}
As a result, \eqref{eq:gradient_in3} implies that
\[
\max_{j=1,\cdots, d}\bu_j^\top \bS_{L,\epsilon,\mathrm{in}}\bw_j\geq \frac{\Sin\sqrt{\sum_{j=1}^d\sin^2\theta_j}}{2\sum_{j=1}^d \tan\theta_j}\geq \frac{\cos\theta_1}{2\sqrt{d}} \Sin.
\]

Denoting the index that achieves the maximum above by $j_0$ and using \eqref{eq:Sout_lb} and the above equation, the proof of Lemma~\ref{lemma:gradient} is concluded:
\[
\|\bP_{L}\bS_{L,\epsilon}\bP_{L^\perp}\|_F\geq \bu_{j_0}^\top\bS_{L,\epsilon}\bw_{j_0}\geq \bu_{j_0}^\top\bS_{L,\epsilon,\mathrm{in}}\bw_{j_0}+\bu_{j_0}^\top\bS_{L,\epsilon,\mathrm{out}}\bw_{j_0}\geq \frac{\cos\theta_1}{2\sqrt{d}} \Sin-\Sout.
\]

\end{proof}

{
\begin{proof}[Proof of Lemma~\ref{lem:assumpholds2}] 
We will prove Lemma~\ref{lem:assumpholds2} by showing that for all $k \geq 0$, $\theta_1(L^{(k+1)}, L_\star) \leq \theta_0$. 

We begin by the naive bound (recall that $\alpha_0\leq \gamma/\gamma_\star\leq 1/2$) \begin{equation}q_{\alpha_0}(\{\dist(\bx, L)\}_{\bx \in \calX_{\mathrm{in}}})\leq \mathrm{median}(\{\dist(\bx, L)\}_{\bx \in \calX_{\mathrm{in}}})\leq 2\mathrm{mean}(\{\dist(\bx, L)\}_{\bx \in \calX_{\mathrm{in}}}).\label{eq:assumpholds2_proof1}\end{equation} 
The second bound follows from the fact that for positive numbers, the median is bounded by twice the mean. For any subspaces $L, L' \in \scrG(D, d)$ such that $F_{\mathrm{in}}(L') \geq F_{\mathrm{in}}(L)/5$, that is, $\mathrm{mean}(\{\dist(\bx, L')\}_{\bx \in \calX_{\mathrm{in}}}) \geq \mathrm{mean}(\{\dist(\bx, L)\}_{\bx \in \calX_{\mathrm{in}}})/5$,  \eqref{eq:global21} and \eqref{eq:assumpholds2_proof1} imply
\[
10\beta_1 q_{\alpha_0}\left(\left\{\dist(\bx, L')\right\}_{\bx \in \calX_{\mathrm{in}}}\right)
\geq  q_{\alpha_0}\left(\left\{\dist(\bx, L)\right\}_{\bx \in \calX_{\mathrm{in}}}\right).
\]
The assumption \eqref{eq:global22} implies that
\begin{align*}
    q_{\alpha_0}\left(\left\{\dist(\bx, L')\right\}_{\bx \in \calX_{\mathrm{in}}}\right) & \leq q_{\alpha_0}\left(\left\{\|\bx\|\right\}_{\bx \in \calX_{\mathrm{in}}}\right) \\
    & \leq \beta_2 \cdot q_{\frac{\gamma-\alpha_0\gamma_\star}{1 - \gamma_\star}}\left(\{\dist(\bx, L')\}_{\bx \in \calX_{\mathrm{out}}}\right).
\end{align*}
Using this we find that
\begin{align*}
q_\gamma\left(\left\{\dist(\bx, L')\right\}_{\bx \in \calX}\right)
&\geq \min\left(
q_{\alpha_0}\left(\left\{\dist(\bx, L')\right\}_{\bx \in \calX_{\mathrm{in}}}\right),
q_{\frac{\gamma-\alpha_0\gamma_\star}{(1 - \gamma_\star)}}\left(\left\{\dist(\bx, L')\right\}_{\bx \in \calX_{\mathrm{out}}}\right)
\right) \\
&\geq \min\left(
q_{\alpha_0}\left(\left\{\dist(\bx, L')\right\}_{\bx \in \calX_{\mathrm{in}}}\right),
\frac{1}{\beta_2} q_{\alpha_0}\left(\left\{\dist(\bx, L')\right\}_{\bx \in \calX_{\mathrm{in}}}\right)\right) \\
&\geq \frac{1}{\beta_2} q_{\alpha_0}\left(\left\{\dist(\bx, L')\right\}_{\bx \in \calX_{\mathrm{in}}}\right).
\end{align*}
We remark that when $\gamma>1-(1-\alpha_0)\gamma_\star$, then $(\gamma-\alpha_0\gamma_\star)/(1-\gamma_\star)>1$, and hence, by setting $\beta_2=1$, the above condition is still satisfied: fewer than $\gamma n$ observations lie strictly below the inlier quantile. If $\nout=0$, use $\alpha_0<\gamma$ and monotonicity of quantiles.  
In summary, we obtain the bound
\begin{equation}\label{eq:global24}
q_{\alpha_0}\left(\left\{\dist(\bx, L)\right\}_{\bx \in \calX_{\mathrm{in}}}\right)\leq {10}{\beta_1}\cdot q_{\alpha_0}\left(\left\{\dist(\bx, L')\right\}_{\bx \in \calX_{\mathrm{in}}}\right) 
\leq \beta \cdot q_\gamma\left(\left\{\dist(\bx, L')\right\}_{\bx \in \calX}\right).
\end{equation}

 Next, we show that at least $\alpha_0$ percentage of the inliers satisfy $\dist(\bx, L^{(k)}) \leq \beta \epsilon_k$, i.e., 
\begin{equation}\label{eq:inliers_epsilonk}
\calX_{\mathrm{in}}' = \left\{\bx \in \calX_{\mathrm{in}} : \dist(\bx, L^{(k)}) \leq \beta \epsilon_k \right\} \implies |\calX_{\mathrm{in}}'| \geq \alpha_0|\calX_{\mathrm{in}}|.
\end{equation}
By the definition of $\epsilon_k$ in \eqref{eq:epsilon_choose}, there exists $l \leq k$ such that
$\epsilon_k = q_\gamma\left(\left\{\dist(\bx, L^{(l)})\right\}_{\bx \in \calX}\right).$ We now consider two cases:

\emph{Case 1: $k = l$.} In this case, \eqref{eq:inliers_epsilonk} follows directly from \eqref{eq:global24} by taking $L = L' = L^{(k)}$.

\emph{Case 2: $k>l$.} Since $\epsilon_l=\epsilon_k$ and the smoothing parameters are nonincreasing, $\epsilon_l=\epsilon_{l+1}=\cdots=\epsilon_k$. Monotonicity of the smoothed objective gives
\[
F(L^{(k)})\leq F_{\epsilon_k}(L^{(k)})
\leq F_{\epsilon_l}(L^{(l)}).
\]
By~\eqref{eq:natural_bound} and the quantile convention,
\begin{align*}
F_{\epsilon_l}(L^{(l)})-F(L^{(l)})
&\leq\frac{\epsilon_l}{2}
|\{\bx\in\calX:\dist(\bx,L^{(l)})<\epsilon_l\}|\\
&\leq\frac{\epsilon_l\gamma|\calX|}{2}
\leq\frac{\epsilon_l|\calX_{\mathrm{in}}|}{4}
\leq\frac12 F_{\mathrm{in}}(L^{(l)}).
\end{align*}
The last inequality follows because at least half of the inliers have distance at least $\epsilon_l$ from $L^{(l)}$. Set $\rho=\pi\sqrt d\,\Sout/(2\Sin)$. Lemma~\ref{lemma:objectivevalue_bound} and Assumption~\ref{assump:global} imply
\[
|F_{\mathrm{out}}(L)-F_{\mathrm{out}}(L_\star)|
\leq\rho F_{\mathrm{in}}(L),\qquad 0\leq\rho\leq\pi/6.
\]
Combining these bounds yields
\[
(1-\rho)F_{\mathrm{in}}(L^{(k)})
\leq\left(\tfrac32+\rho\right)F_{\mathrm{in}}(L^{(l)}).
\]
Since $(3/2+\rho)/(1-\rho)<5$ for $0\leq\rho\leq\pi/6$, we obtain
\[
F_{\mathrm{in}}(L^{(k)})\leq5F_{\mathrm{in}}(L^{(l)}).
\]
Applying~\eqref{eq:global24} with $L=L^{(k)}$ and $L'=L^{(l)}$ proves~\eqref{eq:inliers_epsilonk} in this case. This argument preserves the constant $\beta=10\beta_1\beta_2$.

Finally, observe that $L^{(k+1)}$ is the span of the top $d$ eigenvectors of the matrix
\[
\sum_{\bx \in \calX} \frac{\bx \bx^\top}{\max(\dist(\bx, L^{(k)}), \epsilon_k)} 
= \sum_{\bx \in \calX_{\mathrm{in}}} \frac{\bx \bx^\top}{\max(\dist(\bx, L^{(k)}), \epsilon_k)} 
+ \sum_{\bx \in \calX_{\mathrm{out}}} \frac{\bx \bx^\top}{\max(\dist(\bx, L^{(k)}), \epsilon_k)}.
\]
Note that the outlier contribution can be bounded as
\[
\sum_{\bx \in \calX_{\mathrm{out}}} \frac{\bx \bx^\top}{\max(\dist(\bx, L^{(k)}), \epsilon_k)} 
\preceq \sum_{\bx \in \calX_{\mathrm{out}}} \frac{\bx \bx^\top}{\epsilon_k},
\]
and the inlier contribution satisfies
\[
\sum_{\bx \in \calX_{\mathrm{in}}} \frac{\bx \bx^\top}{\max(\dist(\bx, L^{(k)}), \epsilon_k)} 
\succeq \sum_{\substack{\bx \in \calX_{\mathrm{in}} \\ \dist(\bx, L^{(k)}) \leq \beta \epsilon_k}} \frac{\bx \bx^\top}{\max(\dist(\bx, L^{(k)}), \epsilon_k)} 
\succeq \sum_{\bx \in \calX_{\mathrm{in}}'} \frac{\bx \bx^\top}{\beta \epsilon_k},
\]
where $\calX_{\mathrm{in}}'$ is the subset of inliers defined in \eqref{eq:inliers_epsilonk}.

Define
\[
\bA := \sum_{\bx \in \calX_{\mathrm{in}}} \frac{\bx \bx^\top}{\max(\dist(\bx, L^{(k)}), \epsilon_k)}, 
\quad 
\bB := \sum_{\bx \in \calX} \frac{\bx \bx^\top}{\max(\dist(\bx, L^{(k)}), \epsilon_k)}.
\]
Then by combining the inequalities above, along with \eqref{eq:global23} and \eqref{eq:inliers_epsilonk}, we obtain
\[
\sin \theta_0 \cdot \sigma_d(\bA) \geq 2 \cdot \sigma_1(\bB - \bA).
\]

The inliers span $L_\star$ by Assumption~\ref{assump:lowerdim}, so $\sigma_d(\bA)>0$, and $\sigma_{d+1}(\bB)\leq\|\bB-\bA\|<\sigma_d(\bA)/2$.
We now apply the Davis–Kahan sin–$\theta$ theorem \cite{a6ee9c48-1e5b-385a-93c1-cdf3b873de37}, using the formulation and notation from \citet[Theorem VII.3.1]{bhatia1996matrix}. Let $S_1 = [\sigma_d(\bA), \infty)$ and $S_2 = (-\infty, \sigma_d(\bA)/2]$. Then the spectral gap is $\delta = \sigma_d(\bA)/2$. Let $\bE$ denote the orthogonal projector onto $L_\star$ and $\bF$ the orthogonal projector onto the orthogonal complement of $L^{(k+1)}$. Then
\[
\sin(\theta_1(L^{(k+1)}, L_\star)) = \|\bE\bF\| 
\leq \frac{\sigma_1(\bB - \bA)}{\delta} 
= \frac{2 \cdot \sigma_1(\bB - \bA)}{\sigma_d(\bA)} 
\leq \sin \theta_0.
\]
This concludes the proof.

\end{proof}

}

\begin{proof}[Proof of Lemma~\ref{lemma:objectivevalue_bound}]
We again denote the principal angles between $L$ and $L_\star$ by $\{\theta_j\}_{j=1}^d$. We first claim that
\begin{align}\label{eq:objectivevalue_bound10}
F_{\mathrm{in}}(L)&\geq \frac{\sum_{j=1}^d\sin\theta_j}{\sqrt{d}} \Sin,\\\label{eq:objectivevalue_bound20}
|F_{\mathrm{out}}(L)-F_{\mathrm{out}}(L_\star)|&\leq \Sout\sum_{j=1}^d\theta_j.
\end{align} 
Then the lemma follows from \eqref{eq:objectivevalue_bound10} and \eqref{eq:objectivevalue_bound20} using the fact that $F_{\mathrm{in}}(L_\star) = 0$.

To prove \eqref{eq:objectivevalue_bound10}, we again let the principal vectors between $L$ and $L_\star$ be $\bu_1,\cdots, \bu_d$ and $\bv_1,\cdots, \bv_d$. Then any $\bx\in\calX_{\mathrm{in}}$ can be written as
\[
\bx=\sum_{j=1}^d x_j\bv_j, \text{where $x_j=\bv_j^\top\bx$},
\]
and the projection to $L$ can be written as
\[
P_L(\bx)=\sum_{j=1}^d x_j P_L(\bv_j). 
\]
Note that $\bv_j$ has an angle of $\theta_j$ with $L$ and $\|\bv_j-P_L(\bv_j)\|=\sin\theta_j$. Consequently, 
\[
\dist(\bx,L)=\|\bx - \bP_L \bx\|=\sqrt{ 
\sum_{j=1}^d\sin^2\theta_j x_j^2}=\sqrt{
\sum_{j=1}^d\sin^2\theta_j (\bv_j^\top\bx)^2}\geq \frac{1}{\sqrt{d}}\sum_{j=1}^d\sin\theta_j |\bv_j^\top\bx|.
\]
Summing the above inequality over all $\bx\in\calX_{\mathrm{in}}$ and applying the definition of $\Sin$, \eqref{eq:objectivevalue_bound10} is proved.

To prove \eqref{eq:objectivevalue_bound20}, let $\{\bu_1,\cdots, \bu_d, \bw_1\}\subset\reals^D$ be an orthonormal set, and let 
$L(t)=\Sp(\cos t\theta \bu_1+\sin t\theta \bw_1, \bu_2,\cdots,\bu_d)$.
We first notice that 
\begin{align}
\label{eq:1dderbd}
{\left|\frac{\di}{\di t}\sum_{\bx\in\calX_{\mathrm{out}}}\dist(\bx,L(t))\right|}
&={\theta\left|\sum_{\bx\in\calX_{\mathrm{out}}}\frac{(\cos t\theta\bu_1+\sin t\theta\bw_1)^\top\bx\bx^\top(-\sin t\theta\bu_1+\cos t\theta\bw_1)}{\dist(\bx,L(t))}\right|}\\ \nonumber
&\leq\theta\left\|\sum_{\bx\in\calX_{\mathrm{out}}}\frac{\bP_{L(t)}\bx\bx^\top\bP_{L(t)^\perp}}{\dist(\bx,L(t))}\right\|\leq\theta\Sout.
\end{align}

Now let $\bu_1, \bu_2, \cdots, \bu_d$ and $\bv_1, \bv_2, \cdots, \bv_d$ be the principal vectors between $L_\star$ and $L$ with principal angles $\theta_1, \cdots, \theta_d$, and define that for all $1\leq k\leq d$,
\begin{align*}
L_k=\Span(\bv_1,\cdots,\bv_k,\bu_{k+1},\cdots,\bu_d).
\end{align*}
In particular, $L_0=L_\star$ and $L_d=L$. Then $\dist(L_{k-1},L_k)=\theta_k$ and integrating $\frac{\di}{\di t}\dist(\bx,L(t))$ and applying \eqref{eq:1dderbd} implies that
\[
\left|\sum_{\bx\in\calX_{\mathrm{out}}}\big(\dist(\bx,L_k)-\dist(\bx,L_{k-1})\big)\right|\leq \theta_k \Sout.
\]
Combining the above inequality for all $1\leq k\leq d$, \eqref{eq:objectivevalue_bound20} is proved.

\end{proof}

{
\begin{proof}[Proof of Lemma~\ref{lemma:assumption1}]

We begin by noting that $\frac{F_{\mathrm{in}}(L)}{|\calX_{\mathrm{in}}|} =  \mathrm{mean}_{\bx\in\calX_{\mathrm{in}}}\dist(\bx,L)$. Combining this with the fact that $\dist(\bx, L) \leq \sin\theta_1(L,  L_\star) \|\bx\|$ yields
\begin{equation}\label{eq:meanbd1}
    \mathrm{mean}_{\bx\in\calX_{\mathrm{in}}}\dist(\bx,L)\leq \sin\theta_1(L,L_\star)\cdot\mathrm{mean}_{\bx\in\calX_{\mathrm{in}}}\|\bx\|.
\end{equation}
To prove the lemma, it is thus sufficient to prove that there is a $c_3 > 0$ such that
\begin{equation}\label{eq:percentage_gamma}
     \sup_{L\in \scrG(D,d)} \frac{|\{\bx \in \calX: \dist(\bx,L)<c_3\cdot\sin\theta_1(L,L_\star)\cdot\mathrm{mean}_{\bx\in\calX_{\mathrm{in}}}\|\bx\| \}|}{|\calX|}\leq\gamma.
\end{equation}
This follows from the definition of $q_\gamma$ and the fact that \eqref{eq:meanbd1} implies the following is a sufficient condition for the lemma:
\[
c_3\cdot \sin\theta_1(L,L_\star)\cdot\mathrm{mean}_{\bx\in\calX_{\mathrm{in}}}\|\bx\| \leq q_{\gamma}(\{\dist(\bx, L)\}_{\bx \in \calX}).
\]

To prove the existence of such a $c_3$, we will show that Assumption 1 implies that
\begin{equation*}
    \frac{q_{\gamma}(\{\dist(\bx,L)\}_{\bx\in\calX})}{\sin \theta_1(L, L_\star)}
\end{equation*}
is uniformly lower bounded by a positive constant for $L \neq L_\star$.

{We prove this by contradiction. Suppose the statement does not hold. Then there exists a sequence \(\{L^{(k)}\}_{k \geq 1}\) such that
\begin{equation}\label{eq:bdqgamma}
\lim_{k \to \infty} \frac{q_{\gamma}(\{\dist(\bx,L^{(k)})\}_{\bx \in \calX})}{\sin \theta_1(L^{(k)}, L_\star)} = 0.
\end{equation}

Since the Grassmannian is compact, we may assume without loss of generality that \(L^{(k)} \to \hat{L}\) (otherwise, we can choose a convergent subsequence of $L^{(k)}$).  

\textbf{Case 1:} \(\theta_1(\hat{L}, L_\star) \neq 0\).  
Then $\hat{L}\neq L_\star$ and \eqref{eq:bdqgamma} implies
$q_{\gamma}(\{\dist(\bx,\hat{L})\}_{\bx \in \calX}) = 0,$
which contradicts Assumption~\ref{assump:lowerdim}, which claims that fewer than \(\gamma\) fraction of points can lie in the \(d\)-dimensional subspace $\hat{L}$.

\textbf{Case 2:} \(\theta_1(\hat{L}, L_\star) = 0\), i.e., \(\hat{L} = L_\star\).  
Without loss of generality, we may assume the principal vector \(\bv_1(L^{(k)}, L_\star)\) converges to some \(\hat{\bv}_1 \in L_\star \) (otherwise, choose a subsequence with this property). Then, for any \(\bx \in L_\star\),
\begin{align*}
&\liminf_{k \to \infty} \frac{\dist(\bx,L^{(k)})^2}{\sin^2 \theta_1(L^{(k)}, L_\star)}
= \liminf_{k \to \infty} \left(|\bx^\top \bv_1(L^{(k)},L_\star)|^2 + \sum_{j=2}^d \frac{\sin^2 \theta_j(L^{(k)},L_\star)}{\sin^2 \theta_1(L^{(k)}, L_\star)} |\bx^\top \bv_j(L^{(k)},L_\star)|^2\right)\\\geq& \liminf_{k \to \infty} |\bx^\top \bv_1(L^{(k)},L_\star)|^2= |\bx^\top \hat{\bv}_1|^2.
\end{align*}

On the other hand, for \(\bx \notin L_\star\),  \(\dist(\bx,L^{(k)}) / \sin \theta_1(L^{(k)}, L_\star) \to \infty\) as \(k \to \infty\).  

Thus, \eqref{eq:bdqgamma} can occur only if at least \(\gamma |\calX|\) points satisfy \(\bx^\top \hat{\bv}_1 = 0\) and $\bx\in L_\star$, i.e., \(\gamma |\calX|\) points lie in the \((d-1)\)-dimensional subspace \(L_\star \cap \Sp(\hat{\bv}_1)^\perp\), which again contradicts Assumption~\ref{assump:lowerdim}.

In summary, the existence of $c_3$ follows by contradiction through the analysis of two cases.
}

First, note that as $L \to L_\star$, 
\[
    q_{\gamma}(\{\dist(\bx,L)\}_{\bx\in\calX}) = q_{\gamma |\calX|/|\calX\cap L_\star|}(\{\dist(\bx,L)\}_{\bx\in\calX\cap L_\star})
\]
so that in a neighborhood around $L_\star$, we only need to understand the behavior of \eqref{eq:bdqgamma} with respect to the inlier points.

 Next, observe that for $\bx \in L_\star$,  $\frac{\dist(\bx,L)^2}{\sin^2 \theta_1(L, L_\star)}=|\bx^\top\bv_1(L,L_\star)|^2 + \sum_{j=2}^d \frac{\sin^2\theta_j(L,L_\star)}{\sin^2\theta_1(L, L_\star)} |\bx^\top\bv_j(L,L_\star)|^2$, where $\bv_j(L, L_\star)$ is a principal vector for  $L_\star$ relative to $L$.
 Consequently, 
\begin{align*}
 \{\bx\in \calX_{\mathrm{in}} : &\dist(\bx,L)<c_3\cdot\sin\theta_1(L,L_\star)\cdot\mathrm{mean}_{\bx\in\calX_{\mathrm{in}}}\|\bx\|\} \\
 &\subseteq \{\bx\in \calX_{\mathrm{in}} : |\bx^\top\bv_1(L,L_\star)|<c_3\cdot\mathrm{mean}_{\bx\in\calX_{\mathrm{in}}}\|\bx\|\}.
 \end{align*}
By Assumption 1, we thus know that there is a $c_3$ such that ${q_{\gamma}(\{\dist(\bx,L)\}_{\bx\in\calX})}/{\sin \theta_1(L, L_\star)} > c_3$ for all $L \in B(L_\star, \delta)$.
 
Define $L_0 = \Sp(\bv_2, \dots, \bv_d)$. Then, by Assumption 1, we know that less than a $\gamma$ fraction of points lie in $L_0 \cup L$. In particular, for $L \neq L_\star$,
\[
    \frac{q_{\gamma}(\{\dist(\bx,L)\}_{\bx\in\calX})}{\sin \theta_1(L, L_\star)} \geq \min\left(  \left\{ |\bv_1^\top \bx | : \bx \in \calX_{\mathrm{in}} \setminus L_0\right\} \cup \left\{ \frac{\dist(\bx, L)}{\sin \theta_1(L, L_\star)}: \bx \in \calX_{\mathrm{out}} \setminus L \right\}\right)
\]

Since $\frac{q_{\gamma}(\{\dist(\bx,L)\}_{\bx\in\calX})}{\sin \theta_1(L, L_\star)} \not \to 0$ as $L \to L_\star$, there exists a $c_3$ such that the lemma holds.

\end{proof}
}

\subsection{Proof of Theorem~\ref{thm:affine}}
\label{sec:proof_aff}
We begin with some background before proceeding to the proof of the theorem.

\subsubsection{Background of proof}

\textbf{A parametrization of paths in $\scrA(D,d)$}. Recall that the underlying affine subspace in our model is $[A_\star]=[(L_\star,0)]$, where we choose the representative $(L_\star, 0)$ WLOG.

For another affine subspace parametrized by $A=(L,\bm)$, let us consider a path between $A_\star$ and $A$ defined as follows: let the principal vectors of $L_\star$ and $L$ be
\[
L_\star=\Sp(\bu_1,\cdots, \bu_d),\,\,
L=\Sp(\bv_1,\cdots, \bv_d),
\]
where $\bv_i=\cos\theta_i\bu_i+\sin\theta_i\bw_i$ such that $\bu_i\perp \bw_j$ for all $1\leq i,j\leq d$. Set $\bw_i=0$ when $\theta_i=0$; the nonzero $\bw_i$ are orthonormal and perpendicular to $L_\star$.

Then a path of affine subspaces from $A_\star$ to $A$ given by $C: [0,1]\rightarrow \scrA(D,d)$ is defined by 
\begin{equation}\label{eq:Ct}
    C(t) = (1-t)\ba_\star + t \ba +\Sp(\{\cos t\theta_i \bu_i +\sin t\theta_i \bw_i\}_{i=1}^d),
\end{equation}
where $\ba_\star$ is the point in $A_\star$ closest to $A$ and $\ba$ is the point on $A$ closest to $A_\star$. Note that $\ba-\ba_\star\perp \Sp(L+ L_\star)$. Denote $\bb=\ba-\ba_\star$, then \begin{equation}\label{eq:Ct2}C(t)=\ba_\star + t \bb +\Sp(\{\cos t\theta_i \bu_i +\sin t\theta_i \bw_i\}_{i=1}^d).\end{equation}

Note that here 
\begin{align}\label{eq:distance_0A}
\dist^2(0,A)&=\dist^2(-\ba,L) \\  &=\dist^2(-\ba_\star-\bb,L)\\&=\|\bb\|^2+\sum_{i=1}^d\dist^2(-\bP_{\Sp(\bu_i,\bw_i)}\ba_\star,\Sp(\{\cos \theta_i \bu_i +\sin \theta_i \bw_i\}_{i=1}^d))\\ \nonumber
&{=}\|\bb\|^2+\sum_{i=1}^d(\ba_\star^\top\bu_i)^2\sin^2\theta_i,
\end{align} 
where the { third step uses $\bb \perp \Sp(L + L_\star)$ and} last step applies that $\ba_\star\perp\bw_i$. 
Therefore the definition in Section~\ref{sec:affine_analysis} gives 
\begin{align}\label{eq:distance_0A2}
    \dist^2(A_\star,[A]) &= \dist^2(0,[A]) + \sum_j \theta_j^2 \\ \nonumber
    &= \|\bb\|^2 + \sum_{i=1}^d\big((\ba_\star^\top\bu_i)^2\sin^2\theta_i+\theta_i^2\big).
\end{align}

 \textbf{Auxiliary functions.} 
 Similar to the proof of Theorem~\ref{thm:global}, the proof of Theorem~\ref{thm:affine} also depends on $F_{\epsilon}^{(a)}(A)$, a regularized version of the objective function $F^{(a)}$ defined in \eqref{eq:Fepsilon}:
\begin{align}\label{eq:Fepsilon_affine}
F_{\epsilon}^{(a)}(A)&=\sum_{\bx\in\calX:\dist(\bx,A)>\epsilon}\dist(\bx,A)+
\sum_{\bx\in\calX:\dist(\bx,A)\leq \epsilon}\left(\frac{1}{2}\epsilon+\frac{\dist(\bx,A)^2}{2\epsilon}\right).
\end{align}
We also define a majorizing function
\begin{align}
    G_{\epsilon_k}^{(a)}(A,A_0) &= \sum_{\bx\in\calX:\dist(\bx,A_0)>\epsilon_k}\left(\frac{1}{2}\dist(\bx,A_0)+\frac{\dist(\bx,A)^2}{2\dist(\bx,A_0)}\right)+\\ \nonumber
    &\sum_{\bx\in\calX:\dist(\bx,A_0)\leq \epsilon_k}\left(\frac{1}{2}\epsilon_k+\frac{\dist(\bx,A)^2}{2\epsilon_k}\right). 
\end{align}

\subsubsection{Proof of Theorem~\ref{thm:affine}}

We prove that the algorithm ensures the sequence $F^{(a)}_{\epsilon_k}(A^{(k)})$ is nonincreasing over iterations. That is, we aim to generalize  Lemma~\ref{lemma:decrease} to the affine setting. 
\begin{lemma}[Decrease over iterations, affine case]\label{lemma:decrease_affine} 
Under Assumptions~\ref{assumption3} and~\ref{assumption4}, let $0<\dist(A_\star,[A])\leq c_0$ and $\epsilon>0$, and suppose that at least half the inliers have distance at least $\epsilon$ from $[A]$. Then, with a constant $c>0$ independent of $A,\epsilon$,
\begin{equation}\label{eq:decrease_equation0_affine}
F_{\epsilon}(A)-F_{\epsilon}(T^{(a)}(A))\geq c \epsilon  \dist(A_\star,[A]).
\end{equation}
\end{lemma}
\begin{proof}[Proof of Lemma~\ref{lemma:decrease_affine}]
{ We break the proof of this lemma into the following steps. First, we decompose $\dist^2(\bx,C(t))$ into constituent parts. We then compute the first two derivatives of this decomposition with respect to $t$ and bound their value. Finally, we can use the fact that $G_\epsilon^{(a)}$ is made up of these quadratic terms combined with the derivative bounds to ensure this sufficient decrease.}

\noindent \textbf{Decomposing $\dist^2(\bx,C(t))$:} %
Let us first investigate $\dist^2(\bx,C(t))$. Since $C(t)$ defined in \eqref{eq:Ct} can be expressed by $C(t)=(t\bb+\ba_\star, L(t))$ and $L(t)$ can be decomposed into the direct sum of $d$ components:%
\begin{align}
\bP_{\Sp(\bu_i,\bw_i)}L(t)&= \Sp(\cos t\theta_i \bu_i +\sin t\theta_i \bw_i), 1\leq i\leq d.%
\end{align}

{Then, the distances $\dist^2(\bx,C(t))$ can also be decomposed using
\begin{align}\label{eq:dist_x_Ct}
    &\dist^2(\bx,C(t)) = \dist^2\left(\bx-(t\bb+\ba_\star),L(t)\right)  \\\nonumber &= \left\|\bP_{\Sp(L+ L_\star)^\perp}\left(\bx-(t\bb+\ba_\star)\right)\right\|^2+\sum_{i=1}^d\dist^2\left(\bP_{\Sp(\bu_i,\bw_i)}\left(\bx-(t\bb+\ba_\star)\right),\bP_{\Sp(\bu_i,\bw_i)}L(t)\right)\\\nonumber &= \left\|\bP_{\Sp(L+ L_\star)^\perp}\bx-t\bb\right\|^2+\sum_{i=1}^d\dist^2\left(\bP_{\Sp(\bu_i,\bw_i)}\left(\bx-\ba_\star\right),\Sp(\cos t\theta_i \bu_i +\sin t\theta_i \bw_i)\right) %
    \\\nonumber
    &= \|\bP_{\Sp(L+ L_\star)^\perp}\bx-t\bb\|^2+\sum_{i=1}^d ((\bx-\ba_\star)^\top\bu_i\sin t\theta_i -\bx^\top\bw_i\cos t\theta_i)^2,
\end{align}
where the second equality follows from $\ba_\star\in L_\star$ and $\bb\perp L+ L_\star$ and the last step also follows from $\bw_i\perp L_\star$.

\noindent \textbf{Derivatives of $\dist^2(\bx,C(t))$:} Note that the second derivative satisfies $\left|\frac{\di^2}{\di t^2}(a\sin t+b\cos t)^2\right|\leq2(a^2+b^2)$, so \begin{align}
\left|\frac{\di^2}{\di t^2}\dist^2(\bx,C(t))\right| &\leq 2\|\bb\|^2+2\sum_{i=1}^d\theta_i^2\Big[((\bx-\ba_\star)^\top\bu_i)^2+(\bx^\top\bw_i)^2\Big]
\label{eq:x_C(t)_secondderivative}\\ &\leq 2\|\bb\|^2+2\sum_{i=1}^d\theta_i^2\left[2(\ba_\star^\top\bu_i)^2+2(\bx^\top\bu_i)^2+(\bx^\top\bw_i)^2\right]\nonumber\\
&\leq 6\|\bx\|^2\sum_{i=1}^d\theta_i^2+4(\pi/2)^2\dist^2(0,[A])\nonumber\\
&\leq10\max(\|\bx\|^2,1)\dist^2(A_\star,[A]),\nonumber
\end{align}
where the last two steps apply \eqref{eq:distance_0A} and $\theta\leq \pi\sin\theta/2$.}

We proceed by studying
\[
\phi(t)=G_\epsilon^{(a)}(C(t),A)
=K(A,\epsilon)+\frac12\sum_{\bx\in\calX}
\frac{\dist(\bx,C(t))^2}{\max(\dist(\bx,A),\epsilon)},
\]
where $K(A,\epsilon)$ is independent of $t$. In particular, we will bound the first and second derivatives of $\phi$.%

\noindent \textbf{Bounding derivatives of $\varphi$:} 
{Combining \eqref{eq:x_C(t)_secondderivative} for all $\bx\in\calX$, we have 
\[
\phi''(t)=\frac{\di^2}{\di t^2}G^{(a)}_\epsilon(C(t),A)\leq \sum_{\bx\in\calX}\frac{\left|\frac{\di^2}{\di t^2}\dist^2(\bx,C(t))\right|}{\max(\dist(\bx,A),\epsilon)}\leq \frac{1}{\epsilon}\sum_{\bx\in\calX}\left|\frac{\di^2}{\di t^2}\dist^2(\bx,C(t))\right|\leq \frac{C}{\epsilon}\dist^2(A_\star,[A])
\]
for $C=10\sum_{\bx\in\calX}\max(\|\bx\|^2,1)$ that only depends on $\calX$.}

As for the first derivative of $\dist^2(\bx,C(t))$ at $t=1$, we have
\begin{align}\nonumber&
\frac12\frac{\di}{\di t}\dist^2(\bx,C(t))\Big|_{t=1}=\frac12\frac{\di}{\di t}\dist^2(\bx-t\bb-\ba_\star,L(t))\Big|_{t=1}=-\bb^\top(\bP_{\Sp(L+ L_\star)^\perp}\bx-\bb)\\&+\sum_{i=1}^d\theta_i\Big((\bx-\ba_\star)^\top\bu_i\cos \theta_i + \bx^\top\bw_i\sin \theta_i\Big)\Big((\bx-\ba_\star)^\top\bu_i\sin \theta_i -\bx^\top\bw_i\cos \theta_i\Big).
\label{eq:deriv_d2_affine}
\end{align}
Combining it with 
\[
\dist(\bx,C(1))=\sqrt{\Big\|\bP_{\Sp(L+ L_\star)^\perp}\bx-\bb\Big\|^2+\sum_{i=1}^d\Big((\bx-\ba_\star)^\top\bu_i\sin \theta_i -\bx^\top\bw_i\cos \theta_i\Big)^2},
\]
we have { by Cauchy-Schwarz on \eqref{eq:deriv_d2_affine}
\begin{align}
\frac{\left|\frac{\di}{\di t}\dist^2(\bx,C(t))\big|_{t=1}\right|}{2\dist(\bx,C(1))}\leq& \sqrt{\|\bb\|^2+\sum_{i=1}^d\theta_i^2\Big((\bx-\ba_\star)^\top\bu_i\cos \theta_i + \bx^\top\bw_i\sin \theta_i\Big)^2}\\\leq& \sqrt{\|\bb\|^2+\sum_{i=1}^d\theta_i^2\Big(2\big((\bx-\ba_\star)^\top\bu_i\cos \theta_i\big)^2 +2\big( \bx^\top\bw_i\sin \theta_i\big)^2\Big)}\\
\leq &\sqrt{\|\bb\|^2+2\sum_{i=1}^d\theta_i^2\Big((\bx-\ba_\star)^\top\bu_i\Big)^2}+\sqrt{2\sum_{i=1}^d\theta_i^4(\bx^\top\bw_i)^2}\\
\leq& \sqrt{2}\left(\frac{\pi}{2}\dist(\bP_{A_\star}\bx, [A])+\sqrt{\sum_{i=1}^d\theta_i^4(\bx^\top\bw_i)^2}\right),\label{eq:X_ct_first_derivative1}
\end{align}
where the first inequality follows from the Cauchy–Schwarz inequality, the second follows from $(a+b)^2\leq 2a^2+2b^2$, the third from  $\sin\theta_i\leq \theta_i$ and $\sqrt{a+b}\leq \sqrt{a}+\sqrt{b}$, and the last one again from $\sqrt{a+b}\leq \sqrt{a}+\sqrt{b}$ and 
\begin{align*}
\dist(\bP_{A_\star}\bx, [A])
=&\dist(
\bP_{L_\star}(\bx-\ba_\star), [A]-\ba_\star) \\ \nonumber
=&
\dist(
\bP_{L_\star}(\bx-\ba_\star), (L,\ba-\ba_\star))\\ \nonumber
=&\dist(
\bP_{L_\star}(\bx-\ba_\star), (L,\bb))\\
=&\sqrt{\|\bb\|^2+\sum_{i=1}^d\dist^2(\bP_{\Sp(\bu_i,\bw_i)}\bP_{L_\star}(\bx-\ba_\star),\bP_{\Sp(\bu_i,\bw_i)}L)}\\ \nonumber
=&\sqrt{\|\bb\|^2+\sum_{i=1}^d\dist^2(\bP_{\Sp(\bu_i,\bw_i)}\bP_{L_\star}(\bx-\ba_\star),\cos\theta_i\bu_i+\sin\theta_i\bw_i)}\\
=&\sqrt{\|\bb\|^2+\sum_{i=1}^d\dist^2(\bP_{\Sp(\bu_i)}(\bx-\ba_\star),\bP_{\Sp(\bu_i,\bw_i)}L)}\\ \nonumber
=&\sqrt{\|\bb\|^2+\sum_{i=1}^d\sin^2\theta_i\left((\bx-\ba_\star)^\top\bu_i\right)^2}\\\geq& \frac{2}{\pi}\sqrt{\|\bb\|^2+\sum_{i=1}^d\theta_i^2\left((\bx-\ba_\star)^\top\bu_i\right)^2}.
\end{align*}
On the other hand, if we restrict that $\bx\in \calX_{\mathrm{in}}$, then 
\begin{align}\nonumber
\dist^2(\bx,C(1))&=\dist^2(\bx,[A])=\|\bb\|^2+\sum_{i=1}^d\Big((\bx-\ba_\star)^\top\bu_i\Big)^2\sin^2 \theta_i\\
\frac{\di}{\di t }\dist^2(\bx,C(t))\Big|_{t=1}&=2\|\bb\|^2+2\sum_{i=1}^d\theta_i\Big((\bx-\ba_\star)^\top\bu_i\Big)^2\sin \theta_i\cos \theta_i\geq 2\cos\theta_1\dist^2(\bx,A),\label{eq:X_ct_first_derivative2}
\end{align}
where the last step applies $\dist^2(\bx, A)=\dist^2(\bx, C(1))=\|\bb\|^2+\sum_{i=1}^d\sin^2\theta_i\Big((\bx-\ba_\star)^\top\bu_i\Big)^2$ (which follows from \eqref{eq:dist_x_Ct} with $t=1$ and $\bx\in\calX_{\mathrm{in}}$),  $\cos\theta_i\geq \cos\theta_1$, and $\theta_i\geq \sin\theta_i$.

Combining the estimations above, along with
\begin{align*}
\sum_{\bx\in\calX_{\mathrm{out}}}
\sqrt{\sum_{i=1}^d\theta_i^4(\bx^\top\bw_i)^2}
&\leq\left(\nout\sum_{i=1}^d\theta_i^4
\sum_{\bx\in\calX_{\mathrm{out}}}(\bx^\top\bw_i)^2\right)^{1/2}\\
&\leq\mathcal R_{\mathrm{out}}(A):=\left(\nout\sum_{j=1}^d\theta_j^4\right)^{1/2}
\left\|\bP_{L_\star^\perp}\sum_{\calX_{\mathrm{out}}}
\bx\bx^\top\bP_{L_\star^\perp}\right\|_2^{1/2},
\end{align*}
where Cauchy--Schwarz is applied,
}
we have that for $\phi(t)=G_{\epsilon}^{(a)}(C(t),A)$, {there exists $\calX_{\mathrm{in}}'\subseteq \calX_{\mathrm{in}}$ such that  $|\calX_{\mathrm{in}}'|\geq |\calX_{\mathrm{in}}|/2$, and
\begin{align*}
\phi'(1)= &\sum_{\bx\in\calX_{\mathrm{in}}} \frac{\frac{\di}{\di t }\dist^2(\bx,C(t))\Big|_{t=1}}{2\max(\dist(\bx,A),\epsilon)}+\sum_{\bx\in\calX_{\mathrm{out}}} \frac{\frac{\di}{\di t }\dist^2(\bx,C(t))\Big|_{t=1}}{2\max(\dist(\bx,A),\epsilon)}
\\\geq& \sum_{\bx\in\calX_{\mathrm{in}}}\frac{\cos\theta_1\dist^2(\bx,A)}{\max(\dist(\bx,A),\epsilon)} - \sqrt{2}\left(\frac{\pi}{2}\sum_{\bx\in\calX_{\mathrm{out}}}\dist(\bP_{A_\star}\bx,[A])+\mathcal R_{\mathrm{out}}(A)\right)
\\\geq& \sum_{\bx\in\calX_{\mathrm{in}}'}\frac{\cos\theta_1\dist^2(\bx,A)}{\dist(\bx,A)} - \sqrt{2}\left(\frac{\pi}{2}\sum_{\bx\in\calX_{\mathrm{out}}}\dist(\bP_{A_\star}\bx,[A])+\mathcal R_{\mathrm{out}}(A)\right),
\end{align*}
where the first inequality follows from \eqref{eq:X_ct_first_derivative1} and \eqref{eq:X_ct_first_derivative2}, and the second inequality follows from the fact that $\epsilon$ is chosen such that it is smaller than the median of $\{\dist(\bx,A)\}_{\bx\in\calX_{\mathrm{in}}}.$}

Recall from definition of $F, G$, we have
\[
\min_t\phi(t)\geq G_{\epsilon}(T^{(a)}(A),A)\geq F_{\epsilon}(T^{(a)}(A)).
\]
Combining this with Assumption~\ref{assumption4} and Lemma~\ref{lemma:affineinbd} gives $\phi'(1)\geq a\dist(A_\star,[A])$ for some $a>0$: the sum of the two outlier terms is at most twice their maximum. The cutoff condition also gives
\[
\epsilon\leq\frac{2F_{\mathrm{in}}^{(a)}(A)}{\nin}
\leq B\dist(A_\star,[A]),\qquad
B=\frac2\nin\sum_{\calX_{\mathrm{in}}}(\|\bx\|+1).
\]
Increase the curvature constant $C$ so that $C\geq aB$, and take
$s=a\epsilon/(C\dist(A_\star,[A]))\leq1$. Taylor's inequality at the admissible point $t=1-s$ yields
\[
F_\epsilon(A)-F_\epsilon(T^{(a)}(A))
\geq\phi(1)-\phi(1-s)
\geq s\phi'(1)-\frac{C\dist^2(A_\star,[A])}{2\epsilon}s^2
\geq\frac{a^2}{2C}\epsilon.
\]
This uniform estimate implies \eqref{eq:decrease_equation0_affine}, since $\dist(A_\star,[A])\leq c_0$, and is the estimate used below.
\end{proof}

The following lemma generalizes  Lemma~\ref{lemma:assumption1} and the proof is similar.
\begin{lemma}\label{lemma:assumption1_affine} Under Assumption~\ref{assumption3}, there exists $c_3>0$ such that for all $[A]\in \scrA(D,d) $,
\begin{equation}\label{eq:assumption12_2}
q_{\gamma}(\{\dist(\bx,[A])\}_{\bx\in\calX})\geq c_3\frac{F^{(a)}_{\mathrm{in}}(A)}{|\calX_{\mathrm{in}}|}.
\end{equation}
\end{lemma}
To adapt Lemma~\ref{lemma:assumption1}, write $[A]=L+\bv$, $\bv\perp L$, and set $\delta=\dist(A_\star,[A])$. For $\|\bv\|\to\infty$, both the quantile and $\delta$ are asymptotic to $\|\bv\|$. For bounded translations, the same compactness argument applies; near $[A_\star]$, the normalized maps $[\bQ_L\bU_{L_\star},-\bv]/\delta$ have a nonzero subsequential limit whose zero set lies in a proper affine hyperplane. Assumption~\ref{assumption3} excludes a vanishing normalized quantile as before. Finally, $F_{\mathrm{in}}^{(a)}(A)\leq\delta\sum_{\calX_{\mathrm{in}}}(\|\bx\|+1)$ gives the claim.

{
The next lemma shows that $\dist(A_\star, [A])$ serves as a lower bound for the AFMS objective over at least half of the inliers:

\begin{lemma}\label{lemma:affineinbd}
Under Assumption~\ref{assumption3}, for any subset $\calX_{\mathrm{in}}' \subseteq \calX_{\mathrm{in}}$ with $|\calX_{\mathrm{in}}'| \geq \frac{1}{2} |\calX_{\mathrm{in}}|$, we have
\begin{equation}\label{eq:affineinbd}
   \sum_{\bx \in \calX_{\mathrm{in}}'} \dist(\bx, A) \geq c \cdot \dist(A_\star, [A]),
\end{equation} 
for some constant $c > 0$.
\end{lemma}

\begin{proof}[Proof of Lemma~\ref{lemma:affineinbd}]
Let $\bb$ denote the projection of the origin onto the affine subspace $[A]$. For each inlier $\bx\in A_\star$, we can express
\[
\bx=\sum_{j=1}^d a_j\bv_j,
\]
where $\{\bv_j\}_{j=1}^d$ form an orthonormal basis for the linear part of $A_\star$. Write $\bV=[\bv_1,\ldots,\bv_d]$, $y_i=(a_{i1},\ldots,a_{id},1)^\top$, and $M=[\bQ_L\bV,-\bb]$. Then $\dist(\bx_i,A)=\|My_i\|$. Every retained half-sample affinely spans $L_\star$ by Assumption~\ref{assumption3}; hence its augmented data matrix $Y_I=[y_i]_{i\in I}$ has full row rank. There are finitely many such subsets, so $s_\star=\min_{|I|\geq\nin/2}\sigma_{d+1}(Y_I)>0$.

Note that the distance from $A_\star$ to $[A]$ satisfies
\[
\dist(A_\star,[A])^2=\|\bb\|^2+\sum_{j=1}^d\theta_j^2.
\]
Therefore,
\[
\sum_{i\in I}\dist(\bx_i,A)\geq\|MY_I\|_F\geq s_\star\|M\|_F
=s_\star\left(\|\bb\|^2+\sum_j\sin^2\theta_j\right)^{1/2}
\geq\frac{2s_\star}{\pi}\dist(A_\star,[A]).
\]
Therefore, the bound~\eqref{eq:affineinbd} must hold for some constant $c>0$, completing the proof.
\end{proof}

}

\begin{proof}[Proof of Theorem~\ref{thm:affine}]
The proof of Theorem~\ref{thm:affine} follows similarly from
the proof of Theorem~\ref{thm:global}. In particular,
the uniform estimate established in the proof of
Lemma~\ref{lemma:decrease_affine} implies that
for $k\geq1$,
\[
F_{\epsilon_k}^{(a)}(A^{(k)})
-F_{\epsilon_k}^{(a)}(T^{(a)}(A^{(k)}))
\geq C\epsilon_k.
\]
On the other hand, using the fact that
at least half of the inliers have distance
at least $\epsilon_k$ from $A^{(k)}$, we have
\begin{align*}
F^{(a)}_{\mathrm{in},\epsilon_k}(A^{(k)})
&\leq
\sum_{\bx\in\calX_{\mathrm{in}}}
\max(\dist(\bx,A^{(k)}),\epsilon_k)\\
&\leq
\sum_{\bx\in\calX_{\mathrm{in}}}\dist(\bx,A^{(k)})
+\epsilon_k\frac{|\calX_{\mathrm{in}}|}{2}\\
&\leq
2\sum_{\bx\in\calX_{\mathrm{in}}}\dist(\bx,A^{(k)})
=2F^{(a)}_{\mathrm{in}}(A^{(k)}).
\end{align*}
Moreover, the argument in Lemma~\ref{lemma:decrease_affine}
implies that
\[
|F^{(a)}_{\mathrm{out}}(A^{(k)})
-F^{(a)}_{\mathrm{out}}(A_\star)|\leq F^{(a)}_{\mathrm{in}}(A^{(k)})/2.\]
For $A=C(1)$, \eqref{eq:distance_0A2} with $(\bb,\theta_i)$ replaced by $(t\bb,t\theta_i)$ gives $\dist(A_\star,[C(t)])\leq\dist(A_\star,[A])\leq c_0$ for $0\leq t\leq1$. Applying \eqref{eq:X_ct_first_derivative1}--\eqref{eq:X_ct_first_derivative2} and Assumption~\ref{assumption4} to the subpath $s\mapsto C(st)$ gives, almost everywhere on $(0,1)$,
\begin{align*}
t\left|\frac{\di}{\di t}F_{\mathrm{out}}^{(a)}(C(t))\right|
&\leq \frac{\cos c_0}{2}F_{\mathrm{in}}^{(a)}(C(t))\\
&\leq \frac{t}{2}\frac{\di}{\di t}F_{\mathrm{in}}^{(a)}(C(t)).
\end{align*}
The second inequality follows from \eqref{eq:X_ct_first_derivative2} on the subpath, since $\cos(t\theta_1)\geq\cos c_0$. Canceling $t$ and integrating proves the claim. The quantile rule also gives
\[0\leq F_{\epsilon_k}^{(a)}(A^{(k)})-F^{(a)}(A^{(k)})
\leq \frac{\epsilon_k\gamma|\calX|}{2}
\leq \frac12F^{(a)}_{\mathrm{in}}(A^{(k)}).\]
In summary,
\[
F^{(a)}_{\epsilon_k}(A^{(k)})-F^{(a)}(A_\star)
\leq 3F^{(a)}_{\mathrm{in}}(A^{(k)}).
\]
At $k=0$, the same upper estimate holds with a
dataset-dependent constant in place of $3$, by the triangle
inequality and Lemma~\ref{lemma:affineinbd}.
Combining this with
\begin{align*}
F^{(a)}_{\mathrm{in}}(A^{(k)})
&=\sum_{\bx\in\calX_{\mathrm{in}}}\dist(\bx,A^{(k)})\\
&\leq
\sum_{\bx\in\calX_{\mathrm{in}}}
\big(
\dist(\bx,L^{(k)})+\dist(A_\star,[A^{(k)}])
\big)\\
&\leq
\dist(A_\star,[A^{(k)}])
\sum_{\bx\in\calX_{\mathrm{in}}}(\|\bx\|+1)
\end{align*}
and Lemmas~\ref{lemma:assumption1_affine}
and~\ref{lemma:affineinbd}, the same argument as in
Theorem~\ref{thm:global} can be applied to prove
linear convergence.
\end{proof}

\subsection{Proof of Proposition \ref{prop:haystack_assump}}\label{sec:haystack}

We verify Assumptions~\ref{assump:global} and~\ref{assump:global2} in Sections~\ref{subsec:haystack_assump2_proof} and~\ref{subsec:haystack_assump3_proof}, respectively. Under the i.i.d. mixture model, the component counts are random. Define
\[
\mathcal E_n=\left\{
\tfrac34 n\alpha_s\leq n_s\leq\tfrac54 n\alpha_s
\text{ for }s\in\{\mathrm{in},\mathrm{out}\}
\right\}.
\]
A multiplicative binomial concentration bound and a union bound give
\[
\Pr(\mathcal E_n^c)
\leq4\exp\left(-\frac{n\min(\alpha_{\mathrm{in}},\alpha_{\mathrm{out}})}{48}\right).
\]
On $\mathcal E_n$,
\[
\frac{\nin}{\nout}\geq\frac35
\frac{\alpha_{\mathrm{in}}}{\alpha_{\mathrm{out}}},
\qquad
\gamma_\star=\frac{\nin}{n}\geq\frac34\alpha_{\mathrm{in}}.
\]
Consequently, the population hypotheses of Proposition~\ref{prop:haystack_assump} imply the empirical ratio and inlier-fraction conditions used below, after enlarging the absolute constant in~\eqref{eq:prop:haystack_assump} if necessary.

We now condition on the component labels. Conditional on these labels, the inlier and outlier observations are independent Gaussian samples from their respective components, with fixed counts $\nin$ and $\nout$. All probability estimates in the remainder of this proof are understood conditionally on these labels until the conditioning is removed. Combining~\eqref{eq:allprob1}, \eqref{eq:allprob2}, \eqref{eq:allprob3}, \eqref{eq:allprob4}, and~\eqref{eq:allprob5}, the conditional success probability is at least

\begin{align*}
1-&e^{-c\nin}
-C\exp\{CD^2\log(CD)-c\nout/(D-d)\}\\
&-(\mathcal N_{\mathrm{in}}+1)
  \big(e^{-\nin/8}+e^{-\nin\alpha_0^2/8}\big)\\
&-C\exp\{CD^2\log(CD)-c\nout/(D-d)\}
-e^{-\nin/8}-e^{-\nout\alpha_1^2/8}\\
&-2(C/\alpha_0^{3/2})^d e^{-\nin\alpha_0^2/200}
-2e^{-c\nin}-e^{-\nout/2},
\end{align*}
where $\mathcal N_{\mathrm{in}}=(1+2/\epsilon_{\mathrm{net}})^{d^2}$ and $\epsilon_{\mathrm{net}}$ is the radius specified in the verification of~\eqref{eq:global21}, and
$C$ is an absolute constant chosen to dominate the constants in the referenced estimates. Let $B(\nin,\nout)$ denote the sum of the failure terms in the preceding display. In the term arising from~\eqref{eq:global22}, the conditional quantile level is
\[
\alpha_1=\frac{\gamma-\alpha_0\gamma_\star}{1-\gamma_\star}
=\frac{\nin}{4\nout}.
\]
When $\alpha_1>1$, condition~\eqref{eq:global22} is not required, and its failure term is omitted from $B$. Removing the conditioning, the success probability is bounded below by
\[
1-4\exp\left(-\frac{n\min(\alpha_{\mathrm{in}},\alpha_{\mathrm{out}})}{48}\right)
-\sup_{(\nin,\nout)\in\mathcal E_n}B(\nin,\nout),
\]
where the supremum is over integer count pairs with $\nin+\nout=n$ satisfying the defining inequalities of $\mathcal E_n$. For fixed positive mixture proportions, the counts satisfy $n_s\asymp n\alpha_s$ on $\mathcal E_n$, and the conditional failure estimates tend to zero uniformly over these count pairs in the stated regime $n\gtrsim D^3\log D$. This proves the unconditional high-probability conclusion. %

\subsubsection{Verification of Assumption \ref{assump:global}}
\label{subsec:haystack_assump2_proof}

The following lemma gives bounds on the statistics $\Sin$ and $\Sout$ defined in \eqref{eq:Sin} under this model with high probability. Clearly, given the assumption of the proposition, this lemma implies Assumption~\ref{assump:global} holds with fixed $\theta_0$, with probability tending to 1 as $n, D \to \infty$. 

\begin{lemma}[Estimation of parameters under the Haystack model]\label{lemma:prob}
Assuming that $n \gtrsim D^3 \log D$, then we have the following estimates of the parameters in Theorem~\ref{thm:global} and Assumption~\ref{assump:global} with high probability: \[\Sin \gtrsim  \nin \sigma_{\mathrm{in}}/\sqrt{d},\,\,\, \Sout\lesssim  \frac{ \nout \sigma_{\mathrm{out}}}{\sqrt{D(D-d)}}. 
\]
\end{lemma}

\begin{proof}[Proof of Lemma~\ref{lemma:prob}] We bound the inlier and outlier terms separately.\\

\noindent \textbf{Bound for $S_{\mathrm{in}}$:}
For $\Sin$, write $\bx=\bU_\star B\bz$, where $\bz\sim N(0,\bI_d)$ and $\sigma_d(B)\geq\sigma_{\mathrm{in}}/\sqrt d$. Applying the Gaussian permeance bound after this change of coordinates, 
it follows from \cite[Lemma 8.2]{lerman2015robust} that
\[
\Pr\left(\Sin\geq \left(\sqrt{\frac{2}{\pi}}n_{\mathrm{in}}-2\sqrt{n_{\mathrm{in}}d}-t\sqrt{n_{\mathrm{in}}}\right)\frac{\sigma_{\mathrm{in}}}{\sqrt d}\right)\geq 1-e^{-t^2/2}.
\]
So letting $t=c\sqrt{n_{\mathrm{in}}}$ with $c>0$ small and $n_{\mathrm{in}}\geq Cd$, we find that $\Sin \gtrsim  \nin \sigma_{\mathrm{in}}/\sqrt{d}$ with probability \begin{equation}\label{eq:allprob1}1-\exp(-n_{\mathrm{in}}/2C).\end{equation}

\noindent \textbf{Bound for $S_{\mathrm{out}}$:} For the capped statistic in \eqref{eq:Sin}, we first prove its bound for any fixed $L$ and then apply an $\epsilon$-net argument.
Put $N=\nout$, $m=D-d$, $Q=\bQ_L$, and let $X_{\mathrm{out}}$ have the outliers as columns. Set $Y_{Q,\epsilon}=[Q\bx_i/\max(\|Q\bx_i\|,\epsilon)]_{i=1}^N$ for $\epsilon>0$.

\noindent\emph{Estimation for fixed $L\in G(D,d)$:}
The Gaussian matrix bound \citep{wainwright2019high,vershynin2018high} gives $\|X_{\mathrm{out}}\|^2\leq CN\sigma_{\mathrm{out}}^2/D$ with failure $2e^{-cN}$ for $N\geq CD$. Also,
\begin{equation}\label{eq:singleL}
\left\|\sum_{i=1}^N\frac{\bP_L\bx_i\bx_i^\top Q}{\max(\|Q\bx_i\|,\epsilon)}\right\|
=\|\bP_LX_{\mathrm{out}}Y_{Q,\epsilon}^\top\|
\leq\|X_{\mathrm{out}}\|\,\|Y_{Q,\epsilon}\|.
\end{equation}
Choose $\nu=a\sigma_{\mathrm{out}}\sqrt{m/D}$ with $a>0$ sufficiently small, and $R=C\sigma_{\mathrm{out}}\sqrt m$ with $C$ sufficiently large. The eigenvalues of $Q\bSigma_{\mathrm{out}}Q$ on its range lie between $c\sigma_{\mathrm{out}}^2/D$ and $\sigma_{\mathrm{out}}^2/D$. Diagonalizing this covariance and using the chi-square lower tail gives
\[
\Pr\{\|Q\bx\|\leq2\nu\}\leq\Pr\{\chi_m\leq Ca\sqrt m\}
\leq(Ca\sqrt e)^m\leq\frac1{10^5m}.
\]
Further, $\Pr\{\|\bx\|>R\}\leq\operatorname{tr}(\bSigma_{\mathrm{out}})/R^2\leq1/(10^5m)$. Each empirical count is at most $N/(4000m)$ with failure $e^{-cN/m}$ by a multiplicative binomial bound.

\noindent\emph{Covering argument:}
Take a Frobenius-norm net of rank-$m$ projectors with radius $\eta=\nu/R=c/\sqrt D$ and cardinality at most $(1+2\sqrt D/\eta)^{D^2}\leq\exp\{CD^2\log(CD)\}$. For a neighboring net projector $Q_0$,
\begin{equation}\label{eq:singleL2}
\|Q\bx_i\|\leq\nu,\quad\|\bx_i\|\leq R
\quad\Longrightarrow\quad\|Q_0\bx_i\|\leq2\nu.
\end{equation}
A union bound yields $\#\{i:\|Q\bx_i\|\leq\nu\}\leq N/(2000m)$ uniformly in $Q$. For unit $u\in\operatorname{range}(Q)$, split $\sum_i|u^\top(Y_{Q,\epsilon})_i|^2$ at residual $\nu$: the first part is at most $N/(2000m)$, and the other part is at most $\|X_{\mathrm{out}}\|^2/\nu^2\leq CN/m$. Thus $\sup_{Q,\epsilon}\|Y_{Q,\epsilon}\|^2\leq CN/m$. Combining with \eqref{eq:singleL} gives $\Sout\lesssim\nout\sigma_{\mathrm{out}}/\sqrt{D(D-d)}$ with probability
\begin{equation}\label{eq:allprob2}
1-C\exp\{CD^2\log(CD)-c\nout/(D-d)\}.
\end{equation}
The same count bound gives $q_\eta(\{\dist(\bx,L)\}_{\calX_{\mathrm{out}}})\geq\nu$ for every $\eta\geq1/(1000m)$, uniformly in $L$. This proves Lemma~\ref{lemma:prob}.
\end{proof}

{
\subsubsection{Verification of Assumption \ref{assump:global2} under the Generalized Haystack Model}
\label{subsec:haystack_assump3_proof}

We begin by presenting a lemma that introduces and reviews several key probability bounds for generalized chi-squared and normal distributions, which will be used in the verification.
\begin{lemma}\label{lemma:generalized_chisquared}
(a) Assuming that $\bx_1,\cdots,\bx_n\sim N(0,\bSigma)$ with $\bSigma\in\reals^{p\times p}$, then 
\[
\Pr\left(\frac{1}{n}\sum_{i=1}^n\|\bx_i\|\leq \sqrt{2\tr(\bSigma)}\right)\geq 1-e^{-n/8}.
\]
(b) Assuming that $x_1,\cdots,x_n$ are $n$ samples of a random variable $X$, then its empirical quantiles are not far away from its population quantiles:
\[
\Pr\left(q_{\alpha_0}(\{x_i\}_{i=1}^n)\geq q_{\alpha_0/2}(X)\right)\geq 1-e^{-n\alpha_0^2/2},\,\,\Pr\left(q_{1-\alpha_0}(\{x_i\}_{i=1}^n)\leq q_{1-\alpha_0/2}(X)\right)\geq 1-e^{-n\alpha_0^2/2}.
\]
(c) Assuming that $\bx_1,\cdots,\bx_n\sim N(0,\bSigma)$ with $\bSigma\in\reals^{p\times p}$, then 
\[
\Pr\left(q_{\alpha_0/2}(\{\|\bx_i\|\}_{i=1}^n)\geq \frac{\alpha_0}{4e}\sqrt{2\tr(\bSigma)}\right)\geq 1-e^{-n\alpha_0^2/8}
\]
(d) Assuming that $\bx_1,\cdots,\bx_n\sim N(0,\bI_{p\times p})$, then 
\[
\Pr\left(q_{1-\alpha_0/2}(\{\|\bx_i\|\}_{i=1}^n)\leq \sqrt{p}-2\log \frac{\alpha_0}{4}\right)\geq 1- e^{-n\alpha_0^2/8}
\]
and when $p\geq4\log(4/\alpha_0)$, 
\[
\Pr\left(q_{\alpha_0/2}(\{\|\bx_i\|\}_{i=1}^n)\geq \sqrt{\left(1-\sqrt{\frac{4}{p}\log\frac{4}{\alpha_0}}\right)p} \right)\geq 1- e^{-n\alpha_0^2/8}
\]
(e)  If $g_1,\cdots,g_{n}$ are from $N(0,1)$, then for $0\leq \alpha_0\leq 1$, 
\[
\Pr\left(q_{\alpha_0}(\{|g_i|\}_{i=1}^n)\geq \alpha_0/2\right)\geq 1- e^{-n\alpha_0^2/2}
\]
and 
\[
\Pr\left(\sum_{i=1}^ng_i^2\cdot I(|g_i|\leq\alpha_0/2)\geq \frac{1}{100}n\alpha_0^3\right)\geq 1-e^{-n\alpha_0^2/200}.
\]

\end{lemma}

We will first assume Lemma~\ref{lemma:generalized_chisquared} is true and verify \eqref{eq:global21},  \eqref{eq:global22} and \eqref{eq:global23}. We then prove the lemma.

\textbf{Verification of \eqref{eq:global21}:} Let $\bz=(\bU_{L_\star}^\top\bSigma_{\mathrm{in}}\bU_{L_\star})^{-1/2}\bU_{L_\star}^\top\bx$ and $\bW=(B\bU_{L_\star}^\top\bQ_L\bU_{L_\star}B)^{1/2}$, where $B=(\bU_{L_\star}^\top\bSigma_{\mathrm{in}}\bU_{L_\star})^{1/2}$, then $\dist(\bx,L)=\|\bQ_L\bx\|=\|\bW\bz\|$ and $\bz\sim N(0,\bI_{d\times d})$. As a result, it is equivalent to prove that for $\bz_1,\bz_{n_{\mathrm{in}}}\sim \mathcal{N}(0,\bI_{d\times d})$,
  \begin{equation}\label{eq:global21_proof}
       \mathrm{mean}(\{\|\bW\bz_i\|\}_{1\leq i\leq n_{\mathrm{in}}})\leq \beta_1 q_{\alpha_0}(\{\|\bW\bz_i\|\}_{1\leq i\leq n_{\mathrm{in}}})
    \end{equation}
holds for all $\bW\in\reals^{d\times d}$ with $\|\bW\|_F=1$.

It will be proved by an $\epsilon$-net argument over all such $\bW$. For any fixed  $\bW$, by Lemma~\ref{lemma:generalized_chisquared}(a) and (c) (let $\bSigma=\bW^\top\bW$ and therefore $\tr(\bSigma)=1$), we have 
\[
\Pr\left(\frac{8e}{\alpha_0}q_{\alpha_0/2}(\{\|\bW\bz_i\|\}_{1\leq i\leq n_{\mathrm{in}}}) - \mathrm{mean}(\{\|\bW\bz_i\|\}_{1\leq i\leq n_{\mathrm{in}}})\geq \sqrt{2}\right)\geq 1- e^{-n_{\mathrm{in}}\alpha_0^2/8}-e^{-n_{\mathrm{in}}/8}.  
\]

Then apply the $\epsilon$-net argument to the set of all $d\times d$ matrices $\bW$ using the perturbation argument:
     \begin{equation}\label{eq:global21_proof3}
\mathrm{mean}(\{\|\bW'\bz_i\|\}_{1\leq i\leq n_{\mathrm{in}}})\leq \mathrm{mean}(\{\|\bW\bz_i\|\}_{1\leq i\leq n_{\mathrm{in}}}) + \|\bW-\bW'\|\mathrm{mean}(\{\|\bz_i\|\}_{1\leq i\leq n_{\mathrm{in}}})
    \end{equation}
and by the property of quantile, 
\begin{equation}\label{eq:global21_proof4}
 q_{\alpha_0}(\{\|\bW'\bz_i\|\}_{1\leq i\leq n_{\mathrm{in}}})\geq q_{\alpha_0/2}(\{\|\bW\bz_i\|\}_{1\leq i\leq n_{\mathrm{in}}}) -\|\bW-\bW'\|q_{1-\alpha_0/2}(\{\|\bz_i\|\}_{1\leq i\leq n_{\mathrm{in}}}).
    \end{equation}
Also, Lemma~\ref{lemma:generalized_chisquared}(a) and (d) imply with probability at least $1-e^{-n_{\mathrm{in}}/8}-e^{-n_{\mathrm{in}}\alpha_0^2/8}$,      \begin{align}\label{eq:global21_proof2}\mathrm{mean}(\{\|\bz_i\|\}_{1\leq i\leq n_{\mathrm{in}}})&\leq \sqrt{2d} \\
q_{1-\alpha_0/2}(\{\|\bz_i\|\}_{1\leq i\leq n_{\mathrm{in}}})&\leq \sqrt{d}-2\log\frac{\alpha_0}{4}   \end{align}
Using an $\epsilon$-net argument with 
\[
\epsilon_{\mathrm{net}}=\frac{\sqrt2}{2\left(\sqrt{2d}+\frac{8e}{\alpha_0}(\sqrt d-2\log(\alpha_0/4))\right)},\qquad
\mathcal N_{\mathrm{in}}=(1+2/\epsilon_{\mathrm{net}})^{d^2},
\]
where $\mathcal N_{\mathrm{in}}$ bounds the net cardinality, we have thus shown that for $\beta_1=8e/\alpha_0$, \eqref{eq:global21} holds with probability
\begin{equation}\label{eq:allprob3}
1-(\mathcal N_{\mathrm{in}}+1)(e^{-\nin/8}+e^{-\nin\alpha_0^2/8}).
\end{equation}

\textbf{Verification of \eqref{eq:global22}:} %
The RHS of \eqref{eq:global22} can be controlled by Lemma~\ref{lemma:generalized_chisquared}(a): with probability $1-\exp(-\nin/8)$,
\[
q_{\alpha_0}(\{\|\bx\|\}_{\bx\in\calX_{\mathrm{in}}})
\leq\frac{\sqrt{2\operatorname{tr}(\bSigma_{\mathrm{in}})}}{1-\alpha_0}
\lesssim\frac{\sigma_{\mathrm{in}}}{1-\alpha_0},
\]
using $q_\alpha\leq\mathrm{mean}/(1-\alpha)$ for nonnegative observations.

As for the LHS of \eqref{eq:global22}, similar to \eqref{eq:global21_proof4}, we will use the perturbation bound 
\begin{equation}\label{eq:global21_proof5}
q_{\frac{\gamma-\alpha_0\gamma_\star}{1 - \gamma_\star}}(\{\dist(\bx, L')\}_{\bx \in \calX_{\mathrm{out}}})\geq q_{\frac{\gamma-\alpha_0\gamma_\star}{2(1 - \gamma_\star)}}(\{\dist(\bx, L)\}_{\bx \in \calX_{\mathrm{out}}}) -\|\bP_{L}-\bP_{L'}\|q_{1-\frac{\gamma-\alpha_0\gamma_\star}{2(1 - \gamma_\star)}}(\{\|\bx\|\}_{\bx \in \calX_{\mathrm{out}}}),
\end{equation}
and the uniform count bound in the proof of Lemma~\ref{lemma:prob} implies
\begin{equation}\label{eq:global21_proof6}
q_{\alpha_1/2}(\{\dist(\bx,L)\}_{\calX_{\mathrm{out}}})
\geq a\sigma_{\mathrm{out}}\sqrt{\frac{D-d}{D}},\qquad
\alpha_1=\frac{\gamma-\alpha_0\gamma_\star}{1-\gamma_\star},
\end{equation}
uniformly over $L$. Here $\alpha_1/2\geq\gamma_\star/8\geq\min\{C/(8(D-d)),1/800\}\geq1/(1000(D-d))$ on $\mathcal E_n$, in the branch $\alpha_1\leq1$. The upper-quantile estimate from Lemma~\ref{lemma:generalized_chisquared}(d) is
\begin{equation}\label{eq:global21_proof7}
q_{1-\alpha_1/2}(\{\|\bx\|\}_{\calX_{\mathrm{out}}})
\leq\frac{\sigma_{\mathrm{out}}}{\sqrt D}
\bigl(\sqrt D-2\log(\alpha_1/4)\bigr)\lesssim\sigma_{\mathrm{out}},
\end{equation}
with probability at least $1-e^{-\nout\alpha_1^2/8}$; the endpoint $\alpha_1=1$ follows from the same empirical-count argument. The covering step is already included in \eqref{eq:global21_proof6}.
As a result, \eqref{eq:global22} holds with $\beta_2=C(1+\frac{\sigma_{\mathrm{in}}}{(1-\alpha_0)\sigma_{\mathrm{out}}}\sqrt{D/(D-d)})\geq1$ and probability
\begin{equation}\label{eq:allprob4}
1-e^{-\nin/8}-e^{-\nout\alpha_1^2/8}
-C\exp\{CD^2\log(CD)-c\nout/(D-d)\}.
\end{equation}
When $\alpha_1>1$, condition~\eqref{eq:global22} is omitted.

\textbf{Verification of \eqref{eq:global23}:}
We use the stated covariance bounds directly: $\bu^\top\bSigma_{\mathrm{in}}\bu\geq\sigma_{\mathrm{in}}^2/d$ for unit $\bu\in L_\star$, and $\bSigma_{\mathrm{out}}\preceq\sigma_{\mathrm{out}}^2\bI/D$.

Following a classic result in random matrix theory~\cite[Example 6.2]{wainwright2019high}, we have
\[
\Pr\left(\sigma_1\left(\sum_{\bx \in \calX_{\mathrm{out}}} \bx \bx^\top \right)\leq \frac{\sigma_{\mathrm{out}}^2}{D}(\sqrt{n_{\mathrm{out}}}+\sqrt{D}+t)^2\right)\geq 1-\exp(-t^2/2).
\] 

For the inlier term, set $m_0=\lceil\alpha_0\nin\rceil$. For each unit vector $\bu\in L_\star$, let $\calX_{\mathrm{in},\bu}$ consist of exactly $m_0$ inlier observations with the smallest values of $|\bu^\top\bx|$, counting multiplicities and breaking ties arbitrarily. Define
\[
h(\bu)=\left(\sum_{\bx\in\calX_{\mathrm{in},\bu}}|\bu^\top\bx|^2\right)^{1/2}.
\]
Then
\[
\min_{\substack{\calX'_{\mathrm{in}}\subseteq\calX_{\mathrm{in}}\\
|\calX'_{\mathrm{in}}|\geq\alpha_0\nin}}
\sigma_d\left(\sum_{\bx\in\calX'_{\mathrm{in}}}\bx\bx^\top\right)
=\min_{\substack{\bu\in L_\star\\\|\bu\|=1}}h(\bu)^2.
\]
This exact-cardinality definition avoids including extra observations tied at an empirical quantile. For unit vectors $\bu,\bv\in L_\star$, minimality of $\calX_{\mathrm{in},\bu}$ and the reverse triangle inequality give
\begin{align*}
h(\bu)-h(\bv)
&\leq\left(\sum_{\bx\in\calX_{\mathrm{in},\bv}}|\bu^\top\bx|^2\right)^{1/2}
-\left(\sum_{\bx\in\calX_{\mathrm{in},\bv}}|\bv^\top\bx|^2\right)^{1/2}\\
&\leq\left(\sum_{\bx\in\calX_{\mathrm{in},\bv}}|(\bu-\bv)^\top\bx|^2\right)^{1/2}
\leq\|\bX_{\mathrm{in}}\|\,\|\bu-\bv\|.
\end{align*}
Interchanging $\bu$ and $\bv$ yields $|h(\bu)-h(\bv)|\leq\|\bX_{\mathrm{in}}\|\,\|\bu-\bv\|$, where $\bX_{\mathrm{in}}$ is the inlier data matrix.

For this matrix,~\cite[Example 6.2]{wainwright2019high} 
\[
\Pr\left(\|\bX_{\text{in}}\|\leq C{\frac{\sigma_{\mathrm{in}}}{\sqrt{d}}}(\sqrt{n_{\mathrm{in}}}+\sqrt{d}+t)\right)\geq 1-\exp(-t^2/2)
\]

For any fixed $\bu$,
$\bu^\top\bx=s_{\bu} g$, where $g\sim N(0,1)$
and $s_{\bu}^2=\bu^\top\bSigma_{\mathrm{in}}\bu
\geq\sigma_{\mathrm{in}}^2/d$.
As a result, Lemma~\ref{lemma:generalized_chisquared}(e)
implies that
$h(\bu)^2\geq
\nin\alpha_0^3\sigma_{\mathrm{in}}^2/(100d)$
with probability at least
$1-2e^{-\nin\alpha_0^2/200}$.
On its quantile event, the observations with
$|g_i|\leq\alpha_0/2$ are among the $m_0$ smallest, so their
truncated squared sum is a lower bound for
$h(\bu)^2/s_{\bu}^2$.

Applying a covering argument again over $\bu\in S^{D-1}\cap L_\star$ with balls of radius $\epsilon=c\alpha_0^{3/2}$, we have that
\[
\Pr\left[\min_{\substack{\calX'_{\mathrm{in}}\subseteq\calX_{\mathrm{in}}\\|\calX'_{\mathrm{in}}|\geq\alpha_0\nin}}
\sigma_d\left(\sum_{\calX'_{\mathrm{in}}}\bx\bx^\top\right)
\geq c\frac{\alpha_0^3\nin\sigma_{\mathrm{in}}^2}{d}\right]
\geq1-2(C/\alpha_0^{3/2})^d e^{-\nin\alpha_0^2/200}-2e^{-c\nin}.
\]
Thus, \eqref{eq:global23} holds (and Assumption~\ref{assump:global2} holds) with probability at least
\begin{equation}\label{eq:allprob5}
1-2(C/\alpha_0^{3/2})^d e^{-\nin\alpha_0^2/200}-2e^{-c\nin}-e^{-\nout/2}
\end{equation}
when
\[
\frac{n_{\mathrm{in}}}{n_{\mathrm{out}}} \gtrsim \beta\frac{ d \sigma_{\mathrm{out}}^2}{D\alpha_0^3 \sin(\theta_0) \sigma_{\mathrm{in}}^2}.
\]

\textbf{Summary:} Recall that in our model we have $\theta_0 = \pi/4$ and $\alpha_0=1/4$. Therefore, $\beta_1=O(1)$ and $\beta_2=O(\frac{\sigma_{\mathrm{in}}}{\sigma_{\mathrm{out}}}\sqrt{\frac{D}{D-d}}+1)$, and $\beta = 10 \beta_1 \beta_2$, which verifies the required conditional estimates for Assumption~\ref{assump:global2} on $\mathcal E_n$ under the population condition~\eqref{eq:prop:haystack_assump}. The count concentration and removal of conditioning at the beginning of Section~\ref{sec:haystack} give the unconditional high-probability conclusion.

\begin{proof}[{Proof of Lemma~\ref{lemma:generalized_chisquared}}]
(a) WLOG we may assume $\bSigma=\diag(\sigma_1,\cdots,\sigma_D)$. The case $\tr(\bSigma)=0$ is immediate. For positive trace, the moment-generating-function argument for weighted chi-squared sums (see \cite{Ghosh2021,wainwright2019high}) gives, at $t=1/(4\tr(\bSigma))$,
\[
\Expect\exp\left(t\sum_{i=1}^n\|\bx_i\|^2\right)
=\prod_j(1-2t\sigma_j)^{-n/2}\leq2^{n/2},
\]
since $\prod_j(1-\sigma_j/(2\tr(\bSigma)))\geq1/2$. Markov's inequality therefore yields
\[
\Pr\left(\sum_i\|\bx_i\|^2\geq2n\tr(\bSigma)\right)
\leq e^{-n(1-\log2)/2}\leq e^{-n/8}.
\]
By Cauchy--Schwarz,
\[
\Pr\left(\sum_i\|\bx_i\|\geq n\sqrt{2\tr(\bSigma)}\right)
\leq\Pr\left(\sum_i\|\bx_i\|^2\geq2n\tr(\bSigma)\right)\leq e^{-n/8}.
\]

(b) The first inequality follows from Hoeffding's inequality applied to $I(x_1\mathrel{<} q_{\alpha_0/2}(X)), \cdots, I(x_n\mathrel{<} q_{\alpha_0/2}(X))$ (i.e., to the empirical cdf). The proof of the second inequality is similar.

(c) Applying (b), we only need to investigate the $\alpha_0/4$-th quantile of the population distribution of $\|\bx_i\|$. We use the fact that for a positive random variable $X$, for any $t<0$, (the second inequality follows from the first inequality with $a=-1/t$)
\begin{equation}\label{eq:mgf_tailbound}
\Pr(X\leq a)\leq e^{-ta}\Expect [e^{tX}],\,\,\Pr(X\leq -1/t)\leq e\cdot \Expect [e^{tX}].
\end{equation}
Note that the moment generating function of a $\chi^2_1$ distribution is $(1-2t)^{-1/2}$; so when $X=\|Z\|^2$ with $Z\sim N(0,\diag(\sigma_1,\cdots,\sigma_D))$, is 
$\Expect[e^{t\|Z\|^2}]=\prod_{i=1}^D (1-2\sigma_it)^{-1/2}\leq (1-2\tr(\bSigma)t)^{-1/2}$ (the last inequality follows from $t\leq 0$). Let $t=\frac{1-16e^2/\alpha_0^2}{2\tr(\bSigma)}$, we have that the $\alpha_0/4$-th quantile of the population distribution of $\|\bx_i\|$ is at least $\sqrt{-1/t}=\sqrt{\frac{2\tr(\bSigma)}{16e^2/\alpha_0^2-1}}\geq \frac{\alpha_0}{4e}\sqrt{2\tr(\bSigma)}.$

(d) Following part (b), it is sufficient to show that 
\[
\Pr\left(\chi^2_p\geq (\sqrt{p}-2\log(\alpha_0/4))^2)\right)\leq \alpha_0/4.
\]
This follows from the tail bound for the chi-squared distribution in Lemma 1 of \citet{10.1214/aos/1015957395}.

As for the second part, it is sufficient to show that
\[
\Pr\left(\chi^2_p\leq \left(1-\sqrt{\frac{4}{p}\log\frac{4}{\alpha_0}}\right)p\right)\leq  \alpha_0/4.
\]
It follows from \cite[Corollary 2.3]{Barvinok2005} with setting $\epsilon=\sqrt{\frac{4}{p}\log\frac{4}{\alpha_0}}$, and note that $\epsilon\leq 1$ when $p\geq4\log(4/\alpha_0)$.

(e) The first part follows from part (b), as well as the fact that the distribution of $N(0,1)$ has a density bounded above by $1/\sqrt{2\pi}$, so $\int_{-\alpha_0/2}^{\alpha_0/2}\phi(x)\di x\leq \alpha_0/\sqrt{2\pi}<\alpha_0/2$.

The second part follows from applying Hoeffding's inequality to $X^2I(|X|\leq\alpha_0/2)$ where $X\sim N(0,1)$, which is bounded in $[0,\alpha_0^2/4]$ and has expectation 
\[
2\int_0^{\alpha_0/2}t^2\phi(t)\,dt
\geq\frac{\phi(1/2)}{12}\alpha_0^3>\frac{\alpha_0^3}{40}.
\]
The gap to $\alpha_0^3/100$ is at least $3\alpha_0^3/200$, so the failure probability is at most $e^{-9n\alpha_0^2/1250}\leq e^{-n\alpha_0^2/200}$. At $\alpha_0=0$ the assertion is immediate; at $\alpha_0=1$ the first quantile bound also follows from $\Pr(\max_i|g_i|<1/2)\leq2^{-n}\leq e^{-n/2}$.

\end{proof}

\subsection{Proof of Proposition \ref{prop:adv_assump}}
\label{subsec:prop_adv_proof}

\begin{proof}
    We begin by showing that if condition \eqref{eq:adv_condition} holds, then Assumption \ref{assump:global} is satisfied with high probability. Since outliers lie on the sphere, we have the trivial bound $S_{\mathrm{out}} \leq n_{\mathrm{out}}$. On the other hand, by Lemma \ref{lemma:prob}, $S_{\mathrm{in}} \gtrsim C n_{\mathrm{in}}\sigma_{\mathrm{in}}/\sqrt{d}$ with high probability. Combining these, we see that \eqref{eq:adv_condition} is a sufficient condition for Assumption \ref{assump:global} to hold with high probability, for a fixed $\theta_0$.

    For Assumption \ref{assump:global2}, we again assume that $\alpha_0=1/4$. Since the inliers follow a standard Gaussian distribution, by the verification of \eqref{eq:global21} in the proof of Proposition \ref{prop:haystack_assump}, \eqref{eq:global21} holds with high probability as in Proposition \ref{prop:haystack_assump}. As for \eqref{eq:global22}, it is not needed as $\gamma> 1-(1-\alpha_0)\gamma_\star$ holds (see discussion after \eqref{eq:global22}). Thus, it remains to verify \eqref{eq:global23}. On the one hand, we again have the trivial bound as all outliers have magnitudes $1$:
    \[
        \sigma_1\left(\sum_{\bx \in \calX_{\mathrm{out}}} \bx \bx^\top \right) \leq n_{\mathrm{out}}.
    \]
    On the other hand, for the inliers, following the proof of Proposition \ref{prop:haystack_assump}, we have the bound
    \[
        \min_{\calX_{\mathrm{in}}' \subseteq \calX_{\mathrm{in}} \,:\, |\calX_{\mathrm{in}}'| \geq \alpha_0|\calX_{\mathrm{in}}|} 
        \sigma_d\left(\sum_{\bx \in \calX_{\mathrm{in}}'} \bx \bx^\top \right) \geq c \frac{\alpha_0^3 n_{\mathrm{in}}\sigma_{\mathrm{in}}^2}{d}
    \]
    with high probability, and \eqref{eq:global23} then follows.
\end{proof}

}

\subsection{Proof of Propositions \ref{prop:affine_adversarial} and \ref{prop:affine_haystack}}\label{sec:affine_models}

{

\begin{proof}[Proof of Proposition~\ref{prop:affine_adversarial}]
Let $a=\dist(0,[A])$, where $[A]=[(L,\bm)]$, and denote the largest principal angle between $L$ and $L_\star$ by $\theta_1$. Write the inliers as $\bx_i=\bU_\star B\bg_i$, with $\bg_i\sim N(0,\bI_d)$ and $BB^\top\succeq\bI_d/d$. Uniformly over $[A]$ and every retained half-sample, with probability tending to one,
\begin{align}\label{eq:inlier_sum}
\frac1n\sum_{\bx\in\calX'_{\mathrm{in}}}\dist(\bx,[A])
&=\frac1n\sum_{\bx\in\calX'_{\mathrm{in}}}\|\bQ_L(\bx-\bm)\|\\\nonumber
&\geq c\frac{\nin}{n}\left(\|\bQ_L\bU_\star B\|_F^2+a^2\right)^{1/2}
\geq c\frac{\nin}{n}\left(\frac{\theta_1}{\sqrt d}+a\right).
\end{align}
For the uniform lower bound, put $y_i=(\bg_i^\top,1)^\top$. For unit $v=(u,b)$, $v^\top y\sim N(b,\|u\|^2)$ and $\|u\|^2+b^2=1$, so $\Pr\{|v^\top y|\leq t\}\leq2t$ for $0<t\leq1/4$. For $\|u\|\geq1/2$ use the Gaussian density bound; otherwise $|b|\geq\sqrt3/2$ and the density on $[-1/4,1/4]$ is less than one. Take $t=1/8$. The bounded functions $\min(|v^\top y|,t)$ obey a uniform law of large numbers on the unit sphere: finite nets and the integrable Lipschitz envelope $\|y\|$ prove this directly. Their expectations are at least $3t/4$, so eventually their empirical means are at least $5t/8$. Removing at most half the observations loses at most $t/2$ in the mean; hence
\[
\sum_{i\in I}|v^\top y_i|\geq \nin t/8
\quad\text{for all unit $v$ and $|I|\geq\nin/2$}.
\]
By the singular-value decomposition and concavity of the square root, this gives $\sum_{i\in I}\|My_i\|\geq c\nin\|M\|_F$ for every matrix $M$. Apply it to $M=[\bQ_L\bU_\star B,-\bQ_L\bm]$ and use $\sin\theta_1\geq2\theta_1/\pi$ to obtain \eqref{eq:inlier_sum}.
On the other hand, for any $\bx$, we have
\[
    \dist(P_{A_\star}\bx,[A])\leq\dist(P_{A_\star}\bx,L)+a\leq O\left(\theta_1\|P_{A_\star}\bx\|+a\right).
\]
As a result, the following inequality
\[
\cos c_0\sum_{\bx\in\calX'_{\mathrm{in}}}\dist(\bx,A)
\geq4\sqrt2\cdot\frac\pi2\sum_{\bx\in\calX_{\mathrm{out}}}\dist(P_{A_\star}\bx,[A])
\]
holds with probability tending to one when
\[
\frac{\alpha_{\mathrm{in}}}{\alpha_{\mathrm{out}}}
>\frac C{\cos c_0}\max\left(\sqrt d\int_{\bx\sim\mu_{\mathrm{out}}}\|P_{A_\star}\bx\|,1\right).
\]
Furthermore, applying \eqref{eq:inlier_sum} again, we also obtain the following inequality
\[
\cos c_0\sum_{\bx\in\calX'_{\mathrm{in}}}\dist(\bx,A)
\geq4\sqrt2\left(\nout\sum_j\theta_j^4\right)^{1/2}
\left\|P_{L_\star^\perp}\sum_{\bx\in\calX_{\mathrm{out}}}\bx\bx^\top P_{L_\star^\perp}\right\|^{1/2}
\]
under the condition
\[
\frac{\alpha_{\mathrm{in}}}{\alpha_{\mathrm{out}}}
>\frac C{\cos c_0}\,c_0\sqrt d
\left\|\int_{\bx\sim\mu_{\mathrm{out}}}P_{L_\star^\perp}\bx\bx^\top P_{L_\star^\perp}\right\|^{1/2},
\]
when $\dist(A_\star,[A])\leq c_0$, since $\sum_j\theta_j^4\leq\theta_1^2\sum_j\theta_j^2\leq c_0^2\theta_1^2$.
The random coefficients in these bounds are only the component counts, $\nout^{-1}\sum_{\calX_{\mathrm{out}}}\|P_{L_\star}\bx\|$, and $\nout^{-1}\sum_{\calX_{\mathrm{out}}}P_{L_\star^\perp}\bx\bx^\top P_{L_\star^\perp}$. They converge almost surely under the stated finite-moment conditions, in operator norm for the matrix in fixed dimensions. Together with \eqref{eq:inlier_sum}, this transfers the strict population inequalities uniformly in $A$. Finally, Gaussian general position bounds the number of inliers in the union in Assumption~\ref{assumption3} by $2d$. With $C>4$ the ratio condition gives $\nout+2d<\nin/2=\gamma n$ eventually, so that assumption holds as well.
Combining the estimations above, the proposition is proved.
\end{proof}
\begin{proof}[Proof of Proposition~\ref{prop:affine_haystack}]
Recall that the inliers are distributed as $\mathcal{N}(0, \bSigma_{\mathrm{in}}/d)$, while the projected outliers $P_{A_\star} \bx$ follow the distribution $\mathcal{N}(0, \bP_{L_\star} \bSigma_{\mathrm{out}} \bP_{L_\star} / D)$. 

Since the function $\dist(\bx, [A])$ is convex in $\bx$, and using the assumption $\sigma_d(\bSigma_{\mathrm{in}}) \geq 1$, we obtain
\[
\int_{\bx \sim \mu_{\mathrm{in}}} \dist(\bx, [A]) \geq \int_{\bx \sim \mu_0} \dist(\bx, [A]),
\]
where $\mu_0 \sim \mathcal{N}(0, \bP_{L_\star} / d)$.

Similarly, since  $\sigma_1(\bSigma_{\mathrm{out}})\leq\sigma_{\mathrm{out}}^2\leq D/d$, it follows that
\[
\bP_{L_\star} \bSigma_{\mathrm{out}} \bP_{L_\star} / D \preceq \bP_{L_\star} / d,
\]
and thus
\[
\int_{\bx \sim \mu_0} \dist(\bx, [A]) \geq \int_{\bx \sim \mu_{\mathrm{out}}} \dist(P_{A_\star} \bx, [A]).
\]

Additionally, observe that
\[
\left\|\int_{\bx\sim\mu_{\mathrm{out}}}P_{L_\star^\perp}\bx\bx^\top P_{L_\star^\perp}\right\|
=\frac{\|P_{L_\star^\perp}\bSigma_{\mathrm{out}}P_{L_\star^\perp}\|}{D}
\leq\frac{\sigma_{\mathrm{out}}^2}{D}.
\]
The uniform retained-sample estimate in the preceding proof bounds the inlier sum below by
$c\nin(\sum_j\sin^2\theta_j/d+a^2)^{1/2}$. The augmented projected outliers $(\sqrt d\,\bU_\star^\top\bx,1)$ have population second moment at most the identity; their empirical second-moment matrix is eventually at most twice the identity in fixed dimensions. Cauchy--Schwarz therefore bounds their distance sum above by $C\nout(\sum_j\sin^2\theta_j/d+a^2)^{1/2}$, uniformly over $[A]$. Thus a sufficiently large constant ratio controls the first term in Assumption~\ref{assumption4}.
Combining this bound with the inlier estimation \eqref{eq:inlier_sum} and using $\sum_j\theta_j^4\leq c_0^2\theta_1^2$, the proposition follows with the additional sufficient ratio $Cc_0\sigma_{\mathrm{out}}\sqrt{d/D}/\cos c_0$. The component-count and incidence arguments are as in the preceding proof.
\end{proof}

}

\section{Conclusion}

In this work, we have revisited the use of IRLS algorithms for robust subspace recovery. For linear subspaces, as in the FMS algorithm \citep{lerman2018fast}, we have incorporated dynamic smoothing, which allows us to prove a global recovery result under a deterministic condition. We have extended the FMS algorithm to the affine setting, and we show a local recovery result under a deterministic condition. These results represent significant theoretical contributions: this paper is the first to demonstrate global recovery for a nonconvex IRLS procedure and for IRLS over a Riemannian manifold, and it is the first to include a convergence analysis for an algorithm in affine robust subspace recovery. We include an example application of the (A)FMS algorithms to low-dimensional training of neural networks, where the found subspaces perform better than alternatives like PCA.

There are many open questions for future work.
While our results are stable to small perturbations of the inliers, detailed noise analysis in robust subspace recovery is still an open question. In the case of affine subspace recovery, it would be interesting to show if a similar global convergence result holds as in the linear case or if there are some obstructions to such a result. Another avenue is to develop a deeper understanding of how robust subspaces have better generalization in the training of neural networks, and to leverage this in specialized subspace recovery algorithms for this setting.

\section{Acknowledgements}

G. Lerman was supported by the NSF under Award No. 2152766.
T. Maunu was supported by the NSF under Award No. 2305315.
T. Zhang was supported by the NSF under Award No. 2318926.

\bibliographystyle{plainnat}
\bibliography{bib1,bib2}

\begin{thebibliography}{84}
\providecommand{\natexlab}[1]{#1}
\providecommand{\url}[1]{\texttt{#1}}
\expandafter\ifx\csname urlstyle\endcsname\relax
  \providecommand{\doi}[1]{doi: #1}\else
  \providecommand{\doi}{doi: \begingroup \urlstyle{rm}\Url}\fi

\bibitem[Absil et~al.(2007)Absil, Mahony, and
  Sepulchre]{absil2007optimization_book}
P.-A. Absil, R.~Mahony, and R.~Sepulchre.
\newblock \emph{Optimization Algorithms on Matrix Manifolds}.
\newblock Princeton University Press, USA, 2007.
\newblock ISBN 0691132984.

\bibitem[Arias-Castro and Wang(2017)]{ariascastro2017ransac}
Ery Arias-Castro and Jue Wang.
\newblock {RANSAC} algorithms for subspace recovery and subspace clustering.
\newblock \emph{arXiv preprint}, 2017.

\bibitem[Ba et~al.(2013)Ba, Babadi, Purdon, and Brown]{ba2013convergence}
Demba Ba, Behtash Babadi, Patrick~L Purdon, and Emery~N Brown.
\newblock Convergence and stability of iteratively re-weighted least squares
  algorithms.
\newblock \emph{IEEE Transactions on Signal Processing}, 62\penalty0
  (1):\penalty0 183--195, 2013.

\bibitem[Barvinok(2005)]{Barvinok2005}
Alexander Barvinok.
\newblock Lecture notes of math 710: Measure concentration, winter 2005, 2005.
\newblock URL
  \url{https://www.math.uchicago.edu/~shmuel/AAT-readings/Combinatorial%20Geometry,%20Concentration,%20Real%20Algebraic%20Geometry/Barvinok%20concentration%20notes.pdf}.

\bibitem[Basri and Jacobs(2003)]{basri2003lambertian}
Ronen Basri and David~W Jacobs.
\newblock Lambertian reflectance and linear subspaces.
\newblock \emph{IEEE Transactions on Pattern Analysis and Machine
  Intelligence}, 25\penalty0 (2):\penalty0 218--233, 2003.

\bibitem[Beck and Rosset(2023)]{beck2023dynamic}
Amir Beck and Israel Rosset.
\newblock A dynamic smoothing technique for a class of nonsmooth optimization
  problems on manifolds.
\newblock \emph{SIAM Journal on Optimization}, 33\penalty0 (3):\penalty0
  1473--1493, 2023.

\bibitem[Beck and Sabach(2015)]{Beck:2015vn}
Amir Beck and Shoham Sabach.
\newblock Weiszfeld's method: Old and new results.
\newblock \emph{Journal of Optimization Theory and Applications}, 164\penalty0
  (1):\penalty0 1--40, 2015.
\newblock \doi{10.1007/s10957-014-0586-7}.

\bibitem[Ben~Arous et~al.(2024)Ben~Arous, Gheissari, Huang, and
  Jagannath]{arous2024high}
Gerard Ben~Arous, Reza Gheissari, Jiaoyang Huang, and Aukosh Jagannath.
\newblock High-dimensional {SGD} aligns with emerging outlier eigenspaces.
\newblock In \emph{The Twelfth International Conference on Learning
  Representations}, 2024.

\bibitem[Bhatia(1996)]{bhatia1996matrix}
R.~Bhatia.
\newblock \emph{Matrix Analysis}.
\newblock Graduate Texts in Mathematics. Springer New York, 1996.
\newblock ISBN 9780387948461.
\newblock URL \url{https://books.google.com/books?id=F4hRy1F1M6QC}.

\bibitem[Bhatia(2009)]{bhatia2009positive}
Rajendra Bhatia.
\newblock \emph{Positive Definite Matrices}.
\newblock Princeton Series in Applied Mathematics. Princeton University Press,
  2009.
\newblock ISBN 9781400827787.
\newblock URL \url{https://books.google.com/books?id=-KIFglY18nYC}.

\bibitem[Boumal(2023)]{boumal2023introduction}
Nicolas Boumal.
\newblock \emph{An introduction to optimization on smooth manifolds}.
\newblock Cambridge University Press, 2023.

\bibitem[Burke et~al.(2015)Burke, Curtis, Wang, and Wang]{burke2015iterative}
James~V Burke, Frank~E Curtis, Hao Wang, and Jiashan Wang.
\newblock Iterative reweighted linear least squares for exact penalty
  subproblems on product sets.
\newblock \emph{SIAM Journal on Optimization}, 25\penalty0 (1):\penalty0
  261--294, 2015.

\bibitem[Cand{\`e}s et~al.(2011)Cand{\`e}s, Li, Ma, and
  Wright]{candes2011robust}
Emmanuel~J Cand{\`e}s, Xiaodong Li, Yi~Ma, and John Wright.
\newblock Robust principal component analysis?
\newblock \emph{Journal of the ACM}, 58\penalty0 (3):\penalty0 11, 2011.

\bibitem[Chartrand and Yin(2008)]{chartrand2008iteratively}
Rick Chartrand and Wotao Yin.
\newblock Iteratively reweighted algorithms for compressive sensing.
\newblock In \emph{2008 IEEE International Conference on Acoustics, Speech and
  Signal Processing}, pages 3869--3872. IEEE, 2008.

\bibitem[Chatterjee and Govindu(2017)]{chatterjee2017robust}
Avishek Chatterjee and Venu~Madhav Govindu.
\newblock Robust relative rotation averaging.
\newblock \emph{IEEE Transactions on Pattern Analysis and Machine
  Intelligence}, 40\penalty0 (4):\penalty0 958--972, 2017.

\bibitem[Chen et~al.(2019)Chen, Deng, Ma, and So]{9048840}
Shixiang Chen, Zengde Deng, Shiqian Ma, and Anthony Man-Cho So.
\newblock Manifold proximal point algorithms for dual principal component
  pursuit and orthogonal dictionary learning.
\newblock In \emph{2019 53rd Asilomar Conference on Signals, Systems, and
  Computers}, pages 259--263, 2019.

\bibitem[Cherapanamjeri et~al.(2017)Cherapanamjeri, Jain, and
  Netrapalli]{cherapanamjeri2017thresholding}
Yeshwanth Cherapanamjeri, Prateek Jain, and Praneeth Netrapalli.
\newblock Thresholding based outlier robust {PCA}.
\newblock In \emph{Conference on Learning Theory}, pages 593--628. PMLR, 2017.

\bibitem[Clarkson and Woodruff(2015)]{clarkson2015input}
Kenneth~L. Clarkson and David~P. Woodruff.
\newblock Input sparsity and hardness for robust subspace approximation.
\newblock In \emph{2015 IEEE 56th Annual Symposium on Foundations of Computer
  Science}, pages 310--329, 2015.

\bibitem[Danilova et~al.(2022)Danilova, Dvurechensky, Gasnikov, Gorbunov,
  Guminov, Kamzolov, and Shibaev]{danilova2022recent}
Marina Danilova, Pavel Dvurechensky, Alexander Gasnikov, Eduard Gorbunov,
  Sergey Guminov, Dmitry Kamzolov, and Innokentiy Shibaev.
\newblock Recent theoretical advances in non-convex optimization.
\newblock In \emph{High-Dimensional Optimization and Probability: With a View
  Towards Data Science}, pages 79--163. Springer, 2022.

\bibitem[Daubechies et~al.(2010)Daubechies, DeVore, Fornasier, and
  G{\"u}nt{\"u}rk]{daubechies2010iteratively}
Ingrid Daubechies, Ronald DeVore, Massimo Fornasier, and C.~Sinan
  G{\"u}nt{\"u}rk.
\newblock Iteratively reweighted least squares minimization for sparse
  recovery.
\newblock \emph{Communications on Pure and Applied Mathematics}, 63\penalty0
  (1):\penalty0 1--38, 2010.

\bibitem[Davis and Kahan(1970)]{a6ee9c48-1e5b-385a-93c1-cdf3b873de37}
Chandler Davis and W.~M. Kahan.
\newblock The rotation of eigenvectors by a perturbation. {III}.
\newblock \emph{SIAM Journal on Numerical Analysis}, 7\penalty0 (1):\penalty0
  1--46, 1970.
\newblock ISSN 00361429.
\newblock URL \url{http://www.jstor.org/stable/2949580}.

\bibitem[Deng et~al.(2009)Deng, Dong, Socher, Li, Li, and
  Fei-Fei]{imagenet_cvpr09}
Jia Deng, Wei Dong, Richard Socher, Li-Jia Li, Kai Li, and Li~Fei-Fei.
\newblock {ImageNet}: A large-scale hierarchical image database.
\newblock In \emph{IEEE Conference on Computer Vision and Pattern Recognition
  (CVPR)}, pages 248--255, 2009.

\bibitem[Ding et~al.(2006)Ding, Zhou, He, and Zha]{ding2006r}
Chris Ding, Ding Zhou, Xiaofeng He, and Hongyuan Zha.
\newblock {R1}-{PCA}: rotational invariant {L1}-norm principal component
  analysis for robust subspace factorization.
\newblock In \emph{Proceedings of the 23rd International Conference on Machine
  Learning}, pages 281--288, 2006.

\bibitem[Ding et~al.(2020)Ding, Yang, Zhu, Robinson, Vidal, Kneip, and
  Tsakiris]{ding2020robust}
Tianjiao Ding, Yunchen Yang, Zhihui Zhu, Daniel~P Robinson, Ren{\'e} Vidal,
  Laurent Kneip, and Manolis~C Tsakiris.
\newblock Robust homography estimation via dual principal component pursuit.
\newblock In \emph{Proceedings of the IEEE/CVF Conference on Computer Vision
  and Pattern Recognition}, pages 6080--6089, 2020.

\bibitem[Ding et~al.(2021)Ding, Zhu, Vidal, and Robinson]{ding2021dual}
Tianyu Ding, Zhihui Zhu, Ren{\'e} Vidal, and Daniel~P Robinson.
\newblock Dual principal component pursuit for robust subspace learning: Theory
  and algorithms for a holistic approach.
\newblock In \emph{International Conference on Machine Learning}, pages
  2739--2748. PMLR, 2021.

\bibitem[Ding et~al.(2024)Ding, Zhou, Chen, Zhu, Zharkov, and
  Liang]{ding2024adacontour}
Tianyu Ding, Jinxin Zhou, Tianyi Chen, Zhihui Zhu, Ilya Zharkov, and Luming
  Liang.
\newblock Adacontour: Adaptive contour descriptor with hierarchical
  representation.
\newblock \emph{arXiv preprint arXiv:2404.08292}, 2024.

\bibitem[Edelman et~al.(1998)Edelman, Arias, and Smith]{edelman1998geometry}
Alan Edelman, Tom{\'a}s~A Arias, and Steven~T Smith.
\newblock The geometry of algorithms with orthogonality constraints.
\newblock \emph{SIAM Journal on Matrix Analysis and Applications}, 20\penalty0
  (2):\penalty0 303--353, 1998.

\bibitem[Fischler and Bolles(1981)]{fischler1981random}
Martin~A Fischler and Robert~C Bolles.
\newblock Random sample consensus: a paradigm for model fitting with
  applications to image analysis and automated cartography.
\newblock \emph{Communications of the ACM}, 24\penalty0 (6):\penalty0 381--395,
  1981.

\bibitem[Fornasier et~al.(2011)Fornasier, Rauhut, and
  Ward]{fornasier2011lowrank}
Massimo Fornasier, Holger Rauhut, and Rachel Ward.
\newblock Low-rank matrix recovery via iteratively reweighted least squares
  minimization.
\newblock \emph{SIAM Journal on Optimization}, 21\penalty0 (4):\penalty0
  1614--1640, 2011.

\bibitem[Fornasier et~al.(2016)Fornasier, Peter, Rauhut, and
  Worm]{fornasier2016conjugate}
Massimo Fornasier, Steffen Peter, Holger Rauhut, and Stephan Worm.
\newblock Conjugate gradient acceleration of iteratively re-weighted least
  squares methods.
\newblock \emph{Computational Optimization and Applications}, 65\penalty0
  (1):\penalty0 205--259, 2016.

\bibitem[Ghosh(2021)]{Ghosh2021}
Malay Ghosh.
\newblock Exponential tail bounds for {Chisquared} random variables.
\newblock \emph{Journal of Statistical Theory and Practice}, 15\penalty0
  (2):\penalty0 35, Mar 2021.
\newblock ISSN 1559-8616.
\newblock \doi{10.1007/s42519-020-00156-x}.
\newblock URL \url{https://doi.org/10.1007/s42519-020-00156-x}.

\bibitem[Gorodnitsky and Rao(2002)]{gorodnitsky2002sparse}
Irina~F Gorodnitsky and Bhaskar~D Rao.
\newblock Sparse signal reconstruction from limited data using {FOCUSS}: A
  re-weighted minimum norm algorithm.
\newblock \emph{IEEE Transactions on Signal Processing}, 45\penalty0
  (3):\penalty0 600--616, 2002.

\bibitem[G{\"u}rb{\"u}zbalaban et~al.(2021)G{\"u}rb{\"u}zbalaban,
  {\c{S}}im{\c{s}}ekli, and Zhu]{gurbuzbalaban2021heavy}
Mert G{\"u}rb{\"u}zbalaban, Umut {\c{S}}im{\c{s}}ekli, and Lingjiong Zhu.
\newblock The heavy-tail phenomenon in {SGD}.
\newblock In \emph{International Conference on Machine Learning}, pages
  3964--3975. PMLR, 2021.

\bibitem[Hardt and Moitra(2013)]{hardt2013algorithms}
Moritz Hardt and Ankur Moitra.
\newblock Algorithms and hardness for robust subspace recovery.
\newblock In \emph{Conference on Learning Theory}, pages 354--375. PMLR, 2013.

\bibitem[Hartley and Zisserman(2003)]{hartley2003multiple}
Richard Hartley and Andrew Zisserman.
\newblock \emph{Multiple view geometry in computer vision}.
\newblock Cambridge University Press, 2003.

\bibitem[Holland and Welsch(1977)]{holland1977robust}
Paul~W Holland and Roy~E Welsch.
\newblock Robust regression using iteratively reweighted least-squares.
\newblock \emph{Communications in Statistics-theory and Methods}, 6\penalty0
  (9):\penalty0 813--827, 1977.

\bibitem[Hotelling(1933)]{hotelling1933analysis}
Harold Hotelling.
\newblock Analysis of a complex of statistical variables into principal
  components.
\newblock \emph{Journal of Educational Psychology}, 24\penalty0 (6):\penalty0
  417, 1933.

\bibitem[Hu et~al.(2020)Hu, Liu, Wen, and Yuan]{Hu20_manifold_optimization}
Jiang Hu, Xin Liu, Zaiwen Wen, and Ya{-}xiang Yuan.
\newblock A brief introduction to manifold optimization.
\newblock \emph{Journal of the Operations Research Society of China},
  8:\penalty0 199–248, 2020.
\newblock \doi{10.1007/s40305-020-00295-9}.

\bibitem[Jaquier et~al.(2020)Jaquier, Rozo, Calinon, and
  B{\"u}rger]{jaquier2020bayesian}
No{\'e}mie Jaquier, Leonel Rozo, Sylvain Calinon, and Mathias B{\"u}rger.
\newblock Bayesian optimization meets {Riemannian} manifolds in robot learning.
\newblock In \emph{Conference on Robot Learning}, pages 233--246. PMLR, 2020.

\bibitem[Jastrz{\k{e}}bski et~al.(2017)Jastrz{\k{e}}bski, Kenton, Arpit,
  Ballas, Fischer, Bengio, and Storkey]{jastrzkebski2017three}
Stanis{\l}aw Jastrz{\k{e}}bski, Zachary Kenton, Devansh Arpit, Nicolas Ballas,
  Asja Fischer, Yoshua Bengio, and Amos Storkey.
\newblock Three factors influencing minima in {SGD}.
\newblock \emph{arXiv preprint arXiv:1711.04623}, 2017.

\bibitem[Jolliffe(1986)]{jolliffe1986principal}
I.T. Jolliffe.
\newblock \emph{Principal Component Analysis}.
\newblock Springer series in statistics. Springer-Verlag, 1986.
\newblock ISBN 9780387962696.

\bibitem[Krizhevsky(2009)]{Krizhevsky2009LearningML}
Alex Krizhevsky.
\newblock Learning multiple layers of features from tiny images.
\newblock Technical report, University of Toronto, 2009.

\bibitem[Kuhn(1973)]{Kuhn:1973vr}
Harold~W. Kuhn.
\newblock A note on {Fermat's} problem.
\newblock \emph{Mathematical Programming}, 4\penalty0 (1):\penalty0 98--107,
  1973.

\bibitem[K{\"u}mmerle and Sigl(2018)]{kummerle2018harmonic}
Christian K{\"u}mmerle and Juliane Sigl.
\newblock Harmonic mean iteratively reweighted least squares for low-rank
  matrix recovery.
\newblock \emph{Journal of Machine Learning Research}, 19\penalty0
  (47):\penalty0 1--49, 2018.

\bibitem[K{\"u}mmerle and Verdun(2021)]{kummerle2021scalable}
Christian K{\"u}mmerle and Claudio~M Verdun.
\newblock A scalable second order method for ill-conditioned matrix completion
  from few samples.
\newblock In \emph{International Conference on Machine Learning}, pages
  5872--5883. PMLR, 2021.

\bibitem[K\"{u}mmerle et~al.(2021)K\"{u}mmerle, Mayrink~Verdun, and
  St\"{o}ger]{kummerle2020iteratively}
Christian K\"{u}mmerle, Claudio Mayrink~Verdun, and Dominik St\"{o}ger.
\newblock Iteratively reweighted least squares for basis pursuit with global
  linear convergence rate.
\newblock In M.~Ranzato, A.~Beygelzimer, Y.~Dauphin, P.S. Liang, and J.~Wortman
  Vaughan, editors, \emph{Advances in Neural Information Processing Systems},
  volume~34, pages 2873--2886. Curran Associates, Inc., 2021.

\bibitem[Laurent and Massart(2000)]{10.1214/aos/1015957395}
B.~Laurent and P.~Massart.
\newblock {Adaptive estimation of a quadratic functional by model selection}.
\newblock \emph{The Annals of Statistics}, 28\penalty0 (5):\penalty0 1302 --
  1338, 2000.
\newblock \doi{10.1214/aos/1015957395}.
\newblock URL \url{https://doi.org/10.1214/aos/1015957395}.

\bibitem[Lerman and Maunu(2018{\natexlab{a}})]{lerman2018fast}
Gilad Lerman and Tyler Maunu.
\newblock Fast, robust and non-convex subspace recovery.
\newblock \emph{Information and Inference: A Journal of the IMA}, 7\penalty0
  (2):\penalty0 277--336, 2018{\natexlab{a}}.

\bibitem[Lerman and Maunu(2018{\natexlab{b}})]{lerman2018overview}
Gilad Lerman and Tyler Maunu.
\newblock An overview of robust subspace recovery.
\newblock \emph{Proceedings of the IEEE}, 106\penalty0 (8):\penalty0
  1380--1410, 2018{\natexlab{b}}.

\bibitem[Lerman and Zhang(2014)]{lerman2014lp}
Gilad Lerman and Teng Zhang.
\newblock {$l_p$}-recovery of the most significant subspace among multiple
  subspaces with outliers.
\newblock \emph{Constructive Approximation}, 40\penalty0 (3):\penalty0
  329--385, 2014.

\bibitem[Lerman and Zhang(2024)]{lerman_zhang2024}
Gilad Lerman and Teng Zhang.
\newblock Theoretical guarantees for the subspace-constrained {T}yler's
  estimator, 2024.
\newblock URL \url{https://arxiv.org/abs/2403.18658}.

\bibitem[Lerman et~al.(2015)Lerman, McCoy, Tropp, and Zhang]{lerman2015robust}
Gilad Lerman, Michael~B McCoy, Joel~A Tropp, and Teng Zhang.
\newblock Robust computation of linear models by convex relaxation.
\newblock \emph{Foundations of Computational Mathematics}, 15\penalty0
  (2):\penalty0 363--410, 2015.

\bibitem[Li et~al.(2024)Li, Ma, and Srivastava]{li2024riemannian}
Jiaxiang Li, Shiqian Ma, and Tejes Srivastava.
\newblock A {Riemannian} alternating direction method of multipliers.
\newblock \emph{Mathematics of Operations Research}, 2024.

\bibitem[Li et~al.(2022)Li, Tan, Huang, Tao, Liu, and Huang]{li2022low}
Tao Li, Lei Tan, Zhehao Huang, Qinghua Tao, Yipeng Liu, and Xiaolin Huang.
\newblock Low dimensional trajectory hypothesis is true: {DNNs} can be trained
  in tiny subspaces.
\newblock \emph{IEEE Transactions on Pattern Analysis and Machine
  Intelligence}, 45\penalty0 (3):\penalty0 3411--3420, 2022.

\bibitem[Li et~al.(2019)Li, Lu, Arora, Haupt, Liu, Wang, and
  Zhao]{li2019symmetry}
Xingguo Li, Junwei Lu, Raman Arora, Jarvis Haupt, Han Liu, Zhaoran Wang, and
  Tuo Zhao.
\newblock Symmetry, saddle points, and global optimization landscape of
  nonconvex matrix factorization.
\newblock \emph{IEEE Transactions on Information Theory}, 65\penalty0
  (6):\penalty0 3489--3514, 2019.

\bibitem[Locantore et~al.(1999)Locantore, Marron, Simpson, Tripoli, Zhang, and
  Cohen]{Locantore1999}
N.~Locantore, J.~S. Marron, D.~G. Simpson, N.~Tripoli, J.~T. Zhang, and K.~L.
  Cohen.
\newblock Robust principal component analysis for functional data.
\newblock \emph{Test}, 8\penalty0 (1):\penalty0 1--73, 1999.

\bibitem[Maronna et~al.(2006)Maronna, Martin, and Yohai]{maronna2006robust}
R.A. Maronna, D.R. Martin, and V.J. Yohai.
\newblock \emph{Robust Statistics: Theory and Methods}.
\newblock Wiley Series in Probability and Statistics. Wiley, 2006.
\newblock ISBN 9780470010921.

\bibitem[Maunu and Lerman(2019)]{maunu2019robustsubspacerecoveryadversarial}
Tyler Maunu and Gilad Lerman.
\newblock Robust subspace recovery with adversarial outliers, 2019.
\newblock URL \url{https://arxiv.org/abs/1904.03275}.

\bibitem[Maunu et~al.(2019)Maunu, Zhang, and Lerman]{maunu2019well}
Tyler Maunu, Teng Zhang, and Gilad Lerman.
\newblock A well-tempered landscape for non-convex robust subspace recovery.
\newblock \emph{Journal of Machine Learning Research}, 20\penalty0 (37), 2019.

\bibitem[Mukhoty et~al.(2019)Mukhoty, Gopakumar, Jain, and
  Kar]{mukhoty2019globally}
Bhaskar Mukhoty, Govind Gopakumar, Prateek Jain, and Purushottam Kar.
\newblock Globally-convergent iteratively reweighted least squares for robust
  regression problems.
\newblock In \emph{The 22nd International Conference on Artificial Intelligence
  and Statistics}, pages 313--322. PMLR, 2019.

\bibitem[Pearson(1901)]{pearson1901liii}
Karl Pearson.
\newblock {LIII}. on lines and planes of closest fit to systems of points in
  space.
\newblock \emph{The London, Edinburgh, and Dublin Philosophical Magazine and
  Journal of Science}, 2\penalty0 (11):\penalty0 559--572, 1901.

\bibitem[Peng and Vidal(2023)]{peng2023block}
Liangzu Peng and Ren{\'e} Vidal.
\newblock Block coordinate descent on smooth manifolds: Convergence theory and
  twenty-one examples.
\newblock \emph{arXiv preprint arXiv:2305.14744}, 2023.

\bibitem[Peng et~al.(2022)Peng, K\"{u}mmerle, and Vidal]{peng2022global}
Liangzu Peng, Christian K\"{u}mmerle, and Ren\'{e} Vidal.
\newblock Global linear and local superlinear convergence of {IRLS} for
  non-smooth robust regression.
\newblock In \emph{Proceedings of the 36th International Conference on Neural
  Information Processing Systems}. Curran Associates Inc., 2022.
\newblock ISBN 9781713871088.

\bibitem[Peng et~al.(2023)Peng, K{\"u}mmerle, and Vidal]{peng2023convergence}
Liangzu Peng, Christian K{\"u}mmerle, and Ren{\'e} Vidal.
\newblock On the convergence of {IRLS} and its variants in outlier-robust
  estimation.
\newblock In \emph{Proceedings of the IEEE/CVF Conference on Computer Vision
  and Pattern Recognition}, pages 17808--17818, 2023.

\bibitem[Rahmani and Atia(2017)]{rahmani2017coherence}
Mostafa Rahmani and George~K Atia.
\newblock Coherence pursuit: Fast, simple, and robust principal component
  analysis.
\newblock \emph{IEEE Transactions on Signal Processing}, 65\penalty0
  (23):\penalty0 6260--6275, 2017.

\bibitem[Rao and Kreutz-Delgado(2002)]{rao2002affine}
Bhaskar~D Rao and Kenneth Kreutz-Delgado.
\newblock An affine scaling methodology for best basis selection.
\newblock \emph{IEEE Transactions on Signal Processing}, 47\penalty0
  (1):\penalty0 187--200, 2002.

\bibitem[{\c{S}}im{\c{s}}ekli et~al.(2019){\c{S}}im{\c{s}}ekli, Sagun, and
  G{\"u}rb{\"u}zbalaban]{simsekli2019tailindex}
Umut {\c{S}}im{\c{s}}ekli, Levent Sagun, and Mert G{\"u}rb{\"u}zbalaban.
\newblock A tail-index analysis of stochastic gradient noise in deep neural
  networks.
\newblock In \emph{International Conference on Machine Learning}, pages
  5827--5837. PMLR, 2019.

\bibitem[Sun et~al.(2016)Sun, Babu, and Palomar]{sun2016majorization}
Ying Sun, Prabhu Babu, and Daniel~P Palomar.
\newblock Majorization-minimization algorithms in signal processing,
  communications, and machine learning.
\newblock \emph{IEEE Transactions on Signal Processing}, 65\penalty0
  (3):\penalty0 794--816, 2016.

\bibitem[Tyler(1987)]{tyler1987distribution}
David~E Tyler.
\newblock A distribution-free {M}-estimator of multivariate scatter.
\newblock \emph{The Annals of Statistics}, pages 234--251, 1987.

\bibitem[Vandereycken(2013)]{Vandereycken13_completion_riemannian}
Bart Vandereycken.
\newblock Low-rank matrix completion by {Riemannian} optimization.
\newblock \emph{SIAM Journal on Optimization}, 23\penalty0 (2):\penalty0
  1214--1236, 2013.

\bibitem[Vershynin(2018)]{vershynin2018high}
Roman Vershynin.
\newblock \emph{High-Dimensional Probability: An Introduction with Applications
  in Data Science}.
\newblock Cambridge Series in Statistical and Probabilistic Mathematics.
  Cambridge University Press, 2018.
\newblock ISBN 9781108415194.

\bibitem[Wainwright(2019)]{wainwright2019high}
M.J. Wainwright.
\newblock \emph{High-Dimensional Statistics: A Non-Asymptotic Viewpoint}.
\newblock Cambridge Series in Statistical and Probabilistic Mathematics.
  Cambridge University Press, 2019.
\newblock ISBN 9781108498029.
\newblock URL \url{https://books.google.com/books?id=IluHDwAAQBAJ}.

\bibitem[Wang et~al.(2024)Wang, Zhang, Zhang, Chen, Ma, and
  Qu]{wang2024diffusion}
Peng Wang, Huijie Zhang, Zekai Zhang, Siyi Chen, Yi~Ma, and Qing Qu.
\newblock Diffusion model learns low-dimensional distributions via subspace
  clustering.
\newblock In \emph{NeurIPS 2024 Workshop on Mathematics of Modern Machine
  Learning}, 2024.

\bibitem[Weiszfeld(1937)]{weiszfeld1937point}
Endre Weiszfeld.
\newblock Sur le point pour lequel la somme des distances de n points
  donn{\'e}s est minimum.
\newblock \emph{Tohoku Mathematical Journal, First Series}, 43:\penalty0
  355--386, 1937.

\bibitem[Wiesel and Zhang(2015)]{SIG-053}
Ami Wiesel and Teng Zhang.
\newblock Structured robust covariance estimation.
\newblock \emph{Foundations and Trends{\textregistered} in Signal Processing},
  8\penalty0 (3):\penalty0 127--216, 2015.
\newblock ISSN 1932-8346.
\newblock \doi{10.1561/2000000053}.

\bibitem[Xu et~al.(2012)Xu, Caramanis, and Sanghavi]{xu2012robust}
Huan Xu, Constantine Caramanis, and Sujay Sanghavi.
\newblock Robust {PCA} via outlier pursuit.
\newblock \emph{IEEE Transactions on Information Theory}, 58\penalty0
  (5):\penalty0 3047--3064, 2012.

\bibitem[Yang et~al.(2022{\natexlab{a}})Yang, Wang, and Wang]{yang2022towards}
Xiangyu Yang, Jiashan Wang, and Hao Wang.
\newblock Towards an efficient approach for the nonconvex {$\ell_p$} ball
  projection: algorithm and analysis.
\newblock \emph{Journal of Machine Learning Research}, 23\penalty0
  (101):\penalty0 1--31, 2022{\natexlab{a}}.

\bibitem[Yang et~al.(2022{\natexlab{b}})Yang, Huang, and
  Wipf]{yang2022transformers}
Yongyi Yang, Zengfeng Huang, and David~P Wipf.
\newblock Transformers from an optimization perspective.
\newblock \emph{Advances in Neural Information Processing Systems},
  35:\penalty0 36958--36971, 2022{\natexlab{b}}.

\bibitem[Yu et~al.(2024)Yu, Zhang, and Lerman]{yu2024subspace}
Feng Yu, Teng Zhang, and Gilad Lerman.
\newblock A subspace-constrained {T}yler's estimator and its applications to
  structure from motion.
\newblock In \emph{Proceedings of the IEEE/CVF Conference on Computer Vision
  and Pattern Recognition}, pages 14575--14584, 2024.

\bibitem[Zhang(2016)]{zhang2016robust}
Teng Zhang.
\newblock Robust subspace recovery by {T}yler's {M}-estimator.
\newblock \emph{Information and Inference: A Journal of the IMA}, 5\penalty0
  (1):\penalty0 1--21, 2016.

\bibitem[Zhang and Lerman(2014)]{JMLR:v15:zhang14a}
Teng Zhang and Gilad Lerman.
\newblock A novel {M}-estimator for robust {PCA}.
\newblock \emph{Journal of Machine Learning Research}, 15\penalty0
  (23):\penalty0 749--808, 2014.

\bibitem[Zhang et~al.(2009)Zhang, Szlam, and Lerman]{Zhang_MKF09}
Teng Zhang, Arthur Szlam, and Gilad Lerman.
\newblock Median k-flats for hybrid linear modeling with many outliers.
\newblock In \emph{2009 IEEE 12th International Conference on Computer Vision
  Workshops, ICCV Workshops}, pages 234--241, 2009.
\newblock \doi{10.1109/ICCVW.2009.5457695}.

\bibitem[Zhou et~al.(2020)Zhou, Feng, Ma, Xiong, Hoi, and E]{zhou2020towards}
Pan Zhou, Jiashi Feng, Chao Ma, Caiming Xiong, Steven Chu~Hong Hoi, and Weinan
  E.
\newblock Towards theoretically understanding why {SGD} generalizes better than
  {Adam} in deep learning.
\newblock \emph{Advances in Neural Information Processing Systems},
  33:\penalty0 21285--21296, 2020.

\bibitem[Zhu et~al.(2018)Zhu, Wang, Robinson, Naiman, Vidal, and
  Tsakiris]{NEURIPS2018_af21d0c9}
Zhihui Zhu, Yifan Wang, Daniel Robinson, Daniel Naiman, Rene Vidal, and
  Manolis~C. Tsakiris.
\newblock Dual principal component pursuit: improved analysis and efficient
  algorithms.
\newblock In \emph{Proceedings of the 32nd International Conference on Neural
  Information Processing Systems}, NIPS'18, page 2175–2185, 2018.

\end{thebibliography}

\end{document}